\documentclass{article}

\def \TRtitle{Local Rademacher Complexity-based Learning Guarantees for Multi-Task Learning}
\def \TRauthors{Niloofar Yousefi$^{1}$, Yunwen Lei$^{2}$ , Marius Kloft$^{3}$ \\  Mansooreh Mollaghasemi$^{4}$ and Georgios Anagnostopoulos$^{5}$}
\def \TRemails{niloofar.yousefi@ucf.edu, yunwen.lei@hotmail.com, kloft@hu-berlin.de, Mansooreh.Mollaghasemi@ucf.edu and georgio@fit.edu}
\def \TRaffiliations{$^{1}$Department of Electrical Engineering and Computer Science, University of Central Florida \\$^{2}$Department of Mathematics, City University of Hong Kong \\ $^{3}$Department of Computer Science, Humboldt University of Berlin \\ $^{4}$Department of Industrial Engineering \& Management Systems, University of Central Florida \\ $^{5}$Department of Electrical and Computer Engineering, Florida Institute of Technology}
\def \TRkeywords{Multi-Task Learning, Generalization Bound, Local Rademacher Complexity.}

\usepackage[hyphens]{url}

\usepackage{./sty/options}
\usepackage{./sty/packages}
\usepackage{./sty/macros}
\usepackage{enumitem}
\usepackage{times,mathtools}

\allowdisplaybreaks

\newcommand{\nbb}{\ensuremath{\mathbb{N}}}
\newcommand{\ebb}{\mathbb{E}}
\newcommand{\rbb}{\mathbb{R}}

\newcommand{\fcal}{\mathcal{F}}
\newcommand{\xcal}{\mathcal{X}}
\newcommand{\hcal}{\mathcal{H}}

\newcommand{\boldf}{\boldsymbol{f}}

\newcommand{\citep}{\cite}

\begin{document}

\maketitle

\begin{abstract}
We show a Talagrand-type concentration inequality for \ac {MTL}, using which we establish sharp excess risk bounds for \ac{MTL} in terms of distribution- and data-dependent versions of the \ac{LRC}.
We also give a new bound on the \ac{LRC} for norm regularized as well as strongly convex hypothesis classes, which applies not only to \ac{MTL} but also to the standard i.i.d. setting.
Combining both results, one can now easily derive fast-rate bounds on the excess risk for many prominent \ac{MTL} methods,
including---as we demonstrate---Schatten-norm, group-norm, and graph-regularized MTL.
The derived bounds reflect a relationship akeen to a conservation law of asymptotic convergence rates. 
This very relationship allows for trading off slower rates \wrt\ the number of tasks for faster rates with respect to the number of available samples per task, when compared to the rates obtained via a traditional, global Rademacher analysis.   
\end{abstract}

\ifMakeReviewDraft
	\linenumbers
\fi

\vskip 0.5in
\noindent
{\bf Keywords:} \TRkeywords

\acresetall
\section{Introduction}
\label{sec:Introduction}

A commonly occurring problem, when applying machine learning in the sciences, is the lack of a sufficient amount of training data to attain acceptable performance results; either obtaining such data may be very costly or they may be unavailable due to technological limitations. For example, in cancer genomics, tumor bioptic samples may be relatively scarce due to the limited number of cancer patients, when compared to samples of healthy individuals. In neuroscience, electroencephalogram experiments are carried out on human subjects to record training data and typically involve only a few dozen subjects. 

When considering any type of prediction task per individual subject in such settings (for example, whether the subject is indeed suffering a specific medical affliction or not), relying solely on the scarce data per individual most often leads to inadequate predictive performance. Such a direct approach completely ignores the advantages that might be gained, when considering intrinsic, strong similarities between subjects and, hence, tasks. Revisiting the area of genomics, different living organisms can be related to each other in terms of their evolutionary relationships -- the information of how they are genetically related to each other can be obtained from the tree of life. Taking into account such relationships may be instrumental in detecting genes of recently developed organisms, for which only a limited number of training data is available. While our discussion so far focused on the realm of biomedicine, similar limitations and opportunities to overcome them exist in other fields as well. 

Transfer learning \citep{pan2010survey} and, in particular, \ac{MTL} \citep{caruana1997} leverages such underlying common links among a group of tasks, while respecting the tasks' individual idiosyncrasies to the extent warranted.
This is achieved by phrasing the learning process as a joint, mutually dependent learning problem. An early example of such a learning paradigm is the neural network-based approach introduced by \cite{caruana1997}, while more recent works consider convex \ac{MTL} problems \citep{ando2005,evgeniou2004regularized,Argyriou2008}. At the core of each \ac{MTL} formulation lies a mechanism that encodes task relatedness into the learning problem \citep{evgeniou2005learning}. 
Such relatedness mechanism can always be thought of as jointly constraining the tasks' hypothesis spaces, so that their geometry is mutually coupled, \eg, via a block norm constraint \citep{Yousefi2015}. Thus, from a regularization perspective, the tasks mutually regularize their learning based on their inter-task relatedness. This process of information exchange during co-learning is often referred to as \emph{information sharing}. With respect to learning-theoretical results, the analysis of \ac{MTL} goes back to the seminal work by \cite{baxter2000}, which was followed up by the works of \cite{ando2005,maurer2006}. Nowadays, \ac{MTL} frameworks are routinely employed in a variety of settings. Some recent applications include computational genetics \citep{widmer2013multi},
image segmentation \citep{an2008}, HIV therapy screening \citep{bickel2008}, collaborative filtering \citep{cao2010}, age estimation from facial images \citep{zhang2010multi}, and sub-cellular location prediction \citep{xu2011multitask}, just to name a few prominent ones.

\ac{MTL} learning guarantees are centered around the notion of (global) Rademacher complexities; notions that were introduced to the machine learning community by \cite{bartlett2002model,Bartlett2003,koltchinskii2000rademacher,koltchinskii2001rademacher,Koltchinskii2002}, 
and employed in the context of MTL by \cite{maurer2006,maurer2006Graph,kakade2012,Maurer2013,Maurer2014,maurer2015}.
All these papers are briefly surveyed in \sref{sec:PreviousRelatedWorks}. 
To formally recapitulate the essence of these works, let $T$ denote the number of tasks being co-learned and $n$ denote the number of available observations per task. Then, in terms of convergence rates in the number of samples and tasks, respectively, the fastest-converging error or excess risk bounds derived in these works -- whether distribution- or data-dependent -- are of the order $O(1/\sqrt{n T})$.

More recently, the seminal works by \cite{Koltchinskii2006} and \cite{bartlett2005} introduced a more nuanced variant of these complexities, termed \ac{LRC} (as opposed to the original \ac{GRC}). This new, modified function class complexity measure is attention-worthy, since, as shown by \cite{bartlett2005}, a \ac{LRC}-based analysis is capable of producing more rapidly-converging excess risk bounds (``fast rates''), when compared to the ones obtained via a \ac{GRC} analysis. This can be attributed to the fact that, unlike \acp{LRC}, \acp{GRC} ignore the fact that learning algorithms typically choose well-performing hypotheses that belong only to a subset of the entire hypothesis space under consideration. The end result of this distinction empowers a local analysis to provide less conservative and, hence, sharper bounds than the standard global analysis. To date, there have been only a few additional works attempting to reap the benefits of such local analysis in various contexts: active learning for binary classification tasks \citep{Koltchinskii2010}, multiple kernel learning \citep{Kloft2011,Cortes2013}, transductive learning \citep{Tolstikhin2014}, semi-supervised learning \citep{Oneto2015} and bounds on the \acp{LRC} via covering numbers \citep{Lei2015}.

\subsection{Our Contributions}

Through a Talagrand-type concentration inequality adapted to the \ac{MTL} case, this paper's main contribution is the derivation of sharp bounds on the excess \ac{MTL} risk  in terms of the distribution- and data-dependent \ac{LRC}. For a given number of tasks $T$, these bounds admit faster (asymptotic) convergence characteristics in the number of observations per task $n$, when compared to corresponding bounds hinging on the \ac{GRC}. Hence, these faster rates allow us to increase the confidence that the \ac{MTL} hypothesis selected by a learning algorithm approaches the best-in-class solution as $n$ increases beyond a certain threshold.
We also prove a new bound on the \ac{LRC}, which generally holds for hypothesis classes with any norm function or strongly convex regularizers. This bound readily facilitates the bounding of the \ac{LRC} for a range of such regularizers (not only for \ac{MTL}, but also for the standard i.i.d. setting), as we demonstrate for classes induced by graph-based, Schatten- and group-norm regularizers. Moreover, we prove matching lower bounds showing that, aside from constants, the \ac{LRC}-based bounds are tight for the considered applications.

Our derived bounds reflect that one can trade off a slow convergence speed \wrt\ $T$ for an improved convergence rate \wrt\ $n$. The latter one ranges, in the worst case, from the typical \ac{GRC}-based bounds $O(1/\sqrt{n})$, all the way up to the fastest rate of order $O(1/n)$ by allowing the bound to depend less on $T$.
Nevertheless, the premium in question becomes less relevant to \ac{MTL}, in which $T$ is typically considered as fixed.

Fixing all other parameters when the number of samples per task $n$ approaches to infinity, our local bounds yield faster rates compared to their global counterparts. Also, it is witnessed that if the number of tasks $T$ and the radius $R$ of the ball-norms---in considered norm regularized hypotheses---can grow with $n$, there are cases in which local analysis always improves other global one. In comparison of our local bounds to some related works \citep{Maurer2013,maurer2006Graph} based on global analysis, one can observe that our bounds give faster convergence rates of the orders $1/T$ and $1/n$ in terms of number of tasks $T$ and number of samples per task $n$, receptively.



\subsection{Organization}

The paper is organized as follows: \sref{sec:Talagrand-Inequalities} lays the foundations for our analysis by considering a Talagrand-type concentration inequality suitable for deriving our bounds. Next, in \sref{sec:GeneralizationBounds}, after suitably defining \acp{LRC} for \ac{MTL} hypothesis spaces, we provide our \ac{LRC}-based \ac{MTL} excess risk bounds. Based on these bounds, we follow up this section with a local analysis of linear \ac{MTL} frameworks, in which task-relatedness is presumed and enforced by imposing a norm constraint. In more detail, leveraging off the Fenchel-Young and H\"{o}lder inequalities, \sref{sec:StrongConvexity} presents generic upper bounds for the relevant \ac{LRC} of any strongly convex as well as any norm regularized hypothesis classes. These results are subsequently specialized to the case of group norm, Schatten norm and graph-regularized linear \ac{MTL}. Then, \sref{sec:StrongConvexityExcessRisk} supplies the corresponding excess risk bounds based on the \ac{LRC} bounds of mentioned hypothesis classes. The paper concludes with \sref{sec:Discussion}, which compares side by side the \ac{GRC}- and \ac{LRC}-based excess risk bounds for the aforementioned hypothesis spaces, as well as two additional related works based on \ac {GRC} analysis.

\subsection{Previous Related Works}
\label{sec:PreviousRelatedWorks}


An earlier work by \cite{maurer2006}, which considers linear \ac{MTL} frameworks for binary classification, investigates the generalization guarantees based on Rademacher averages. In this framework, all tasks are pre-processed by a common bounded linear operator, and operator norm constraints are used to control the complexity of the associated hypothesis spaces. The \ac{GRC}-based error bounds derived are of order $O(1/\sqrt{nT})$ for both distribution-dependent and data-dependent cases. Another study \citep{maurer2006Graph}, provides bounds for the empirical and expected Rademacher complexities of linear transformation classes. Based on H\"{o}lder's inequality, \ac{GRC}-based risk bounds of order $O(1/\sqrt{nT})$ are established for \ac {MTL} hypothesis spaces with graph-based and $L_{S_{q}}$-Schatten norm regularizers, where $q \in \{2\} \cup [4, \infty]$.

The subject of \ac{MTL} generalization guarantees experienced renewed attention in recent years. \cite{kakade2012} take advantage of the strongly-convex nature of certain matrix-norm regularizers to easily obtain generalization bounds for a variety of machine learning problems. Part of their work is devoted to the realm of online and off-line \ac{MTL}. In the latter case, which pertains to the focus of our work, the paper provides a distribution-dependent \ac{GRC}-based excess risk bound of order $O(1/ \sqrt{nT})$. Moreover, \cite{Maurer2013} present a global Rademacher complexity analysis leading to both data and distribution-dependent excess risk bounds of order $O(\sqrt{\log(nT) / nT}$ for a trace norm regularized \ac{MTL} model. Also, \cite{Maurer2014} examines the bounding of (global) Gaussian complexities of function classes that result from considering composite maps, as it is typical in \ac{MTL} among other settings. An application of the paper's results yields \ac{MTL} risk bounds of order $O( 1/\sqrt{nT})$. More recently, \cite{maurer2015} presents excess risk bounds of order $O(1/\sqrt{nT})$ for both \ac{MTL} and \ac{LTL} settings and reveals conditions, under which \ac{MTL} is more beneficial over learning tasks independently.

Finally, due to being domains related to \ac{MTL}, but, at the same time, less connected to the focus of this paper, we only mention in passing a few works that pertain to generalization guarantees in the realm of life-long learning and domain adaptation. Generalization performance analysis in life-long learning has been investigated by \cite{thrun1996,ben2003a,ben2008b,pentina2015a} and \cite{pentina2015b}. Also, in the context of domain adaptation, similar considerations are examined by \cite{Mansour2009, Mansour2009a, Mansour2009b, Cortes2011, Zhang2012, Mansour2013} and \cite{Cortes2014}.

\subsection{Basic Assumptions \& Notation}


Consider $T$ supervised learning tasks sampled from the same input-output space $\mathcal{X} \times \mathcal{Y}$. Each task $t$ is represented by an independent random variable $(X_{t},Y_{t})$ governed by a probability distribution $\mu_{t}$. Also, the \iid\ samples related to each task $t$ are described by the sequence $(X_{t}^{i},Y_{t}^{i})_{i=1}^{n}$, which is distributed according to $\mu_{t}$. 

In what follows, we use the following notational conventions: vectors and matrices are depicted in boldface. The superscript $T$, when applied to a vector/matrix, denotes the transpose of that quantity. We define $\mathbb{N}_T := \left\{  1, \ldots, T \right\} $. For any random variables $X,Y$ and function $g$ we use $\ebb g(X,Y)$ and $\ebb_Xg(X,Y)$  to denote the expectation with respect to (w.r.t.) all the involved random variables and the conditional expectation w.r.t. the random variable $X$. For any vector-valued function $\boldsymbol{f}=(f_1,\ldots,f_T)$, we introduce the following two notations:
$$
   P \boldsymbol{f} := \frac{1}{T} \sum_{t=1}^{T} P f_{t}= \frac{1}{T} \sum_{t=1}^{T} \mathbb{E} (f(X_{t})),
 \qquad
 P_{n}\boldsymbol{f} := \frac{1}{T} \sum_{t=1}^{T} P_{n} f_{t}= \frac{1}{T} \sum_{t=1}^{T} \frac{1}{n} \sum_{i=1}^{n} f(X_{t}^{i}).
$$
We also denote $\boldsymbol{f}^\alpha=(f_1^\alpha,\ldots,f_T^\alpha),\forall\alpha\in\rbb$.
For any loss function $\ell$ and any $\boldsymbol{f}=(f_1,\ldots,f_T)$ we define $\ell_{\boldf} = (\ell_{f_{1}},\ldots,\ell_{f_{T}})$ where $\ell_{f_t}$ is the function defined by $\ell_{f_t}((X_t,Y_t))=\ell(f_t(X_t),Y_t)$. 

Finally, let us mention that, in the subsequent material, measurability of functions and suprema is assumed whenever necessary. Additionally, in later material, operators on separable Hilbert spaces are assumed to be of trace class. 

\section{Talagrand-Type Inequality for Multi-Task Learning}
\label{sec:Talagrand-Inequalities}



The derivation of our LRC-based error bounds for \ac{MTL} is founded on the following modified Talagrand's concentration inequality adapted to the context of \ac{MTL}, showing that the uniform deviation between the true and empirical means in a vector-valued function class $\fcal$ can be dominated by the associated \emph{multi-task Rademacher complexity} plus a term involving the variance of functions in $\fcal$. We defer the proof of this theorem in Appendix \ref{A}.

 \begin{thrm}[\textsc{Talagrand-Type Inequality for MTL}]
 \label{TalagrandforMTL}
 Let $\mathcal{F}=\{\boldsymbol{f} : =(f_1,\ldots,f_T)\}$ be a class of vector-valued functions satisfying $\sup_{t,x}|f_t(x)|\leq b$. Let $X : =(X^{i}_{t})_{(t,i)=(1,1)}^{(T,N_{t})}$ be a vector of $\sum_{t=1}^{T} N_{t}$ independent random variables where $X_{t}^{1},\ldots, X_{t}^{n},\forall t$ are identically distributed. Let $\{\sigma_t^i\}_{t,i}$ be a sequence of independent Rademacher variables. If $ \frac{1}{T} \sup_{\boldsymbol{f} \in \mathcal{F}}  \sum_{t=1}^{T}  \mathbb{E} \left[ f_t(X^{1}_{t})\right]^2 \leq r$, then for every $x > 0$, with probability at least $1-e^{-x}$,
     \begin{align}
     \label{TalagrandforMTLInq}
       \sup_{\boldsymbol{f} \in \mathcal{F}} (P \boldsymbol{f}-P_{n}\boldsymbol{f}) \leq 4 \mathfrak{R}
       (\mathcal{F}) + \sqrt{\frac{8xr}{nT} }+ \frac{12bx}{nT},
     \end{align}
 \noindent
 where $n:=\min_{t \in \mathbb{N}_{T}} N_{t}$, and the multi-task Rademacher complexity of function class $\mathcal{F}$ is defined as
    \begin{align}
    \label{RCDefinition}
    \mathfrak{R} (\mathcal{F}) := \mathbb{E}_{X, \sigma} \left\lbrace \sup_{\boldsymbol{f}=(f_1,\ldots,f_T) \in \mathcal{F}} \frac{1}{T}  \sum_{t=1}^{T} \frac{1}{N_{t}} \sum_{i=1}^{N_{t}} \sigma_{t}^{i} f_{t}(X_{t}^{i}) \right\rbrace .
    \nonumber
    \end{align}
   Note that the same bound also holds for $ \sup_{\boldsymbol{f} \in \mathcal{F}} (P_n \boldsymbol{f}-P \boldsymbol{f})$.
\end{thrm}

In Theorem \ref{TalagrandforMTL}, the data from different tasks are assumed to be mutually independent, which is typical in \ac{MTL} setting \citep{maurer2006}. To present the results in a clear way, we always assume in the following that the available data for each task is the same, namely $n$.
 
\begin{Remark}
At this point, we would like to present the result of the above theorem for the special case $T=1$, which corresponds to the traditional single task learning framework. It is very easy to verify that for $T=1$, the bound in \eqref{TalagrandforMTLInq} can be written as
\begin{equation}
\label{OurTalagrand}
 \sup_{f \in \mathcal{F}} (Pf - P_{n}f) \leq 4 \mathfrak{R}(\mathcal{F}) + \sqrt{\frac{8xr}{n} } + \frac{12bx}{n},
\end{equation}
where the function $f$ is chosen from an scalar-valued function class $\fcal$. This bound can be compared to the result of Theorem 2.1 of \cite{bartlett2005} (for $\alpha=1$), which reads as  
 \begin{equation}
 \label{BartlettTalagrand}
 \sup_{f \in \mathcal{F}} (P f- P_{n} f) \leq  4 \mathfrak{R}(\mathcal{F}) + \sqrt {\frac {2xr} {n} } +   \frac{8bx}{n}.
  \end{equation}
  Note that the difference between the constants in \eqref{OurTalagrand} and \eqref{BartlettTalagrand} is due to the fact that we failed in directly applying Bousquet's version of Talagrand inequality---similar to what has been done by \cite{bartlett2005} for scalar-valued functions---to the class of vector-valued functions. To make it more clear, let $Z$ be defined as \eqref{Z} with the jackknife replication $Z_{s,j}$ for which a lower bound $Z^{''}_{s,j}$ can be found such that $ Z^{''}_{s,j} \leq Z-Z_{s,j}$. Then, in order to apply Theorem 2.5 of \cite{bousquet2002Thesis}'s work, one needs to show that the quantity $\frac{1}{nT} \sum_{s=1}^{T} \sum_{j=1}^{n} \ebb_{s,j} [(Z^{''}_{s,j})^{2}]$ is bounded. This goal, ideally, can be achieved by including a constraint similar to $ \frac{1}{T} \sup_{\boldsymbol{f} \in \mathcal{F}}  \sum_{t=1}^{T}  \mathbb{E} \left[ f_t(X^{1}_{t})\right]^2 \leq r$ in \thrmref{TalagrandforMTL}. However, we could not come up with any obvious and meaningful way---appropriate for \ac{MTL}---of defining this constraint to satisfy the boundedness condition $\frac{1}{nT} \sum_{s=1}^{T} \sum_{j=1}^{n} \ebb_{s,j} [(Z^{''}_{s,j})^{2}]$ in terms of $r$.
  We would like emphasize that the key ingredient to the proof of \thrmref{TalagrandforMTL} is the so-called \emph{Logarithmic Sobolev} inequality---\thrmref{SobolevInq}---which can be considered as the exponential version of \emph{Efron-Stein} inequality.
\end{Remark}

\section{Excess \ac{MTL} Risk Bounds based on Local Rademacher Complexities}
\label{sec:GeneralizationBounds}

The cornerstone of \sref{sec:Talagrand-Inequalities}'s results is the presence of an upper bound of an empirical process's variance (the second term in the right-hand side of \eqref{TalagrandforMTLInq}). In this section, we consider the Rademacher averages associated with a smaller subset of the function class $\mathcal{F}$ and use them as a complexity term in the context of excess risk bounds. As pointed out by \cite{bartlett2005}, these (local) averages are always smaller than the corresponding global Rademacher averages and allow for eventually deriving sharper generalization bounds. Herein, we exploit this very fact for \ac{MTL} generalization guarantees. 


Theorem \ref{TalagrandforMTL} motivates us to extend the definition of classical \ac{LRC} $\mathfrak{R}(\mathcal{F}^{sclr},r)$ for a scalar-valued function class $\mathcal{F}^{sclr}$ as 
$$\mathfrak{R}(\mathcal{F}^{sclr},r) := \mathbb{E}_{X,\sigma} \big[\sup_{f \in \mathcal{F}^{sclr}, V(f) \leq r} \frac{1}{n} \sum_{i=1}^{n} \sigma_{i} f(X_i) \big].$$
to the \ac {MT-LRC} using the following definition.
\begin{defin}[\textsc{Multi-Task Local Rademacher Complexity}]
\label{ELRCBound}
For a vector-valued function class $\mathcal{F}$, the \emph{Local Rademacher Complexity} $\mathfrak{R}(\mathcal{F},r)$ and its empirical counterpart $\hat{\mathfrak{R}}(\mathcal{F}, r)$ are defined as
\begin{equation}
 \begin{aligned}
\label{MT-LRC}
  \mathfrak{R}(\mathcal{F},r) & := \mathbb{E} \bigg[ \sup_{\substack{\boldsymbol{f}=(f_{1},\ldots,f_{T}) \in \mathcal{F}\\ V(\boldsymbol{f}) \leq r }} \frac{1}{nT} \sum_{t=1}^T\sum_{i=1}^{n} \sigma_t^{i} f_t(X_t^{i}) \bigg],
  \\
  \hat{\mathfrak{R}}(\mathcal{F}, r) & := \mathbb{E}_{\sigma} \bigg[  \sup_{\substack{\boldsymbol{f}=(f_{1},\ldots,f_{T}) \in \mathcal{F}\\ V_{n}(\boldsymbol{f}) \leq r }} \frac{1}{nT} \sum_{t=1}^T\sum_{i=1}^{n} \sigma_{t}^{i} f_t(X_{t}^{i}) \bigg],
  \end{aligned}
\end{equation}
where $V(\boldsymbol{f})$ and $V_{n}(\boldsymbol{f})$ are upper bounds on the variances and empirical variances of the functions in $\mathcal{F}$, respectively. This paper makes the choice $V(\boldsymbol{f})=P\boldsymbol{f}^2$ and $V_n(\boldsymbol{f})=P_n\boldsymbol{f}^2$ where
\begin{align}
P     \boldf^{2} & := \frac{1}{T} \sum_{t=1}^{T} P f_{t}^{2} = \frac{1}{T} \sum_{t=1}^{T} \ebb \left[  f_{t}(X_{t})\right] ^{2},
\nonumber\\
P_{n} \boldf^{2} & : = \frac{1}{T} \sum_{t=1}^{T} P_{n} f_{t}^{2} = \frac{1}{T} \sum_{t=1}^{T} \frac{1}{n} \sum_{i=1}^{n} (f_{t} (X_{t}^{i}))^{2}.
\nonumber
\end{align}
\end{defin}
%

Analogous to single task learning, the challenge in applying MT-LRCs in \eqref{MT-LRC} to refine the existing learning rates is to find an optimal radius trading-off the size of the set $\{\boldsymbol{f}\in\fcal:V(\boldsymbol{f})\leq r\}$ and its complexity, which, as we show later, reduces to the calculation of the fixed-point of a sub-root function.

\begin{defin}[\textsc{Sub-Root Function}]
\label{SubRootFunction}
A function $\psi : [0, \infty] \rightarrow [0 , \infty]$ is sub-root if
\begin{enumerate}
\item $\psi$ is non-negative,
\item $\psi$ is non-decreasing,
\item $r \mapsto \psi(r)/ \sqrt{r}$ is non-increasing for $r > 0$.
\end{enumerate}
\end{defin}
The following lemma is an immediate consequence of the above definition.
 \begin{lemm}[Lemma 3.2 in \cite{bartlett2005}]
 If $\psi$ is a sub-root function, then it is continuous on $[0, \infty]$, and the equation $\psi(r) = r$ has a unique (non-zero) solution which is known as the fixed point of $\psi$ and it is denotes by $r^{*}$. Moreover, for any $r > 0$, it holds that $r > \psi(r)$ if and only if $r^{*} \leq r$.
 \end{lemm}
 We will see later that this fixed point plays a key role in the local error bounds.
 
The definition of local Rademacher complexity is based on the fact that by incorporating the variance constraint, a better error rate for the bounds can be obtained. In other words, the key point in deriving fast rate bounds is that around the best function $f^{*}$ (the function that minimizes the true risk) in the class, the variance of the difference between the true and empirical errors of functions is upper bounded by a linear function of the expectation of this difference. We will call a class with this property a \emph{Bernstein} class, and we provide a definition of a vector-valued Berstein class $\fcal$ as following.

\begin{defin}[\textsc{Vector-Valued Bernstein Class}]
\label{BernsteinClass}
A vector-valued function class $\fcal$ is said to be a $(\beta, B)$-Bernstein class with respect to the probability measure $P$, if for every $0 < \beta \leq 1$, $B \geq 1$ and any $\boldsymbol{f} \in \fcal$, there exists a function $V:\fcal\to\rbb^+$ such that
\begin{align}
P \boldsymbol{f}^{2} \leq V(\boldsymbol{f}) \leq B (P \boldsymbol{f})^\beta.
\label{BernsteinCondition}
\end{align}
\end{defin}

It can be shown that the Bernstein condition \eqref{BernsteinCondition} is not too restrictive and it holds, for example, for non-negative bounded functions with respect to any probability distribution \citep{bartlett2004local}. Other examples include the class of excess risk functions $\mathcal{L}_{\fcal} := \{\ell_{f} - \ell_{f^{*}} : f \in \fcal \}$---with $f^{*} \in \fcal$ the minimizer of $P \ell_{f}$--- when the function class $\fcal$ is convex and the loss function $\ell$ is strictly convex. 

In this section, we show that under some mild assumptions on a vector-valued Bernstein class of functions, the \ac {LRC}-based excess risk bounds can be established for \ac {MTL}. We suppose that the loss function $\ell$ and the vector-valued hypothesis space $\mathcal{F}$ satisfy the following conditions:

\begin{assum}
\label{assumption}\leavevmode
\begin{enumerate}
\item \label{first} There is a function $\boldsymbol{f}^{*}=(f_1^{*},\ldots,f_T^{*}) \in \mathcal{F}$ satisfying $P\ell_{\boldsymbol{f}^{*}} = \inf_{\boldsymbol{f} \in \mathcal{F}} P\ell_{\boldsymbol{f}}$.
\item \label{second} There is constant $B' \geq 1$, such that for every $\boldf \in \mathcal{F}$ we have $P(\boldf - \boldf^{*})^{2}  \leq B' P (\ell_{\boldf} - \ell_{\boldf^{*}})$.
\item \label{third} There is a constant $L$, such that the loss function $\ell$ is $L$-Lipschitz in its first argument.
\end{enumerate}
\end{assum}
As it has been pointed out by \cite{bartlett2005}, there are many examples of regularized algorithms for which these conditions can be satisfied. More specifically, a uniform convexity condition on the loss function $\ell$ is usually sufficient to satisfy \assumref{assumption}.2. As an example for which this assumption holds, \cite{bartlett2005} refereed to the quadratic loss function $\ell(f(X),Y) = (f(X) - Y)^{2}$ when the functions $f \in \fcal$ are uniformly bounded, More specifically, if for all $f \in \fcal$, $X \in \xcal$ and $Y \in \mathcal{Y}$, it holds that $|f(X) - Y| \in [0,1]$, then it can be shown that the conditions of \assumref{assumption} are met with $L=1$ and $B=1$.
We now present the main result of this section showing that the excess error of MTL can be bounded by the fixed-point of a sub-root function dominating the  \ac {MT-LRC}. The proof of the results is provided in Appendix \ref{B}.
\begin{thrm}[Distribution-dependent excess risk bound for \ac {MTL}]
\label{MainTheorem}
 Let  $\mathcal{F}:=\{\boldsymbol{f} : =(f_1,\ldots,f_T): \forall t, f_{t} \in \rbb^{\mathcal{X}}\}$ be a class of vector-valued functions $\boldf$ satisfying $\sup_{t,x}|f_t(x)|\leq b$. Also, Let $X : =(X^{i}_{t}, Y_{t}^{i})_{(t,i)=(1,1)}^{(T,n)}$ be a vector of $nT$ independent random variables where for each task $t$, $(X_{t}^{1},Y_{t}^{1})\ldots, (X_{t}^{n},Y_{t}^{n})$ be identically distributed. Suppose that \assumref{assumption} holds. Define $\fcal^{*} := \{\boldf - \boldf^{*}\}$, where $\boldf^{*}$ is the function satisfying $P\ell_{\boldsymbol{f}^{*}} = \inf_{\boldsymbol{f} \in \mathcal{F}} P\ell_{\boldsymbol{f}}$. Let $B := B' L^{2}$ and $\psi$ be a sub-root function with the fixed point $r^*$ such that $B L \mathfrak{R}(\mathcal{F}^{*}, r) \leq \psi(r),\forall r \geq r^{*}$, where $\mathfrak{R}(\mathcal{F}^{*}, r)$ is the \ac {LRC} of the functions class $\mathcal{F}^{*}$:
\begin{align}
\label{LRCofFstar}
\mathfrak{R}(\mathcal{F}^{*}, r):= \mathbb{E}_{X, \sigma} \Big[ \sup_{\substack{\boldsymbol{f} \in \mathcal{F}, \\ L^{2}P(\boldsymbol{f}-\boldsymbol{f}^{*})^{2} \leq r}} \frac{1}{nT} \sum_{t=1}^{T} \sum_{i=1}^{n} \sigma_{t}^{i} f_t(X_{t}^{i})  \Big].
\end{align}
Then, we obtain the following bounds in terms of the fixed point $r^{*}$ of $\psi(r)$:
\begin{enumerate}
\item For any function class $\mathcal{F}$, $K > 1$ and $x > 0$, with probability at least $1-e^{-x}$,
\begin{align}
\label{MainIneq:non-convex}
\forall \boldf \in \fcal \qquad P (\ell_{{\boldsymbol{f}}} - \ell_{\boldsymbol{f}^{*}})  \leq \frac{K}{K-1} P_n (\ell_{{\boldsymbol{f}}} - \ell_{\boldsymbol{f}^{*}}) +  \frac{560 K}{B} r^{*} + \frac{(48Lb+ 28 B K)x}{n T}.
\end{align}

\item For any \underline{convex} function class $\mathcal{F}$, $K > 1$ and $x > 0$, with probability at least $1-e^{-x}$,
\begin{align}
\label{MainIneq:convex}
\forall \boldf \in \fcal \qquad P (\ell_{{\boldsymbol{f}}} - \ell_{\boldsymbol{f}^{*}})  \leq \frac{K}{K-1} P_n (\ell_{{\boldsymbol{f}}} - \ell_{\boldsymbol{f}^{*}}) +  \frac{32 K}{B} r^{*} + \frac{(48Lb+ 16 B K)x}{n T}.
\end{align}
\end{enumerate}
\end{thrm}

\begin{corr}
\label{DistDependentcorr}
Let $\boldsymbol{\hat{f}}$ be any element of convex class $\mathcal{F}$ satisfying $P_n \ell_{\boldsymbol{\hat{f}}} = \inf_{\boldsymbol{f} \in \mathcal{F}} P_n \ell_{\boldsymbol{f}}$. Assume that the conditions of \thrmref{MainTheorem} hold. Then for any $x>0$ and $r > \psi(r)$, with probability at least $1 - e^{-x}$,
\begin{align}
\label{DistDependentcorrInq}
P (\ell_{{\boldsymbol{\hat{f}}}} - \ell_{\boldsymbol{f}^{*}}) \leq  \frac{32 K}{B} r^{*} + \frac{(48Lb+ 16 B K)x}{n T}.
\end{align}
\end{corr}
\begin{proof}
The results follows by noticing that $P _{n}(\ell_{{\boldsymbol{\hat{f}}}} - \ell_{\boldsymbol{f}^{*}}) \leq 0$.
\end{proof}
The next theorem, analogous to Corollary 5.4 of \cite{bartlett2005}, presents a data-dependent version of \eqref{DistDependentcorrInq} replacing the Rademacher complexity in \corref{DistDependentcorr} with its empirical counterpart. The proof of this Theorem, which repeats the same basic steps utilized by Theorem 5.4 of \cite{bartlett2005}, can be found in Appendix \ref{B}.

%
\begin{thrm}[Data-dependent excess risk bound for \ac {MTL}]
\label{MainTheorem2}
  Let $\boldsymbol{\hat{f}}$ be any element of convex class $ \mathcal{F}$ satisfying $P_n \ell_{\boldsymbol{\hat{f}}} = \inf_{\boldsymbol{f} \in \mathcal{F}} P_n \ell_{\boldsymbol{f}}$. Assume that the conditions of \thrmref{MainTheorem} hold. Define
\begin{align}
\hat{\psi}_{n}(r) = c_{1} \hat{\mathfrak{R}}(\mathcal{F}^{*}, c_{3}r) + \frac{c_{2} x}{nT},\quad
\hat{\mathfrak{R}}(\mathcal{F}^{*}, c_{3} r) := \mathbb{E}_{\sigma} \Big[ \sup_{\substack{\boldsymbol{f} \in \mathcal{F},\\ L^{2} P_n (\boldsymbol{f}- \boldsymbol{\hat{f}})^{2} \leq c_{3} r}} \frac{1}{nT} \sum_{t=1}^{T} \sum_{i=1}^{n} \sigma_{t}^{i} f_t(X_{t}^{i})  \Big],
\nonumber
\end{align}
where $c_{1} = 2L \max\left( B,16Lb\right)$, $c_{2} = 128 L^{2} b^{2} +2b c_{1}$ and $c_{3} = 4+128 K+4B(48Lb+16BK)/c_{2}$. Then for any $K > 1$ and $x > 0$, with probability at least $1 - 4e^{-x}$, we have
\begin{align}
P (\ell_{{\boldsymbol{\hat{f}}}} - \ell_{\boldsymbol{f}^{*}})  \leq    \frac{32 K}{B} \hat{r}^{*} + \frac{(48 Lb+ 16 B K)x}{n T},
\nonumber
\end{align}
where $\hat{r}^{*}$ is the fixed point of the sub-root function $\hat{\psi}_{n} (r)$.
\end{thrm}
 An immediate consequence of the results of this section is that one can derive excess risk bounds for given regularized \ac {MTL} hypothesis spaces. In the next section, by further bounding the fixed point $r^{*}$ in \corref{DistDependentcorr} (and $\hat{r}^{*}$ in \thrmref{MainTheorem2}), we will derive distribution (and data)-dependent excess risk bounds for several commonly used norm-regularized \ac{MTL} hypothesis spaces. 

%
%
%
%
%
%
%
%
%

\section{Local Rademacher Complexity Bounds for Norm Regularized \ac{MTL} Models}
\label{sec:StrongConvexity}
This section presents very general distribution-dependent MT-LRC bounds for hypothesis spaces defined by norm as well as strongly convex regularizers, which allows us to immediately derive, as specific application cases, LRC bounds for  group-norm, Schatten-norm, and graph-regularized MTL models. It should be mentioned that similar data-dependent MT-LRC bounds can be easily obtained by a similar deduction process.


\subsection{Preliminaries}
We consider linear MTL models where we associate to each task-wise function $f_t$ a weight $\boldsymbol{w}_t \in \mathcal{H}$ by $f_t(X) = \langle  \boldsymbol{w}_t, \phi(X) \rangle $. Here $\phi$ is a feature map associated to a Mercer kernel $k$ satisfying $k(X,\tilde{X})=\langle\phi(X),\phi(\tilde{X})\rangle,\forall X,\tilde{X}\in\xcal$ and $\boldsymbol{w}_t$ belongs to the \emph{reproducing kernel Hilbert space} $\hcal_K$ induced by $k$ with inner product $\langle\cdot,\cdot\rangle$. We assume that the multi-task model $\boldsymbol{W} = (\boldsymbol{w}_1, \ldots, \boldsymbol{w}_T) \in \mathcal{H} \times \ldots \times \mathcal{H}$ is learned by a regularization scheme:
\begin{align}
\label{GeneralFixedMappingModel}
\min_{\boldsymbol{W}}  \Omega \left(  \boldsymbol{W} \right)   +  C \sum_{t=1}^{T}  \sum_{i=1}^{n} \ell( \left\langle \boldsymbol{w}_t , \phi(X_t^{i})\right\rangle_{\mathcal{H}} , Y_t^{i}),
\end{align}
where the regularizer $\Omega(\cdot)$ is used to enforce information sharing among tasks. This regularization scheme amounts to performing \emph{\ac {ERM}}  in the hypothesis space
\begin{align}
\label{GeneralMTLHS}
\mathcal{F} := \left\lbrace  X \mapsto [\left\langle \boldsymbol{w}_1, \phi (X_1)\right\rangle , \ldots, \left\langle \boldsymbol{w}_T, \phi (X_T)\right\rangle ]^{T} :  \Omega (\boldsymbol{D}^{1/2} \boldsymbol{W})   \leq R^{2} \right\rbrace,
\end{align}
where $\boldsymbol{D}$ is a given positive operator defined in $\mathcal{H}$. Note that the hypothesis spaces corresponding to group and Schatten norms can be recovered by taking $\boldsymbol{D} = \boldsymbol{I}$, and choosing their associated norms. More specifically, by choosing $\Omega(\boldsymbol{W}) =\frac{1}{2} \|\boldsymbol{W}\|_{2,q}^{2}$, we can retrieve an $L_{2,q}$-norm hypothesis space in \eqref{GeneralMTLHS}. Similarly, the choice $\Omega(\boldsymbol{W}) =\frac{1}{2} \|\boldsymbol{W}\|_{S_q}^{2}$ gives an $L_{S_{q}}$-Schatten norm hypothesis space in \eqref{GeneralMTLHS}. Furthermore, the graph-regularized \ac{MTL} \citep{micchelli2004kernels,evgeniou2005learning,maurer2006Graph} can be specialized by taking $\Omega(\boldsymbol{W}) =\frac{1}{2} \|\boldsymbol{D}^{1/2} \boldsymbol{W}\|_{F}^{2}$, wherein $\|.\|_{F}$ is a Frobenius norm, and $\boldsymbol{D} = \boldsymbol{L} + \eta \boldsymbol{I}$ with $\boldsymbol{L}$ being a graph-Laplacian, and $\eta$ being a regularization constant. On balance, all these \ac{MTL} models can be considered as norm-regularized models. Also, for specific values of $q$ in group and Schatten norm cases, it can be shown that they are strongly convex.

\subsection{General Bound on the LRC}

Now, we can provide the main results of this section, which give general \ac {LRC} bounds for any \ac {MTL} hypothesis space of the form \eqref{GeneralMTLHS} in which $\Omega (\boldsymbol{W})$ is given as a strongly convex or a norm function of $\boldsymbol{W}$.
\begin{thrm}[Distribution-dependent \ac {MT-LRC} bounds by strong convexity]
\label{GeneralLRCBoundForFixedMappingStrongConvex}
Let $\Omega(\boldsymbol{W})$ in \eqref{GeneralFixedMappingModel} be $\mu$-strongly convex with $\Omega^{*}(\boldsymbol{0}) =0$ and $\| k \|_{\infty} \leq \mathcal{K}  \leq \infty$. Let $X_t^{1},\ldots, X_t^{n}$ be an \iid \,sample drawn from $P_t$. Also, assume that for each task $t$, the eigenvalue-eigenvector decomposition of the Hilbert-Schmidt covariance operator $J_t$ is given by $J_t := \mathbb{E}(\phi(X_t )\otimes \phi(X_t)) = \sum_{j=1}^{\infty} \lambda_t^j \boldsymbol{u}_{t}^j \otimes \boldsymbol{u}_{t}^{j}$, where $(\boldsymbol{u}_{t}^{j})_{j=1}^{\infty}$ forms an orthonormal basis of $\mathcal{H}$, and $(\lambda_t^j)_{j=1}^{\infty}$ are the corresponding eigenvalues, arranged in non-increasing order. Then for any given positive operator $\boldsymbol{D}$ on $\mathbb{R}^{T}$, any $r > 0$ and any non-negative integers $h_1, \ldots, h_T$:
 \begin{align}
 \label{LRCforStrongConvexRegularizedMTL}
 \mathfrak{R}(\mathcal{F},r) \leq \min_{ \{0 \leq h_t \leq \infty\}_{t=1}^{T}} \left\lbrace \sqrt{  \frac{ r \sum_{t=1}^{T} h_t} {nT}  } + \frac{R}{T} \sqrt{ \frac{2}{\mu} \mathbb{E}_{X, \sigma} \left\|  \boldsymbol{D}^{-1/2} \boldsymbol{V} \right\|_{*}^{2}} \right\rbrace,
 \end{align}
 \noindent
 where $\boldsymbol{V} = \left( \sum_{j > h_t}  \left\langle \frac{1}{n} \sum_{i=1}^{n} \sigma_{t}^{i} \phi(X_t^{i}), \boldsymbol{u}_{t}^{j} \right\rangle \boldsymbol{u}_{t}^{j}\right)_{t=1}^{T}$.
 \end{thrm}
 \begin{proof}
%
Note that with the help of \ac{LRC} definition, we have for any function class $\mathcal{F}$,
   \begin{align}
   \mathfrak{R}(\mathcal{F},r) &= \frac{1}{nT} \mathbb{E}_{X,\sigma} \left\lbrace  \sup_{\substack{\boldsymbol{f}=(f_{1},\ldots,f_{T}) \in \mathcal{F},\\ P \boldsymbol{f}^{2} \leq r}}  \sum_{i=1}^{n} \left\langle  \left( \boldsymbol{w}_{t}\right)_{t=1}^{T} , \left( \sigma_t^{i} \phi (X_t^{i}) \right)_{t=1}^{T} \right\rangle    \right\rbrace
   \nonumber\\
   & = \frac{1}{T} \mathbb{E}_{X,\sigma} \left\lbrace  \sup_{\substack{\boldsymbol{f} \in \mathcal{F},\\ P \boldsymbol{f}^{2} \leq r}}  \left\langle   \left( \boldsymbol{w}_{t}\right)_{t=1}^{T} ,   \left( \sum_{j=1}^{\infty}  \left\langle \frac{1}{n} \sum_{i=1}^{n}  \sigma_t^{i} \phi (X_t^{i}) , \boldsymbol{u}_{t}^{j}   \right\rangle  \boldsymbol{u}_{t}^{j}  \right)_{t=1}^{T}  \right\rangle    \right\rbrace
   \nonumber\\
    \label{A_1DefinitionStrongConvexity}
    & \leq \frac{1}{T} \mathbb{E}_{X,\sigma} \left\lbrace  \sup_{P \boldsymbol{f}^{2} \leq r} \left\langle  \left( \sum_{j=1}^{h_t} \sqrt{\lambda_t^{j}} \left\langle \boldsymbol{w}_{t}, \boldsymbol{u}_{t}^{j}\right\rangle  \boldsymbol{u}_{t}^{j} \right)_{t=1}^{T},  \right. \right. \nonumber\\
    & \qquad \qquad \qquad \qquad \; \: \left. \left.      \left( \sum_{j=1}^{h_t}  {\sqrt{\lambda_t^{j}}}^{-1} \left\langle \frac{1}{n} \sum_{i=1}^{n}  \sigma_t^{i} \phi (X_t^{i}) , \boldsymbol{u}_{t}^{j} \right\rangle  \boldsymbol{u}_{t}^{j} \right)_{t=1}^{T}   \right\rangle    \right\rbrace  \\
    \label{A_2DefinitionStrongConvexity}
    & + \frac{1}{T} \mathbb{E}_{X,\sigma} \left\lbrace \sup_{\boldsymbol{f} \in \mathcal{F}} \left\langle   \left( \boldsymbol{w}_t \right)_{t=1}^{T} , \left( \sum_{j>h_t}  \left\langle \frac{1}{n}  \sum_{i=1}^{n} \sigma_{t}^{i} \phi (X_t^{i}) , \boldsymbol{u}_{t}^{j}  \right\rangle  \boldsymbol{u}_{t}^{j} \right)_{t=1}^{T}  \right\rangle   \right\rbrace \\
    & = A_{1} + A_{2}.
    \nonumber
   \end{align}
   \noindent
where in the last equality, we defined the term in \eqref{A_1DefinitionStrongConvexity} as $A_1$, and the term in \eqref{A_2DefinitionStrongConvexity} as $A_{2}$.
  
  \textbf{Step 1. Controlling $A_1$:} Applying Cauchy-Schwartz (C.S.) inequality on $A_1$ yields the following
   \begin{align}
   A_1 & \leq \frac{1}{T} \mathbb{E}_{X,\sigma} \left\lbrace  \sup_{P \boldsymbol{f}^{2} \leq r} \left[  \left( \sum_{t=1}^{T} \left\|  \sum_{j=1}^{h_t} \sqrt{\lambda_t^{j}} \left\langle \boldsymbol{w}_t, \boldsymbol{u}_{t}^{j} \right\rangle  \boldsymbol{u}_{t}^{j}  \right\|  ^{2}   \right) ^{\frac{1}{2}}          \right. \right. \nonumber\\
   & \qquad \qquad \qquad \quad \qquad \left. \left.     \left( \sum_{t=1}^{T} \left\|  \sum_{j=1}^{h_t} {\sqrt{\lambda_t^{j}}}^{-1} \left\langle \frac{1}{n} \sum_{i=1}^{n} \sigma_t^{i} \phi (X_t^{i}) , \boldsymbol{u}_{t}^{j} \right\rangle  \boldsymbol{u}_{t}^{j} \right\| ^{2}  \right) ^{\frac{1}{2}}  \right] \right\rbrace
   \nonumber\\
   & = \frac{1}{T} \mathbb{E}_{X,\sigma} \left\lbrace  \sup_{P \boldsymbol{f}^{2} \leq r} \left[  \left( \sum_{t=1}^{T}  \sum_{j=1}^{h_t} \lambda_t^{j} \left\langle \boldsymbol{w}_t, \boldsymbol{u}_{t}^{j} \right\rangle^{2}     \right) ^{\frac{1}{2}}   \right. \right. \nonumber\\
   & \qquad \qquad \qquad \qquad \quad   \left. \left.       \left( \sum_{t=1}^{T}  \sum_{j=1}^{h_t} {\lambda_t^{j}}^{-1} \left\langle \frac{1}{n} \sum_{i=1}^{n} \sigma_t^{i} \phi (X_t^{i}) , \boldsymbol{u}_{t}^{j} \right\rangle^{2}   \right) ^{\frac{1}{2}}  \right] \right\rbrace.
   \nonumber
   \end{align}
   \noindent
   With the help of Jensen's inequality and regarding the fact that $\mathbb{E}_{X, \sigma} \left\langle \frac{1}{n} \sum_{i=1}^{n} \sigma_t^{i} \phi (X_t^{i}) , \boldsymbol{u}_{t}^{j} \right\rangle^{2} = \frac{\lambda_t^{j}}{n}$ and $P\boldsymbol{f}^{2} \leq r $ implies $\frac{1}{T} \sum_{t=1}^{T} \sum_{j=1}^{\infty} \lambda_t^{j} \left\langle  \boldsymbol{w}_{t}, \boldsymbol{u}_{t}^{j} \right\rangle^{2} \leq r$ (see \lemmref{MTLemma} in the Appendix for the proof), we can further bound $A_1$ as
   \begin{align}
   \label{GeneralA1BoundStrongConvexity}
   A_1 \leq  \sqrt{  \frac{ r \sum_{t=1}^{T} h_t} {nT}  }.
   \end{align}

\textbf{Step 2. Controlling $A_2$:} We use strong convexity assumption on the regularizer in order to further bound the second term $A_2 = \frac{1}{T} \mathbb{E}_{X, \sigma} \left\lbrace  \sup_{\boldsymbol{f} \in \mathcal{F}} \left\langle \boldsymbol{D}^{1/2}  \boldsymbol{W} , \boldsymbol{D}^{-1/2} \boldsymbol{V} \right\rangle \right\rbrace$.

 Let $\lambda >0$. Applying \eqref{FenchelYoungInq}  with $\boldsymbol{w} =  \boldsymbol{D}^{1/2} \boldsymbol{W}$ and $\boldsymbol{v} = \lambda \boldsymbol{D}^{-1/2} \boldsymbol{V}$ gives

 \begin{align}
 & \left\langle  \boldsymbol{D}^{1/2} \boldsymbol{W} ,  \lambda \boldsymbol{D}^{-1/2}  \boldsymbol{V} \right\rangle  \leq \Omega(\boldsymbol{D}^{1/2} \boldsymbol{W})  +  \left\langle \triangledown    \Omega^{*} \left(\boldsymbol{0}\right)  , \lambda \boldsymbol{D}^{-1/2} \boldsymbol{V}     \right\rangle + \frac{\lambda ^{2}}{2 \mu} \left\| \boldsymbol{D}^{-1/2} \boldsymbol{V}  \right\|_{*}^{2}.
 \nonumber
 \end{align}
 Note that, regrading the definition of $\boldsymbol{V}$, we get $\mathbb{E}_{\sigma}\left\langle \triangledown    \Omega^{*} \left(\boldsymbol{0}\right)  , \lambda \boldsymbol{D}^{-1/2} \boldsymbol{V}     \right\rangle = 0$, Therefore, taking supremum and expectation on both sides, dividing throughout by $\lambda$ and $T$, and then optimizing over $\lambda$ gives
 \begin{align}
 \label{A_2BoundStrongConvexity}
 A_2 & = \frac{1}{T} \mathbb{E}_{X, \sigma} \left\lbrace  \sup_{\boldsymbol{f} \in \mathcal{F}} \left\langle \boldsymbol{D}^{1/2}  \boldsymbol{W} , \boldsymbol{D}^{-1/2} \boldsymbol{V} \right\rangle \right\rbrace \leq \min_{0< \lambda < \infty }\left\lbrace  \frac{R^{2}}{\lambda T} + \frac{\lambda}{2 \mu T}  \mathbb{E}_{X, \sigma} \left\| \boldsymbol{D}^{-1/2} \boldsymbol{V}  \right\|_{*}^{2} \right\rbrace
 \nonumber\\
 & = \frac{R}{T} \sqrt{ \frac{2}{\mu} \mathbb{E}_{X, \sigma} \left\| \boldsymbol{D}^{-1/2} \boldsymbol{V}  \right\|_{*}^{2}  }.
 \end{align}
 \noindent
 Combining \eqref{A_2BoundStrongConvexity} with \eqref{GeneralA1BoundStrongConvexity} completes the proof.
 \end{proof}
 \begin{Remark}
 \label{Holder}
 Note that when considering a norm regularized space similar to \eqref{GeneralMTLHS},  more general result can be obtained with the help of H\"{o}lder inequality which holds for any norm regularizer $\Omega (\boldsymbol{D}^{1/2} \boldsymbol{W})$ and not necessarily strongly convex norms. More specially, for any regularizer $\Omega(\boldsymbol{W})$, which is presented as a norm function $\left\| . \right\|$ of $\boldsymbol{W}$, we can derive a general \ac {LRC} bound presented in the following theorem.
 \end{Remark}
 
 \begin{thrm}[Distribution-dependent \ac {MT-LRC} bounds by H\"{o}lder inequality]
 \label{GeneralLRCBoundForFixedMapping2}
Let the regularizer $\Omega(\boldsymbol{W})$ in \eqref{GeneralFixedMappingModel} be given as a norm function in the form of $\|.\|$, where its dual conjugate is denoted by $\left\|. \right\|_{*} $. Let the kernels be uniformly bounded, that is $\| k \|_{\infty} \leq \mathcal{K}  \leq \infty$, and $X_t^{1},\ldots, X_t^{n}$ be an \iid \,sample drawn from $P_t$. Also, assume that for each task $t$, the eigenvalue-eigenvector decomposition of the Hilbert-Schmidt covariance operator $J_{t}$ is given by $J_t := \mathbb{E}(\phi(X_t )\otimes \phi(X_t)) = \sum_{j=1}^{\infty} \lambda_t^j \boldsymbol{u}_{t}^j \otimes \boldsymbol{u}_{t}^{j}$, where $(\boldsymbol{u}_{t}^{j})_{j=1}^{\infty}$ forms an orthonormal basis of $\mathcal{H}$, and $(\lambda_t^j)_{j=1}^{\infty}$ are the corresponding eigenvalues, arranged in non-increasing order. Then for any given positive operator $\boldsymbol{D}$ on $\mathbb{R}^{T}$, any $r > 0$ and any non-negative integers $h_1, \ldots, h_T$:
 \begin{align}
 \label{GeneralBoundHolderInq}
 \mathfrak{R}(\mathcal{F},r) \leq \min_{ 0 \leq h_t \leq \infty} \left\lbrace \sqrt{  \frac{ r \sum_{t=1}^{T} h_t} {nT}  } + \frac{\sqrt{2} R}{T} \mathbb{E}_{X, \sigma} \left\|  \boldsymbol{D}^{-1/2} \boldsymbol{V} \right\|_{*} \right\rbrace,
 \end{align}
 \noindent
  where $\boldsymbol{V} = \left( \sum_{j > h_t}  \left\langle \frac{1}{n} \sum_{i=1}^{n} \sigma_{t}^{i} \phi(X_t^{i}), \boldsymbol{u}_{t}^{j} \right\rangle \boldsymbol{u}_{t}^{j}\right)_{t=1}^{T}$.
  \end{thrm}
  \begin{proof}
  The proof of this theorem repeats the same steps as the proof of \thrmref{GeneralLRCBoundForFixedMappingStrongConvex}, except for controlling term $A_{2}$ in \eqref{A_2DefinitionStrongConvexity}, in which the H\"{o}lder inequality can be efficiently used to further bound it as following
  \begin{align}
    \label{A_2BoundHolder}
    A_2 & =  \frac{1}{T} \mathbb{E}_{X,\sigma} \left\lbrace \sup_{\boldsymbol{f} \in \mathcal{F}}  \left\langle  \left( \boldsymbol{w}_t\right)_{t=1}^{T}  , \left( \sum_{j>h_t}  \left\langle \frac{1}{n}  \sum_{i=1}^{n} \sigma_{t}^{i} \phi (X_t^{i}) , \boldsymbol{u}_{t}^{j} \right\rangle  \boldsymbol{u}_{t}^{j}\right)_{t=1}^{T}  \right\rangle   \right\rbrace 
    \nonumber\\
    & =  \frac{1}{T} \mathbb{E}_{X,\sigma} \left\lbrace \sup_{\boldsymbol{f} \in \mathcal{F}} \left\langle \boldsymbol{D}^{1/2} \boldsymbol{W}, \boldsymbol{D}^{-1/2} \boldsymbol{V} \right\rangle \right\rbrace
    \nonumber\\
    & \stackrel{\text {H\"{o}lder}}{\leq} \frac{1}{T} \mathbb{E}_{X,\sigma} \left\lbrace \sup_{\boldsymbol{f} \in \mathcal{F}} \left\| \boldsymbol{D}^{1/2} \boldsymbol{W} \right\| . \left\| \boldsymbol{D}^{-1/2} \boldsymbol{V} \right\|_{*} \right\rbrace  
    \nonumber\\
    & \leq \frac{\sqrt{2} R}{T} \mathbb{E}_{X,\sigma}  \left\| \boldsymbol{D}^{-1/2} \boldsymbol{V} \right\|_{*}.
    \end{align}
   \end{proof}
   \begin{Remark}
   \label{HolderBetter}
   Notice that, obviously, $\sqrt{2} \mathbb{E}_{X,\sigma}  \left\| \boldsymbol{D}^{-1/2} \boldsymbol{V} \right\|_{*} \leq \sqrt{ \frac{2}{\mu} \mathbb{E}_{X, \sigma} \left\| \boldsymbol{D}^{-1/2} \boldsymbol{V}  \right\|_{*}^{2}  }$ for any $\mu \leq 1$. Interestingly, for the cases considered in our study, it holds that $\mu \leq 1$. More specifically, from Theorem 3 and Theorem 12 in \cite{kakade2012}, it can be shown that $R(\boldsymbol{W}) = 1/2 \left\| \boldsymbol{W}\right\|^{2}_{2,q}$ is $\frac{1}{q^{*}}$-strongly convex \wrt\,the group norm $\left\| .\right\|_{2,q}$. Similarly, using Theorem 10 in \cite{kakade2012}, it can be shown that the regularization function $R(\boldsymbol{W}) = \frac{1}{2} \left\| \boldsymbol{W}\right\|^{2}_{S_q}$ with $q \in [1,2]$ is $(q-1)$-strongly convex \wrt\,the $L_{S_{q}}$-Schatten norm $\left\| .\right\| _{S_q}$. Therefore, given the range of $q$ in $[1,2]$, for which these two norms are strongly convex, it can be easily seen that $\mu \leq \frac{1}{2}$ and $\mu \leq 1$ for the group and Schatten-norm hypotheses, respectively. Therefore, for this cases, H\"{o}lder inequality yields slightly tighter bounds for \ac {MT-LRC}.   
   \end{Remark}
   \begin{Remark}
   \label{HolderisMoreGeneralthanStrongConvexity}
   It is worth mentioning that, when applied to the norm-regularized \ac {MTL} models, the result of \thrmref{GeneralLRCBoundForFixedMapping2} could be more general than that of \thrmref{GeneralLRCBoundForFixedMappingStrongConvex}. More specially, for $L_{2,q}$-group and $L_{S_{q}}$-Schatten norm regularizers, \thrmref{GeneralLRCBoundForFixedMappingStrongConvex} can only be applied to the special case of $q \in [1,2]$, for which these two norms are strongly convex. In contrast, \thrmref{GeneralLRCBoundForFixedMapping2} is applicable to any value of $q$ for these two norms. For this reason and considering the fact that very similar results can be obtained from \thrmref{GeneralLRCBoundForFixedMappingStrongConvex} and \thrmref{GeneralLRCBoundForFixedMapping2} (see \lemmref{A_2ExcpectationGroupNorm} and \remref{HolderExpectation}), we will use \thrmref{GeneralLRCBoundForFixedMapping2} in the sequel to find the \ac{LRC} bounds of several norm regularized \ac{MTL} models.   
   
   \end{Remark}
   
   In what follows, we demonstrate the power of \thrmref{GeneralLRCBoundForFixedMapping2} by applying it to derive the \ac {LRC} bounds for some popular \ac {MTL} models, including group norm, Schatten norm and graph regularized \ac {MTL} models extensively studied in the literature of \ac {MTL} \citep{maurer2006Graph,evgeniou2007,argyriou2007,Argyriou2008,li2013,argyriou2014Graph}.
 \subsection{Group Norm Regularized \ac {MTL}}\label{ssec:GroupNormRegularizer}
 We first consider a group norm regularized MTL capturing the inter-task relationships by the group norm regularizer $\frac{1}{2} \| \boldsymbol{W} \| _{2,q}^{2} := \frac{1}{2} \big(  \sum_{t=1}^{T} \left\|  \boldsymbol{w}_t \right\| _{2} ^{q}  \big)^{2/q}$ \citep{evgeniou2007,Argyriou2008,lounici2009,romera2012}, for which the associated hypothesis space takes the form
 \begin{align}
  \label{GroupHSForFixedMApping}
  \mathcal{F}_{q} := \left\lbrace  X \mapsto [\left\langle \boldsymbol{w}_1, \phi (X_1)\right\rangle , \ldots, \left\langle \boldsymbol{w}_T, \phi (X_T)\right\rangle ]^{T} : \frac{1}{2} \left\| \boldsymbol{W} \right\| _{2,q}^{2} \leq R^{2}_{max} \right\rbrace.
  \end{align}
  
 Before presenting the result for the group-norm regularized \ac {MTL}, we want to bring it into attention that $A_1$ does not depend on the $\boldsymbol{W}$-constraint in the hypothesis space, therefore the bound for $A_1$ is the same for all cases we consider in this study, despite the choice of the reqularizer. However, $A_2$ can be further bounded for different hypothesis spaces corresponding to different choice of regularization functions. In the following we start with a useful lemma which helps bounding $A_{2}$ for the group-norm hypothesis space \eqref{GroupHSForFixedMApping}. The proof of this Lemma, which is based on the application of the Khintchine \eqref{K.K.Inq} and Rosenthal \eqref{Rose} inequalities, is presented in Appendix \ref{C}.
 
 \begin{lemm}
 \label{A_2ExcpectationGroupNorm}
Assume that the kernels in \eqref{GeneralFixedMappingModel} are uniformly bounded, that is $\| k \|_{\infty} \leq \mathcal{K}  \leq \infty$. Then, for the group norm regularizer $\frac{1}{2} \left\| \boldsymbol{W} \right\|_{2,q}^{2}$ in \eqref{GroupHSForFixedMApping} and for any $1 \leq q \leq 2$, the expectation $\mathbb{E}_{X, \sigma} \left\|  \boldsymbol{D}^{-1/2} \boldsymbol{V} \right\|_{2,q^{*}}$ for $\boldsymbol{D}=\boldsymbol{I}$ can be upper-bounded as 
\begin{align}
 \mathbb{E}_{X, \sigma} \left\| \left(\sum_{j > h_{t}}  \left\langle \frac{1}{n} \sum_{i=1}^{n} \sigma_{t}^{i} \phi(X_{t}^{i}) , \boldsymbol{u}_{t}^{j} \right\rangle \boldsymbol{u}_{t}^{j} \right)_{t=1}^{T}  \right\|_{2,q^{*}} \! \! \leq    \frac{\sqrt{\mathcal{K} e} q^{*} T^{\frac{1}{q^{*}}}}{n}  + \sqrt{\frac{ e {q^{*}}^{2} }{n}  \left\|  \left( \sum_{j>h_t} \lambda_t^{j} \right)_{t=1}^{T}   \right\|_{\frac{q^{*}}{2}} }.
 \nonumber
 \end{align}
  \end{lemm}
\begin{Remark}
\label{HolderExpectation}
Similarly as in \lemmref{A_2ExcpectationGroupNorm}, one can easily prove that  
\begin{align}
 \mathbb{E}_{X, \sigma} \left\| \left(\sum_{j > h_{t}}  \left\langle \frac{1}{n} \sum_{i=1}^{n} \sigma_{t}^{i} \phi(X_{t}^{i}) , \boldsymbol{u}_{t}^{j} \right\rangle \boldsymbol{u}_{t}^{j} \right)_{t=1}^{T}  \right\|_{2,q^{*}}^{2}   \leq    \frac{\mathcal{K} e {q^{*}}^{2} T^{\frac{2}{q^{*}}}}{n^{2}}  + \frac{ e {q^{*}}^{2}}{n}  \left\|  \left( \sum_{j>h_t} \lambda_t^{j} \right)_{t=1}^{T}   \right\|_{\frac{q^{*}}{2}}.
 \label{ExpBoundingStrongLemma}
 \end{align}
 To see this, note that in the first step of the proof of \lemmref{A_2ExcpectationGroupNorm} (see Appendix \ref{C}), by replacing the outermost exponent $\frac{1}{q^{*}}$ with $\frac{2}{q^{*}}$, and following the same procedure, one can  verify \eqref{ExpBoundingStrongLemma}. Therefore, it can be concluded that very similar \ac {LRC} bounds can be obtained via \thrmref{GeneralLRCBoundForFixedMappingStrongConvex} and \thrmref{GeneralLRCBoundForFixedMapping2}.
\end{Remark}
 \begin{corr}
 \label{GroupNormLRCStrongConvexity}
 Using \thrmref{GeneralLRCBoundForFixedMapping2}, for any $1 \leq q \leq 2$, the \ac {LRC} of function class $\mathcal{F}_{q}$ in \eqref{GroupHSForFixedMApping} can be bounded as
 \begin{align}
 \label{GroupNormLRCStrongConvexityEq}
 \mathfrak{R}(\mathcal{F}_{q},r) \leq \sqrt{ \frac{4}{nT}   \left\| \left( \sum_{j=1}^{\infty} \min  \left(  r T^{1 - \frac{2}{q^{*}}}     ,    \frac{2 e  {q^{*}}^{2} R^{2}_{max}}{T}  \lambda_t^{j} \right)  \right)_{t=1}^{T}   \right\|_{\frac{q^{*}}{2}}   } + \frac{\sqrt{2 \mathcal{K} e } { R_{max} q^{*}} T^{\frac{1}{q^{*}}}}{nT}.
 \end{align}
 \end{corr}
 \textbf{Proof Sketch:} 
  The proof of the corollary uses the result of \lemmref{A_2ExcpectationGroupNorm} to  upper bound $A_{2}$ for the group-norm hypothesis space \eqref{GroupHSForFixedMApping} as,
   \begin{align}
   \label{A2BoundStrongConvexitymain}
   A_{2} (\mathcal{F}_{q}) \leq \sqrt{\frac{2 e {q^{*}}^{2} R^{2}_{max}}{n T^{2}} \left\|  \left( \sum_{j>h_t} \lambda_t^{j} \right)_{t=1}^{T}   \right\|_{\frac{q^{*}}{2}}} + \frac{\sqrt{2 \mathcal{K} e } R_{max} q^{*} T^{\frac{1}{q^{*}}}}{nT}.
   \end{align}
   Now, combining \eqref{GeneralA1BoundStrongConvexity} and \eqref{A2BoundStrongConvexitymain} provides the bound on $\mathfrak{R}(\mathcal{F}_{q},r)$ as
   \begin{align}
   \label{LRCBoundGroupNormStrongConvexitymain}
   \mathfrak{R}(\mathcal{F}_{q},r) & \leq   \sqrt{  \frac{ r \sum_{t=1}^{T} h_t} {nT}  } + \sqrt{\frac{2 e  {q^{*}}^{2} R^{2}_{max} }{nT^{2}}  \left\|  \left( \sum_{j>h_t} \lambda_t^{j} \right)_{t=1}^{T}   \right\|_{\frac{q^{*}}{2}} }  +  \frac{\sqrt{2 \mathcal{K} e } R_{max} q^{*} T^{\frac{1}{q^{*}}}}{nT},
  \end{align}
  Then using inequalities shown below which hold for any $\alpha_{1}, \alpha_{2} > 0$, any non-negative vectors $ \boldsymbol{a}_1, \boldsymbol{a}_2 \in \mathbb{R}^{T}$, any $0 \leq q  \leq  p  \leq  \infty$ and any $s \geq 1$,
   \begin{align}
   & \label{star} (\star) \sqrt{\alpha_1}+ \sqrt{\alpha_2} \leq \sqrt{2 (\alpha_1 + \alpha_2)}
   \\
   & \label{sstar}(\star \star)  \quad l_p-to-l_q:   \quad    \left\| \boldsymbol{a}_1 \right\|_q  = \left\langle \boldsymbol{1}, \boldsymbol{a}_1 \right\rangle^{\frac{1}{q}} \stackrel{\text{H\"{o}lder's}}{\leq} \left( \left\| \boldsymbol{1} \right\|_{(p/q)^{*}}  \left\| \boldsymbol{a}_1^{q} \right\|_{(p/q)}  \right) ^{\frac{1}{q}} = T^{\frac{1}{q} - \frac{1}{p}}  \left\| \boldsymbol{a}_1 \right\|_p
   \\
   & \label{ssstar} (\star \star \star)  \quad    \left\| \boldsymbol{a}_1 \right\|_s  +  \left\| \boldsymbol{a}_2 \right\|_s  \leq  2^{1 - \frac{1}{s}} \left\| \boldsymbol{a}_1  +  \boldsymbol{a}_2 \right\|_s  \leq   2 \left\| \boldsymbol{a}_1  +  \boldsymbol{a}_2 \right\|_{s},
   \end{align}
 we can obtain the desired result. See Appendix \ref{C} for the detailed proof.
 
 \begin{Remark}
 \label{NonMonoBoundStrongConvexity}
   Since the LRC bound above is not monotonic in $q$, it is more reasonable to state the above bound in terms of $\kappa \geq q$; choosing $\kappa=q$ is not always the optimal choice. Trivially, for the group norm regularizer with any $\kappa \geq q$, it holds that $\|\boldsymbol{W}\|_{2,\kappa} \leq \|\boldsymbol{W}\|_{2,q}$ and therefore $\mathfrak{R}(\mathcal{F}_{q},r) \leq \mathfrak{R}(\mathcal{F}_{\kappa},r)$. Thus, we have the following bound on $\mathfrak{R}(\mathcal{F}_{q},r)$ for any $\kappa \in [q,2]$,
   \begin{align}
   \label{KappaBound}
    \mathfrak{R}(\mathcal{F}_{q},r)  \leq \sqrt{ \frac{4}{nT}   \left\| \left( \sum_{j=1}^{\infty} \min  \left(   r T^{1 - \frac{2}{\kappa^{*}}}       ,    \frac{2 e  {\kappa^{*}}^{2} R^{2}_{max}}{T}  \lambda_t^{j} \right)  \right)_{t=1}^{T}   \right\|_{\frac{\kappa^{*}}{2}}   }
    +  \frac{\sqrt{2 \mathcal{K} e } R_{max} {\kappa^{*}} T^{\frac{1}{\kappa^{*}}}}{nT}.
   \end{align}
 \end{Remark}
 \begin{Remark}[Sparsity-inducing group-norm]
   \label{SparseGroupNormStrongConvexity}
Assuming a sparse representations shared across multiple tasks is a well-known presumption in \ac {MTL} \citep{argyriou2007,Argyriou2008} which leads to the use of group norm regularizer $\frac{1}{2} \| \boldsymbol{W} \|_{2,1}^{2}$. Notice that for any $\kappa \geq 1$, it holds that $\mathfrak{R}(\mathcal{F}_{1},r) \leq \mathfrak{R}(\mathcal{F}_{\kappa},r)$. Also, assuming an identical tail sum $\sum_{j \geq h} \lambda^{j}$ for all tasks, reduces the bound in \eqref{KappaBound} to the function $\kappa^{*} \mapsto \kappa^{*} T^{1/\kappa^{*}}$ in terms of $\kappa$. This function attains its minimum at $\kappa^{*} = \log T$. Thus, by choosing $\kappa^{*} = \log T$ it is easy to show:
      \begin{align*}
         \mathfrak{R}(\mathcal{F}_{1},r)  \leq & \sqrt{ \frac{4}{nT}   \Big\| \Big( \sum_{j=1}^{\infty} \min  \Big(   r T^{1 - \frac{2}{\kappa^*}}      ,    \frac{2 e  {\kappa^*}^{2} R^{2}_{max}}{T}  \lambda_t^{j} \Big)  \Big)_{t=1}^{T}   \Big\|_{\frac{\kappa^*}{2}}   } +  \frac{\sqrt{2 \mathcal{K} e } R_{max} \kappa^* T^{\frac{1}{\kappa^*}}}{nT}\\
         \stackrel{(l_{\frac{\kappa^*}{2}}-to-l_{\infty} )}{\leq} & \sqrt{ \frac{4}{nT}   \Big\| \Big( \sum_{j=1}^{\infty} \min  \Big(   r T      ,     \frac{2 e^{3}   (\log{T})^{2} R^{2}_{max} }{ T}  \lambda_t^{j} \Big)  \Big)_{t=1}^{T}   \Big\|_{\infty}   } + \frac{\sqrt{ 2 \mathcal{K} } R_{max} e^{\frac{3}{2}} \log{T} }{nT}.
     \end{align*}
 \end{Remark}
 \begin{Remark} [$L_{2,q}$ Group-norm regularizer with $ q \geq 2$] 
 \label{GroupNormLargeQHolder}
For any $q \geq 2$, \thrmref{GeneralLRCBoundForFixedMapping2} provides a \ac {LRC} bound for the function class $\mathcal{F}_{q}$ in \eqref{GroupHSForFixedMApping} as
   \begin{align}
   \mathfrak{R}(\mathcal{F}_{q},r) \leq  \sqrt{\frac{4}{nT}   \left\|   \left(  \sum_{j=1}^{\infty} \min \left( r T^{1-\frac{2}{{q^{*}}}} ,\frac{ 2 R^{2}_{max}}{T} \lambda_{t}^{j}  \right) \right)_{t=1}^{T}  \right\|_{\frac{q^*}{2}} },  
   \end{align} 
   \noindent
   where $q^{*} := \frac{q}{q-1}$.
 \end{Remark}
 \begin{proof}
   \begin{align}
   A_2 (\mathcal{F}_{q}) & \stackrel{ \text{H\"{o}lder's}}{\leq}  \frac{1}{T}  \mathbb{E}_{X,\sigma} \left\lbrace \sup_{\boldsymbol{f} \in \mathcal{F}_{q}} \left\|  \boldsymbol{W} \right\|_{2,q} \left\|  \boldsymbol{V} \right\|_{2,q^{*}} \right\rbrace
   \nonumber\\
   & \leq  \frac{\sqrt{2}R_{max}}{T} \mathbb{E}_{X,\sigma}    \left( \sum_{t=1}^{T} \left\| \sum_{j > h_t} \left\langle \frac{1}{n} \sum_{i=1}^{n} \sigma_t^{i} \phi(X_t^{i}) , \boldsymbol{u}_{t}^{j} \right\rangle  \boldsymbol{u}_{t}^{j} \right\| ^{q^{*}} \right) ^{\frac {1}{q^{*}}}
   \nonumber\\
   & \stackrel{\text{Jensen's}}{\leq} \frac{\sqrt{2} R_{max}}{T} \left( \sum_{t=1}^{T} \left( \mathbb{E}_{X,\sigma} \left\| \sum_{j > h_t} \left\langle \frac{1}{n} \sum_{i=1}^{n} \sigma_t^{i} \phi(X_t^{i}) , \boldsymbol{u}_{t}^{j} \right\rangle  \boldsymbol{u}_{t}^{j} \right\|^{2} \right) ^{\frac{q^{*}}{2}} \right) ^{\frac {1}{q^{*}}}
   \nonumber\\
   & = \frac{\sqrt{2} R_{max}}{T} \left( \sum_{t=1}^{T} \left( \sum_{j > h_t} \mathbb{E}_{X,\sigma} \left\langle \frac{1}{n} \sum_{i=1}^{n} \sigma_t^{i} \phi(X_t^{i}) , \boldsymbol{u}_{t}^{j} \right\rangle^{2}  \right) ^{\frac{q^{*}}{2}} \right) ^{\frac {1}{q^{*}}}
   \nonumber\\
   & = \frac{\sqrt{2} R_{max}}{T} \left( \sum_{t=1}^{T} \left( \sum_{j > h_t} \frac{\lambda_{t}^{j}}{n}  \right) ^{\frac{q^{*}}{2}} \right) ^{\frac {1}{q^{*}}} =  \sqrt{\frac{2 R_{max}^{2}}{nT^{2}} \left\| \left( \sum_{j>h_{t}} \lambda_{t}^{j} \right)_{t=1}^{T}  \right\|_{\frac{q^{*}}{2}} }.
   \nonumber  
   \end{align}
  By applying $(\star)$, $(\star \star)$ and $(\star \star \star) $, this last result together with the bound in \eqref{GeneralA1BoundStrongConvexity} for $A_{1}$, yields the result.
   \end{proof}
 To investigate the tightness of the bound in \eqref{GroupNormLRCStrongConvexityEq}, we derive the lower bound which holds for the \ac {LRC} of $\mathcal{F}_{q}$ with any $q \geq 1$.  The proof of the result can be found in Appendix \ref{C}.
 \begin{thrm}[Lower bound]
  \label{LowerBoundStrongConvexity}
     The following lower bound holds for the local Rademacher complexity of $\mathcal{F}_{q}$ in \eqref{GroupNormLRCStrongConvexityEq} with any $q \geq 1$. There is an absolute constant $c$ so that $\forall t$, if $\lambda_{t}^{1} \geq 1/(n R_{max}^{2})$ then for all $r \geq \frac{1}{n}$ and $q \geq 1$,
    \begin{align}
    \label{LowerBoundFixedMappingStrongConvexity}
       \mathfrak{R} (\mathcal{F}_{q, R_{max}, T}, r) \geq \sqrt{\frac{c}{n T^{1-\frac{2}{q^{*}}}} \sum_{j=1}^{\infty} \min \left( r T^{1-\frac{2}{q^{*}}} , \frac{R_{max}^{2}}{T} \lambda_{1}^{j}\right) }.
    \end{align}
    \end{thrm}
 A comparison between the lower bound in \eqref{LowerBoundFixedMappingStrongConvexity} and the upper bound in \eqref{GroupNormLRCStrongConvexityEq} can be clearly illustrated by assuming identical eigenvalue tail sums $\sum_{j \geq \infty} \lambda_{t}^{j}$ for all tasks, for which the upper bound translates to
    \begin{align}
       \mathfrak{R} (\mathcal{F}_{q, R_{max}, T},r)  \leq   \sqrt{\frac{4}{n T^{1-\frac{2}{q^{*}}}}  \sum_{j=1}^{\infty} \min \left( r T^{1-\frac{2}{{q^{*}}}} ,\frac{2 e {q^{*}}^{2} R^{2}_{max}  }{T} \lambda_{t}^{j}  \right)}  + \frac{\sqrt{2 \mathcal{K} e } R_{max} q^{*} T^{\frac{1}{q^{*}}}}{nT}.
     \nonumber
    \end{align}
By comparing this to \eqref{LowerBoundFixedMappingStrongConvexity}, we see that the lower bound matches the upper bound up to constants. The same analysis for \ac {MTL} models with Schatten norm and graph regularizers yields similar results confirming that the \ac {LRC} upper bounds that we have obtained are reasonably tight.

\begin{Remark}
\label{Lowebound}
It is worth pointing out that a matching lower bound on the local Rademacher complexity does not necessarily implies a tight bound on the expectation of an empirical minimizer. As it has been shown in Section 4 of \cite{ bartlett2004local}, by direct analysis of the empirical minimizer, sharper bounds than the \ac {LRC}-based bounds can be obtained. Consequently, based on Theorem 8 in \cite{ bartlett2004local}, there might be cases in which the local Rademacher complexity bounds are constants, however $P \hat{f}$ is of some order depending on the number of samples $n$---O(1/n))---which decreases with $n$ growing.  As it has pointed out in that paper, under some mild conditions on the loss function $\ell$, similar argument also holds for the class of loss functions $\{ \ell_{f} - \ell_{f^{*}} :  f \in \mathcal{F}\}$.
\end{Remark}


\subsection{Schatten Norm Regularized \ac {MTL}}
\label{ssec:SchattenNormRegularizer}
\cite{argyriou2007} developed a spectral regularization framework for \ac {MTL} where the $L_{S_{q}}$-Schatten norm
$\frac{1}{2} \| \boldsymbol{W}\|_{S_q}^{2} := \frac{1}{2} \big[\text{tr}\big(\boldsymbol{W}^T\boldsymbol{W}\big)^{\frac{q}{2}}\big]^{\frac{2}{q}}$ is studied as a concrete example, corresponding to performing \ac{ERM} in the following hypothesis space:
 \begin{equation}
 \label{HSForFixedMAppingSchatten}
 \mathcal{F}_{S_q} := \left\lbrace  X \mapsto [\left\langle \boldsymbol{w}_1, \phi (X_1)\right\rangle , \ldots, \left\langle \boldsymbol{w}_T, \phi (X_T)\right\rangle ]^{T} :  \frac{1}{2} \| \boldsymbol{W} \|_{S_q}^{2}  \leq R'^{2}_{max} \right\rbrace.
 \end{equation}
\begin{corr}
\label{SchattenNormSC}
For any $1 \leq q \leq 2$ in \eqref{HSForFixedMAppingSchatten}, the \ac {LRC} of function class $\mathcal{F}_{S_q}$ is bounded as
\begin{align}
\mathfrak{R}(\mathcal{F}_{S_q},r) \leq  \sqrt{ \frac{4}{nT}   \Big\| \Big( \sum_{j=1}^{\infty} \min  \Big(   r ,  \frac{ 2 q^{*}  R'^{2}_{max} }{T}  \lambda_t^{j} \Big)
 \Big)_{t=1}^{T}   \Big\|_{1}}.
 \nonumber
\end{align}
\end{corr}
The proof is provided in Appendix \ref{C}.

 \begin{Remark}[Sparsity-inducing Schatten-norm (trace norm)]
   \label{SparseSchattenNormStrongConvexity}
   Trace-norm regularized MTL, corresponding to Schatten norm regularization with $q=1$~\citep{Maurer2013,pong2010}, imposes a low-rank structure on the spectrum of $\boldsymbol{W}$ and can also be interpreted as low dimensional subspace learning \citep{argyriou2008algorithm,kumar2012,kang2011}. Note that for any $q \geq 1$, it holds that $\mathfrak{R}(\mathcal{F}_{S_1},r) \leq \mathfrak{R}(\mathcal{F}_{S_q},r)$. Therefore, choosing the optimal $q^{*}=1$, we get
      \begin{align}
         \mathfrak{R}(\mathcal{F}_{S_1},r)  \leq & \sqrt{ \frac{4}{nT}   \Big\| \Big( \sum_{j=1}^{\infty} \min  \Big(   r ,  \frac{2  R'^{2}_{max} }{T}  \lambda_t^{j} \Big)
          \Big)_{t=1}^{T}   \Big\|_{1}}.
       \nonumber
     \end{align}
 \end{Remark}
\begin{Remark}[$L_{S_{q}}$ Schatten-norm regularizer with $ q \geq 2$] 
 \label{SchattenNormLargeQHolder}
 For any $ q \geq 2$, \thrmref{GeneralLRCBoundForFixedMapping2} provides a \ac {LRC} bound for the function class $\mathcal{F}_{S_q}$ in \eqref{HSForFixedMAppingSchatten} as
  \begin{align}
   \mathfrak{R}(\mathcal{F}_{S_q},r) \leq  \sqrt{\frac{4}{nT}   \left\|   \left(  \sum_{j=1}^{\infty} \min \left( r ,\frac{2 R'^{2}_{max}}{T} \lambda_{t}^{j}  \right) \right)_{t=1}^{T}  \right\|_{1} }.  
    \end{align} 
   \noindent
   \end{Remark}
   \begin{proof}
  Taking $q^{*} := \frac{q}{q-1}$, we first bound the expectation $\mathbb{E}_{X,\sigma}  \left\| \boldsymbol{V} \right\|_{S_{q^{*}}}$. Take $\boldsymbol{U}_{t}^{i}$ as a matrix with $T$ columns where the only non-zero column $t$ of $\boldsymbol{U}_{t}^{i}$ is defined as $ \sum_{j>h_t} \left\langle \frac{1}{n} \phi (X_t^{i}) , \boldsymbol{u}_{t}^{j} \right\rangle  \boldsymbol{u}_{t}^{j}$. Based on the definition of $\boldsymbol{V} = \left( \sum_{j > h_t}  \left\langle \frac{1}{n} \sum_{i=1}^{n} \sigma_{t}^{i} \phi(X_t^{i}), \boldsymbol{u}_{t}^{j} \right\rangle \boldsymbol{u}_{t}^{j}\right)_{t=1}^{T}$, we can then provide a bound for this expectation as
   \begin{align}
   \mathbb{E}_{X, \sigma} \left\|  \boldsymbol{V} \right\|_{S_{q^{*}}} & =
    \mathbb{E}_{X, \sigma} \left\| \left(\sum_{j > h_{t}}  \left\langle \frac{1}{n} \sum_{i=1}^{n} \sigma_{t}^{i} \phi(X_{t}^{i}) , \boldsymbol{u}_{t}^{j} \right\rangle \boldsymbol{u}_{t}^{j} \right)_{t=1}^{T}  \right\|_{S_{q^{*}}} 
    \nonumber\\
    & = \mathbb{E}_{X, \sigma} \left\|  \sum_{t=1}^{T} \sum_{i=1}^{n} \sigma_{t}^{i} \boldsymbol{U}_{t}^{i} \right\|_{S_{q^{*}}} 
   \nonumber\\
  & \stackrel{ \text{Jensen}}{\leq}  \left(\text {tr} \left( \sum_{t,s=1}^{T} \sum_{i,j=1}^{n} \mathbb{E}_{X, \sigma} \left( \sigma_{t}^{i} \sigma_{s}^{j} {\boldsymbol{U}_{t}^{i}}^{T} \boldsymbol{U}_{s}^{j}\right)  \right)^{\frac{q^{*}}{2}}  \right)^{\frac{1}{q^{*}}} 
   \nonumber\\
   & =  \left(\text {tr} \left( \sum_{t=1}^{T} \sum_{i=1}^{n} \mathbb{E}_{X} \left( {\boldsymbol{U}_{t}^{i}}^{T} \boldsymbol{U}_{t}^{i}\right)  \right)^{\frac{q^{*}}{2}}  \right)^{\frac{1}{q^{*}}} 
   \nonumber\\
   & =  \left( \sum_{t=1}^{T} \sum_{i=1}^{n} \sum_{j>h_t} \mathbb{E}_{X} \left\langle \frac{1}{n} \phi (X_t^{i}) , \boldsymbol{u}_{t}^{j} \right\rangle^{2} \right)^{\frac{1}{2}}  
   \nonumber\\
    & =  \left( \frac{1}{n} \sum_{t=1}^{T} \sum_{j>h_t} \lambda_{t}^{j} \right)^{\frac{1}{2}} 
   = \sqrt{\frac{1}{n} \left\| \left(  \sum_{j> h_t} \lambda_{t}^{j} \right)_{t=1}^{T} \right\|_{1} }.
   \nonumber
   \end{align}
   Note that replacing this into \eqref{GeneralBoundHolderInq}, and with the help of $(\star)$, $(\star \star)$ and $(\star \star \star) $, one can conclude the result.
   \end{proof}
 \subsection{Graph Regularized \ac {MTL}}
 \label{:GraphRegularizer}
   The idea underlying graph regularized \ac {MTL}  is to force the models of related tasks to be close to each other, by penalizing the squared distance $\| \boldsymbol{w}_{t} - \boldsymbol{w}_{s} \|^{2}$ with different weights $\omega_{ts}$. We consider the following  \ac {MTL} graph regularizer \citep{maurer2006Graph}
 \begin{align}
 \Omega (\boldsymbol{W}) = \frac{1}{2} \sum_{t=1}^{T} \sum_{s=1}^{T} \omega_{ts} \| \boldsymbol{w}_{t} - \boldsymbol{w}_{s} \|^{2} + \eta \sum_{t=1}^{T} \| \boldsymbol{w}_{t} \|^{2} = \sum_{t=1}^{T} \sum_{s=1}^{T} \left( \boldsymbol{L}+ \eta \boldsymbol{I}\right)_{ts} \left\langle \boldsymbol{w}_{t} , \boldsymbol{w}_{s}\right\rangle,
 \nonumber
 \end{align}
 \noindent
 where $\boldsymbol{L}$ is the graph-Laplacian associated to a matrix of edge-weights $\omega_{ts}$, $\boldsymbol{I}$ is the identity operator, and $\eta > 0$ is a regularization parameter. According to the identity $ \sum_{t=1}^{T} \sum_{s=1}^{T} \big( \boldsymbol{L}+ \eta \boldsymbol{I}\big)_{ts} \big\langle \boldsymbol{w}_{t} , \boldsymbol{w}_{s}\big\rangle = \|  (\boldsymbol{L}+ \eta \boldsymbol{I})^{1/2} \boldsymbol{W} \|_{F}^{2}$, the corresponding hypothesis space is:
 \begin{align}
 \label{GraphRegHS}
 \mathcal{F}_{G} := \Big\lbrace  X \mapsto [\Big\langle \boldsymbol{w}_1, \phi (X_1)\Big\rangle , \ldots, \Big\langle \boldsymbol{w}_T, \phi (X_T)\Big\rangle ]^{T} : \frac{1}{2} \| \boldsymbol{D}^{1/2} \boldsymbol{W}\|_{F}^{2}  \leq R''^{2}_{max} \Big\rbrace.
 \end{align}
 where we define $\boldsymbol{D} := \boldsymbol{L}+ \eta \boldsymbol{I}$.

\begin{corr}
\label{GraphSC}
For any given positive definite matrix $\boldsymbol{D}$ in \eqref{GraphRegHS}, the \ac {LRC} of $\mathcal{F}_{G}$ is bounded by
\begin{align}
\mathfrak{R} (\mathcal{F}_{G},r) \leq  \sqrt{ \frac{4}{nT}   \Big\| \Big( \sum_{j=1}^{\infty} \min  \Big( r ,
\frac {2 \boldsymbol{D}_{tt}^{-1} R''^{2}_{max}}{T}   \lambda_t^{j} ) \Big) \Big)_{t=1}^{T}   \Big\|_{1}}.
\end{align}
\noindent
where $\left( \boldsymbol{D}_{tt}^{-1}\right)_{t=1}^{T} $ are the diagonal elements of $\boldsymbol{D}^{-1}$.
\end{corr}
See Appendix \ref{C} for the proof.

\section{Excess Risk Bounds for Norm Regularized \ac{MTL} Models}
\label{sec:StrongConvexityExcessRisk}
In this section we will provide the distribution and data-dependent excess risk bounds for the hypothesis spaces considered earlier. Note that, due to space limitations, the proofs are provided only for the hypothesis space $\mathcal{F}_{q}$ with $q \in [1,2]$ in \eqref{GroupHSForFixedMApping}. However, in the cases involving the $L_{2,q}$-group norm with $q \geq 2$, as well as the $L_{S_q}$-Schatten and graph norms, the proofs can be obtained in a very similar way. More specifically, by using the \ac{LRC} bounds of \remref{GroupNormLargeQHolder}, \corref{SchattenNormSC}, \remref{SchattenNormLargeQHolder} and \corref{GraphSC}, one can follow the same steps of the proofs of this section to arrive at the results pertaining to these cases.
\begin{thrm}\textbf{\emph{(Distribution-dependent excess risk bound for a $L_{2,q}$ group-norm regularized \ac {MTL})}}
\label{RiskBoundForFixedMappingStrongConvexity}
Assume that $\mathcal{F}_{q}$ in \eqref{GroupHSForFixedMApping} is a convex class of functions with ranges in $[-b,b]$, and let the loss function $\ell$ of Problem \eqref {GeneralFixedMappingModel} be such that \assumref{assumption} is satisfied. Let $\boldsymbol{\hat{f}}$ be any element of $ \mathcal{F}_{q}$ with $1 \leq q \leq 2$ which satisfies $P_n \ell_{\boldsymbol{\hat{f}}} = \inf_{\boldsymbol{f} \in \mathcal{F}_{q}} P_n \ell_{\boldsymbol{f}}$. Assume moreover that $k$ is a positive semi-definite kernel on $\mathcal{X}$ such that $\| k \|_{\infty} \leq \mathcal{K}  \leq \infty$. Denote by $r^{*}$ the fixed point of $2 B L \mathfrak{R}(\mathcal{F}_{q},\frac{r}{4L^{2}})$. Then, for any $K >1$ and $x > 0 $, with probability at least $1 - e^{-x}$, the excess loss of function class $\mathcal{F}_{q}$ is bounded as
\begin{align}
P (\ell_{\boldsymbol{\hat{f}}} - \ell_{\boldsymbol{f}^{*}})  \leq   \frac{32 K}{B} r^{*} + \frac{(48 Lb+ 16 B K) x}{n T}, 
\end{align}
where for the fixed point $r^{*}$ of the local Rademacher complexity $2 B L \mathfrak{R}(\mathcal{F}_{q},\frac{r}{4L^{2}})$, it holds that
\begin{align}
\label{BoundFixefPointGroupNormStrongConvexity}
r^{*} \leq \min_{0 \leq h_{t} \leq \infty} \frac{B^{2} \sum_{t=1}^{T} h_t}{Tn}  + 4BL \sqrt{\frac{2 e {q^{*}}^{2} R^{2}_{max}}{nT^{2}}  \left\| \left( \sum_{j>h_t} \lambda_{t}^{j}\right)_{t=1}^{T} \right\|_{\frac{q^{*}}{2}} } + \frac{ 4\sqrt{2 \mathcal{K} e } R_{max} q^{*} T^{\frac{1}{q^{*}}}}{nT},
\end{align}
\noindent 
where $h_1, \ldots, h_T$ are arbitrary non-negative integers.
\end{thrm}
\begin{proof}
 First notice that $\mathcal{F}_{q}$ is convex, thus it is star-shaped around any of its points. Hence according to \lemmref{LRCIsStarHull}, $\mathfrak{R}(\mathcal{F}_{q},r)$ is a sub-root function. Moreover, because of the symmetry of $\sigma_t^{i}$ and because $\mathcal{F}_{q}$ is convex and symmetric, it can be shown that $\mathfrak{R}(\mathcal{F}_{q}^{*},r) \leq 2 \mathfrak{R}(\mathcal{F}_{q},\frac{r}{4L^{2}}) $, where $\mathfrak{R}(\mathcal{F}^{*}_{q}, r)$ is defined according to \eqref{LRCofFstar} for the class of functions $\mathcal{F}_{q}$. Therefore, it suffices to find the fixed point of $2 B L \mathfrak{R}(\mathcal{F}_{q},\frac{r}{4L^{2}})$ by solving $\phi(r) =r$. For this purpose, we will use \eqref{LRCBoundGroupNormStrongConvexitymain} as a bound for $\mathfrak{R}(\mathcal{F}_{q}, r)$, and solve $\sqrt{\alpha r} +\gamma =r$ (or equivalently $r^2 - (\alpha +2 \gamma)r + \gamma^2 =0$) for $r$, where we define
\begin{align}
\alpha= \frac{B^{2} \sum_{t=1}^{T} h_t}{Tn}, \; \text{and} \; \gamma = 2BL \sqrt{\frac{2 e {q^{*}}^{2} R^{2}_{max}}{nT^{2}}  \left\| \left( \sum_{j>h_t} \lambda_{t}^{j}\right)_{t=1}^{T} \right\|_{\frac{q^{*}}{2}} }+\frac{ 2 \sqrt{2 \mathcal{K} e } R_{max} B L q^{*} T^{\frac{1}{q^{*}}}}{nT}.
\end{align}
\noindent
It is not hard to verify that $r^{*} \leq \alpha+2\gamma$. Substituting the definition of $\alpha$ and $\gamma$ gives the result.
\end{proof}

\begin{Remark}
If the conditions of \thrmref{GeneralLRCBoundForFixedMapping2} hold, then it can be shown that the following results hold for the fixed point of the considered hypothesis spaces in \eqref{GroupHSForFixedMApping}, \eqref{HSForFixedMAppingSchatten} and \eqref{GraphRegHS}.
  \label{FixedPointSchatten&Graph}
  \begin{itemize}
\item Group norm: For the fixed point $r^{*}$ of the local Rademacher complexity $2 B L \mathfrak{R}(\mathcal{F}_{q},\frac{r}{4L^{2}})$ with any $1 \leq q \leq 2$ in \eqref{GroupHSForFixedMApping}, it holds
    \begin{align}
     \label{BoundFixefPointGroupNormq<2}
    \!\! r^{*} \leq \!\! \min_{0 \leq h_{t} \leq \infty} \!\! \frac{B^{2} \sum_{t=1}^{T} h_t}{Tn}  + 4BL \!\! \sqrt{\frac{2 e {q^{*}}^{2} R^{2}_{max}}{nT^{2}}  \left\| \left( \sum_{j>h_t} \lambda_{t}^{j}\right)_{t=1}^{T} \right\|_{\frac{q^{*}}{2}} } \!\! + \frac{4 \sqrt{2 \mathcal{K} e} R_{max} {q^{*}} BL T^{\frac{1}{{q^{*}}} }  }{nT}.
     \end{align}
Also, for any $q \geq 2$ in \eqref{GroupHSForFixedMApping}, it holds
     \begin{align}
      \label{BoundFixefPointGroupNormq>2}
      r^{*} \leq \min_{0 \leq h_{t} \leq \infty} \frac{B^{2} \sum_{t=1}^{T} h_t}{Tn}  + 4BL \sqrt{\frac{2 R^{2}_{max}}{nT^{2}}  \left\| \left( \sum_{j>h_t} \lambda_{t}^{j}\right)_{t=1}^{T} \right\|_{\frac{q^{*}}{2}} }.
     \end{align}
\item Schatten-norm: For the fixed point $r^{*}$ of the local Rademacher complexity $2 B L \mathfrak{R}(\mathcal{F}_{S_q},\frac{r}{4L^{2}})$ with any $1 \leq q \leq 2$ in \eqref{HSForFixedMAppingSchatten}, it holds
    \begin{align}
    \label{BoundFixefPointSchattenNormq<2}
    r^{*} \leq \min_{0 \leq h_{t} \leq \infty} \frac{B^{2} \sum_{t=1}^{T} h_t}{Tn}  + 4BL \sqrt{\frac{2 q^{*} R'^{2}_{max}}{nT^{2}}  \left\| \left( \sum_{j>h_t} \lambda_{t}^{j}\right)_{t=1}^{T} \right\|_{1} }.
    \end{align}    
Also, for any $q \geq 2 $ in \eqref{HSForFixedMAppingSchatten}, it holds
    \begin{align}
     \label{BoundFixefPointSchattenNormq>2}
     r^{*} \leq \min_{0 \leq h_{t} \leq \infty} \frac{B^{2} \sum_{t=1}^{T} h_t}{Tn}  + 4BL \sqrt{\frac{2 R'^{2}_{max}}{nT^{2}}  \left\| \left( \sum_{j>h_t} \lambda_{t}^{j}\right)_{t=1}^{T} \right\|_{1} }.
     \end{align}
\item Graph regularizer: For the fixed point $r^{*}$ of the local Rademacher complexity $2 B L \mathfrak{R}(\mathcal{F}_{G},\frac{r}{4L^{2}})$ with any positive operator $\boldsymbol{D}$ in \eqref{GraphRegHS}, it holds
   \begin{align}
    \label{BoundFixefPointGraph}
    r^{*} \leq \min_{0 \leq h_{t} \leq \infty} \frac{B^{2} \sum_{t=1}^{T} h_t}{Tn}  + 4BL \sqrt{\frac{2  R''^{2}_{max}}{nT^{2}}  \left\| \left( \boldsymbol{D}_{tt}^{-1} \sum_{j>h_t} \lambda_{t}^{j}\right)_{t=1}^{T} \right\|_{1} }.
    \end{align}
   \end{itemize}
\end{Remark} 

 Regarding the fact that $\lambda_{t}^{j}$s are decreasing with respect to $j$, we can assume $\exists d_{t}: \lambda_{t}^{j} \leq d_{t} j^{-\alpha_{t}}$ for some $\alpha_{t} > 1$. As examples, this assumption holds for finite rank kernels as well as convolution kernels. Thus, it can be shown that
 \begin{align}
 \label{BoundforLambdaSumSC}
 \sum_{j> h_{t}} \lambda_{t}^{j} \leq d_{t} \sum_{j > h_{t}} j ^{-\alpha_{t}} \leq d_{t} \int_{h_{t}}^{\infty} x^{-\alpha_{t}} dx = d_{t} \left[ \frac{1}{1-\alpha_{t}} x^{1-\alpha_{t}}\right]_{h_{t}}^{\infty} = -\frac{d_{t}}{1 - \alpha_{t}} h_{t}^{1 - \alpha_{t}}.
 \end{align}
 \noindent
 Note that via $l_{p}-to-l_{q}$ conversion inequality in \eqref{sstar}, for $p=1$ and $q = \frac{q*}{2}$, we have
 \begin{align}
 \frac{B^{2} \sum_{t=1}^{T} h_t}{Tn}  \leq B \sqrt{\frac{B^{2} T \sum_{t=1}^{T} h_{t}^{2}}{n^{2}T^{2}}} \stackrel{(\star \star)}{\leq} B \sqrt{\frac{B^{2} T^{2-\frac{2}{q^{*}}} \left\| \left( h_{t}^{2} \right)_{t=1}^{T}  \right\|_{\frac{q^{*}}{2}} }{n^{2}T^{2}}}.
 \nonumber
 \end{align}
 \noindent
Now, applying, \eqref{star} and \eqref{ssstar}, and inserting \eqref{BoundforLambdaSumSC} into \eqref{BoundFixefPointGroupNormStrongConvexity}, it holds for group norm regularized \ac{MTL} with $1 \leq  q \leq 2$, 
 \begin{align}
 \label{BoundFixedPoint2SC}
 r^{*} \leq \!\! \min_{0 \leq h_{t} \leq \infty} \!\! 2B  \sqrt{ \left\| \left( \frac{B^{2} T^{2-\frac{2}{q^{*}}} h_{t}^{2} }{n^{2} T^{2}} - \frac{32 d_{t} e {q^{*}}^{2} R^{2}_{max} L^{2}}{n T^{2} (1-\alpha_{t})} h_{t}^{1-\alpha_{t}}\right) _{t=1}^{T}\right\|_{\frac{q^{*}}{2}} } \!\! + \frac{ 4 \sqrt{2 \mathcal{K} e } R_{max} B L q^{*} T^{\frac{1}{q^{*}}}}{nT}.
 \end{align}
 \noindent
 Taking the partial derivative of the above bound with respect to $h_{t}$ and setting it to zero yields the optimal $h_{t}$ as
 \begin{align}
 h_{t} = \left( 16 d_{t} e {q^{*}}^{2} R^{2}_{max} B^{-2} L^{2} T^{\frac{2}{q^{*}} - 2} n \right)^{\frac{1}{1+\alpha_{t}}}.
 \nonumber 
 \end{align}
 \noindent
Note that substituting the above for $\alpha:= \min_{t \in \mathbb{N}_{T}} \alpha_{t}$ and $d = max_{t \in \mathbb{N}_{T}} d_{t}$ into \eqref{BoundFixedPoint2SC}, we can upper-bound the fixed point of $r^{*}$ as 
\begin{align}
 r^{*} \leq \frac{14 B^{2}}{n} \sqrt{\frac{\alpha + 1}{\alpha-1}}\left( d {q^{*}}^{2} R^{2}_{max} B^{-2} L^{2} T^{\frac{2}{q^{*}} - 2} n \right)^{\frac{1}{1+\alpha}}+ \frac{ 10 \sqrt{ \mathcal{K}} R_{max} B L q^{*} T^{\frac{1}{q^{*}}}}{nT}, 
 \nonumber
 \end{align}
  which implies that
 \begin{align}
 r^{*} = O \left( \left( \frac{T^{1-\frac{1}{q^{*}}}}{{q^{*}}}\right)^{\frac{-2}{1+ \alpha}} n^{\frac{-\alpha}{1 + \alpha}} \right). 
 \nonumber
 \end{align}
 
 It can be seen that the convergence rate can be as slow as $O\left( \frac{q^{*} T^{1/q^{*}}}{T \sqrt{n}} \right)$ (for small $\alpha$, where at least one $\alpha_{t} \approx 1$), and as fast as $O(n^{-1})$ (when $\alpha_{t} \rightarrow \infty$, for all $t$). The bound obtained for the fixed point together with \thrmref{RiskBoundForFixedMappingStrongConvexity} provides a bound for the excess risk, which leads to the following remark.
  \begin{Remark}[Excess risk bounds for selected norm regularized \ac {MTL} problems]
  \label{ExcessRiskBoundsSC}
   Assume that $\mathcal{F}_{q}$, $\mathcal{F}_{S_{q}}$ and $\mathcal{F}_{G}$ are convex classes of functions with ranges in $[-b,b]$, and let the loss function $\ell$ of Problem \eqref {GeneralFixedMappingModel} be such that \assumref{assumption} are satisfied. Assume moreover that $k$ is a positive semidefinite kernel on $\mathcal{X}$ such that $\| k \|_{\infty} \leq \mathcal{K}  \leq \infty$. Also, denote $\alpha:= \min_{t \in \mathbb{N}_{T}} \alpha_{t}$ and $d = max_{t \in \mathbb{N}_{T}} d_{t}$. Also,
   \begin{itemize}
    \item Group norm:  If $\boldsymbol{\hat{f}}$ satisfies $P_n \ell_{\boldsymbol{\hat{f}}} = \inf_{\boldsymbol{f} \in \mathcal{F}_{q}} P_n \ell_{\boldsymbol{f}}$, and  $r^{*}$ is the fixed point of the local Rademacher complexity $2 B L \mathfrak{R}(\mathcal{F}_{q},\frac{r}{4L^{2}})$ with any $1 \leq q \leq 2$ in \eqref{GroupHSForFixedMApping} and any $K >1$, it holds with probability at least $1 -e^{-x}$,
       \begin{align}
      P (\ell_{{\boldsymbol{\hat{f}}}} - \ell_{\boldsymbol{f}^{*}}) \leq  \min_{\kappa \in [q,2]} & 448 K \sqrt{\frac{\alpha+1}{\alpha-1}}\left( d {\kappa^{*}}^{2} R^{2}_{max} L^{2} \right)^{\frac{1}{1+\alpha}} B^{\frac{ \alpha -1}{\alpha+1}} \left( T^{\frac{2}{\kappa} } \right)^{\frac{-1}{1+\alpha}}  n^{\frac{-\alpha}{1+\alpha}} 
      \nonumber\\
      & + \frac{ 320 \sqrt{\mathcal{K} } R_{max} K L \kappa^{*} T^{\frac{1}{\kappa^{*}}}}{nT} + \frac{(48 Lb+ 16 B K) x}{n T} .
       \end{align}
           Also, for any $ q \geq 2$ in \eqref{GroupHSForFixedMApping} and any $K >1$, it holds with probability at least $1 -e^{-x}$,
              \begin{align}
             P (\ell_{{\boldsymbol{\hat{f}}}} - \ell_{\boldsymbol{f}^{*}}) \leq \min_{q \in [2,\infty]}  & 256 K \sqrt{\frac{\alpha+1}{\alpha-1}}\left( d R^{2}_{max} L^{2} \right)^{\frac{1}{1+\alpha}} B^{\frac{ \alpha - 1}{\alpha+1}} \left( T^{\frac{2}{q}} \right)^{\frac{-1}{1+\alpha}}   n^{\frac{-\alpha}{1+\alpha}} 
             \nonumber\\
             &  + \frac{(48 Lb+ 16 B K) x}{n T} .
              \end{align}
   \item Schatten-norm: If $\boldsymbol{\hat{f}}$ satisfies $P_n \ell_{\boldsymbol{\hat{f}}} = \inf_{\boldsymbol{f} \in \mathcal{F}_{S_q}} P_n \ell_{\boldsymbol{f}}$, and $r^{*}$ is the fixed point  of the local Rademacher complexity $2 B L \mathfrak{R}(\mathcal{F}_{S_q},\frac{r}{4L^{2}})$ with any $1 \leq q \leq 2$ in \eqref{HSForFixedMAppingSchatten} and any $K >1$, it holds with probability at least $1 -e^{-x}$,
     \begin{align}
     P (\ell_{{\boldsymbol{\hat{f}}}} - \ell_{\boldsymbol{f}^{*}}) & \leq  \min_{q \in [1,2]} 256 K \sqrt{\frac{\alpha +1}{\alpha-1}}\left( d q^* R'^{2}_{max} L^{2} \right)^{\frac{1}{1+\alpha}} B^{\frac{\alpha -1}{\alpha+1}} T^{\frac{-1}{1+\alpha}}  n^{\frac{-\alpha}{1+\alpha}} 
     \nonumber\\
     & + \frac{(48 Lb+ 16 B K)x}{n T}.
     \end{align}    
       Also, for any $q \geq 2$ in \eqref{HSForFixedMAppingSchatten} and any $K >1$, it holds with probability at least $1 -e^{-x}$,
          \begin{align}
          P (\ell_{{\boldsymbol{\hat{f}}}} - \ell_{\boldsymbol{f}^{*}}) & \leq  256 K \sqrt{\frac{\alpha +1}{\alpha-1}}\left( d R'^{2}_{max} L^{2} \right)^{\frac{1}{1+\alpha}} B^{\frac{\alpha -1}{\alpha+1}} T^{\frac{-1}{1+\alpha}}  n^{\frac{-\alpha}{1+\alpha}} 
          \nonumber\\
          & + \frac{(48 Lb+ 16 B K) x}{n T}.
          \end{align}       
  \item Graph regularizer: If $\boldsymbol{\hat{f}}$ satisfies $P_n \ell_{\boldsymbol{\hat{f}}} = \inf_{\boldsymbol{f} \in \mathcal{F}_{G}} P_n \ell_{\boldsymbol{f}}$, and $r^{*}$ is the fixed point of the local Rademacher complexity $2 B L \mathfrak{R}(\mathcal{F}_{G},\frac{r}{4L^{2}})$ with any positive operator $\boldsymbol{D}$ in \eqref{GraphRegHS} and any $K >1$, it holds with probability at least $1 -e^{-x}$,
     \begin{align}
     P (\ell_{{\boldsymbol{\hat{f}}}} - \ell_{\boldsymbol{f}^{*}}) & \leq   256 K \sqrt{\frac{\alpha +1}{\alpha-1 }}\left( d R''^{2}_{max} L^{2} \boldsymbol{D}_{max}^{-1} \right)^{\frac{1}{1+\alpha}} B^{\frac{\alpha -1}{\alpha+1}} T^{\frac{-1}{1+\alpha}}  n^{\frac{-\alpha}{1+\alpha}} 
          \nonumber\\
          & + \frac{(48 Lb+ 16 B K)x}{n T}.
      \end{align}
      \noindent
      where $\boldsymbol{D}_{max}^{-1} := \max_{t \in \mathbb{N}_{T}} \boldsymbol{D}_{tt}^{-1}$.
    \end{itemize}
\end{Remark} 
\begin{corr}\textbf{\emph{(Data-dependent excess risk bound for a \ac {MTL} problem with a $L_{2,q}$ group-norm regularizer)}}
\label{DataBoundGroupNorm}
Assume the convex class $\mathcal{F}_{q}$ in \eqref{GroupHSForFixedMApping} has ranges in $[-b,b]$, and let the loss function $\ell$ in Problem \eqref {GeneralFixedMappingModel} be such that \assumref{assumption} are satisfied. Let $\boldsymbol{\hat{f}}$ be any element of $ \mathcal{F}_{q}$ with $1 \leq q \leq 2$ which satisfies $P_n \ell_{\boldsymbol{\hat{f}}} = \inf_{\boldsymbol{f} \in \mathcal{F}_{q}} P_n \ell_{\boldsymbol{f}}$. Assume moreover that $k$ is a positive semidefinite kernel on $\mathcal{X}$ such that $\| k \|_{\infty} \leq \mathcal{K}  \leq \infty$. Let $\boldsymbol{K}_{t}$ be the $n \times n$ normalized Gram matrix (or kernel matrix) of task $t$ with entries $(\boldsymbol{K}_{t} )_{ij} := \frac{1}{n} k(X_{t}^{i}, X_{t}^{j} ) = \frac{1}{n} \left\langle \hat{\phi}(X_{t}^{i}) , \hat{\phi}(X_{t}^{j}) \right\rangle $. Let $\hat{\lambda}_{t}^{1}, \ldots \hat{\lambda}_{t}^{n}$ be the ordered eigenvalues of matrix $\boldsymbol{K}_{t}$, and $\hat{r}^{*}$ be the fixed point of
\begin{align}
\label{ExcessRiskLossFixedMappingEmpirical}
\hat{\psi}_{n} (r) = c_{1} \hat{\mathfrak{R}}(\mathcal{F}^{*}_{q}, c_{3} r) + \frac{c_{2}x}{nT},
\nonumber
\end{align}
\noindent
where $c_{1} = 2L \max\left( B,16Lb\right)$, $c_{2} = 128 L^{2} b^{2} + 2 b c_{1}$ and $c_{3} = 4+ 128 K+4B(48Lb+16BK)/c_{2}$, and
\begin{align}
\hat{\mathfrak{R}}(\mathcal{F}^{*}_{q},c_{3} r) := \mathbb{E}_{\sigma} \left[ \sup_{\substack{\boldsymbol{f} \in \mathcal{F}_{q},\\ L^{2} P_n \left(  \boldsymbol{f} - \boldsymbol{\hat{f}} \right)^{2}  \leq c_{3} r}} \frac{1}{nT} \sum_{t=1}^{T} \sum_{i=1}^{n} \sigma_{t}^{i} f_t(X_{t}^{i})  \Biggr \vert \left\{ x_t^i \right\}_{t \in \mathbb{N}_T, i \in \mathbb{N}_n}  \right]. 
\end{align}
Then, for any $K > 1$ and $x > 0 $, with probability at least $1 - 4e^{-x}$ the excess loss of function class $\mathcal{F}_{q}$ is bounded as
\begin{align}
P (\ell_{\boldsymbol{\hat{f}}} - \ell_{\boldsymbol{f}^{*}})  \leq   \frac{32 K}{B} \hat{r}^{*} + \frac{(48Lb+ 16 B K) x}{n T}, 
\end{align}
where for the fixed point $\hat{r}^{*}$ of the empirical local Rademacher complexity $\hat{\psi}_{n} (r)$, it holds
\begin{align}
\hat{r}^{*} & \leq   \frac{c_{1}^{2}  c_{3} \sum_{t=1}^{T} \hat{h}_{t}}{nT L^{2}} + 4 \sqrt{\frac{2 c_{1}^{2} {q^{*}}^{2} R^{2}_{max}}{nT^{2}} \left\| \left( \sum_{j>\hat{h}_{t}}^{n} \hat{\lambda}_{t}^{j}\right)_{t=1}^{T}  \right\|_{\frac{q^{*}}{2}} } + \frac{2 c_{2} x}{nT} ,
\nonumber 
\end{align}
\noindent 
where $\hat{h}_1, \ldots, \hat{h}_T$ are arbitrary non-negative integers, and $(\hat{\lambda_{t}^{j}})_{j=1}^{n}$ are eigenvalues of the normalized Gram matrix $\boldsymbol{K}$ obtained from kernel function $k$.
\end{corr}
The proof of the result is provided in Appendix \ref{D}.

\section {Discussion}
\label{sec:Discussion}
This section is devoted to compare the excess risk bounds based on local Rademacher complexity to those of the global ones. 
\subsection{Global vs. Local Rademacher Complexity Bounds}
\label{ssec:GlobalVsLocal}
First, note that to obtain the \ac {GRC}-based bounds, we apply Theorem 16 of \cite{maurer2006}, as we consider the same setting and assumptions for tasks' distributions as considered in this work. This theorem presents a \ac {MTL} bound based on the notion of \ac {GRC}.
\begin{thrm}[\ac{MTL} excess risk bound based on \ac {GRC}; Theorem 16 of \cite{maurer2006} ]
\label{GRCBoundforMTL}
 Let the vector-valued function class $\mathcal{F}$ be defined as $\fcal := \{\boldsymbol{f}=(f_1,\ldots,f_T) : \mathcal{X}\mapsto[-b,b]^T \}$. Assume that $X=(X^{t}_{i})^{(n,T)}_{(i,t)=(1,1)}$ is a vector of independent random variables where for all fixed $t$, $X^{t}_{1},\ldots, X^{t}_{n}$ are identically distributed according to $P_t$. Let the loss function $\ell$ be $L$-Lipschitz in its first argument. Then for every $x > 0$, with probability at least $1-e^{-x}$,
  \begin{align}
     \label{GRCBoundforMTLInq}
       P (\ell_{\boldsymbol{f}}- \ell_{\boldsymbol{f^{*}}}) \leq P_{n} (\ell_{\boldsymbol{f}} - \ell_{\boldsymbol{f^{*}}} )+ 2 L \mathfrak{R} 
       (\mathcal{F}) +  \sqrt{\frac{2Lbx}{nT}}.
     \end{align}
 \end{thrm}
  \noindent
 \begin{proof}
 As it has been shown in \cite{maurer2006}, the proof of this theorem is based on using McDiarmid's inequality for $Z$ defined in \thrmref{TalagrandforMTL}, and noticing that for the function class $\mathcal{F}$ with values in $[-b,b]$, it holds that $|Z- Z_{s,j}| \leq 2b/nT$.
 \end{proof}

It can be observed that, in order to obtain the excess risk bound in the above theorem, one has to bound the \ac {GRC} term $\mathfrak{R}(\mathcal{F})$ in \eqref{GRCBoundforMTLInq}. Therefore, we first upper-bound the \ac {GRC} of different hypothesis spaces considered in the previous sections. The proof of the results can be found in Appendix \ref{E}.
\begin{thrm}[Distribution-dependent \ac {GRC} bounds] \leavevmode
 \label{GRCFixedMappingStrongConvexity} 
 Assume that the conditions of \thrmref{GeneralLRCBoundForFixedMappingStrongConvex} hold. Then, the following results hold for the \ac {GRC} of the hypothesis spaces in \eqref{GroupHSForFixedMApping}, \eqref{HSForFixedMAppingSchatten} and \eqref{GraphRegHS}, respectively.
   \begin{itemize}
\item Group-norm regularizer: For any $ 1 \leq q \leq 2$ in \eqref{GroupHSForFixedMApping}, the \ac {GRC} of the function class $\mathcal{F}_{q}$ can be bounded as
    \begin{align}
    \label{GRCGroupq<2}
       \forall \kappa \in [q,2]: \quad \mathfrak{R}(\mathcal{F}_{q})  \leq  \sqrt{\frac{2 e {\kappa^{*}}^{2} R^{2}_{max}}{nT^{2}}  \left\|  \left(  \textbf{tr}\left( J_t\right)  \right)_{t=1}^{T}   \right\|_{\frac{\kappa^{*}}{2}} } + \frac{\sqrt{2 \mathcal{K} e} R_{max} \kappa^{*} T^{\frac{1}{\kappa^{*}}}}{nT}.
    \end{align}
 Also, for any $q \geq 2$ in \eqref{GroupHSForFixedMApping}, the \ac {GRC} of the function class $\mathcal{F}_{q}$ can be bounded as
        \begin{align}
        \label{GRCGroupq>2}
           \mathfrak{R}(\mathcal{F}_{q})  \leq  \sqrt{\frac{2 R^{2}_{max}}{nT^{2}}  \left\|  \left(  \textbf{tr}\left( J_t\right)  \right)_{t=1}^{T}   \right\|_{\frac{q^{*}}{2}}}.
        \end{align}
\item Schatten-norm regularizer: For any $ 1 \leq q \leq 2$ in \eqref{HSForFixedMAppingSchatten}, the \ac {GRC} of the function class $\mathcal{F}_{S_q}$ can be bounded as 
   \begin{align}
   \label{GRCShattenq<2}
    \mathfrak{R}(\mathcal{F}_{S_q}) \leq  \sqrt{ \frac{2  q^{*} R'^{2}_{max}}{nT^{2}} \left\|  \left(  \textbf{tr}\left( J_t\right)  \right)_{t=1}^{T} \right\|_{1}}.
   \end{align}
Also, for any $q \geq 2$ in \eqref{HSForFixedMAppingSchatten}, the \ac {GRC} of the function class $\mathcal{F}_{S_q}$ can be bounded as 
      \begin{align}
      \label{GRCShattenq>2}
       \mathfrak{R}(\mathcal{F}_{S_q}) \leq  \sqrt{ \frac{2 R'^{2}_{max}}{nT^{2}} \left\|  \left(  \textbf{tr}\left( J_t\right)  \right)_{t=1}^{T} \right\|_{1}}.
      \end{align}
\item Graph regularizer: For any positive operator $\boldsymbol{D}$ in \eqref{GraphRegHS}, the \ac {GRC} of the function class $\mathcal{F}_{G}$ can be bounded as
   \begin{align}
   \label{GRCGraphBound}
   \mathfrak{R}(\mathcal{F}_{G}) \leq  \sqrt{ \frac{2 R''^{2}_{max}} {nT^{2}}  \left\| \left( \boldsymbol{D}_{tt}^{-1} \textbf{tr} (J_{t})\right)_{t=1}^{T}  \right\|_{1} }.
   \end{align}
  \end{itemize} 
  where for the covariance operator $J_t = \mathbb{E}(\phi(X_t )\otimes \phi(X_t)) = \sum_{j=1}^{\infty} \lambda_t^j \boldsymbol{u}_{t}^j \otimes \boldsymbol{u}_{t}^{j}$, the \emph{trace} $\textbf{tr} (J_{t})$ is defined as
  \begin{align}
  \textbf{tr} (J_{t}) := \sum_{j} \left\langle J_{t} \boldsymbol{u}_{t}^j, \boldsymbol{u}_{t}^j \right\rangle = \sum_{j=1}^{\infty} \lambda_{t}^{j}. 
  \nonumber
  \end{align}
\end{thrm}
 Notice that, assuming a unique bound for the traces of all tasks' kernels, the bound in \eqref{GRCGroupq<2} is determined by $O\left( \frac{q^{*} T^{\frac{1}{q^{*}}}}{T \sqrt{n}} \right) $. Also, taking $q^{*}= \log T$, we obtain a bound of order $O \left( \frac{ \log T }{T\sqrt{n}}\right)$. We can also remark that, when the kernel traces are bounded, the bounds in \eqref{GRCGroupq>2}, \eqref{GRCShattenq<2}, \eqref{GRCShattenq>2} and \eqref{GRCGraphBound} are of the order of $O\left( \frac{1}{\sqrt{nT}}\right)$.
  
Note that for the purpose of comparison, we concentrate only on the parameters $R, n, T, q^{*}$ and $\alpha$ and assume all the other parameters are fixed and hidden in the big-$O$ notation. Also, for the sake of simplicity, we assume that the eigenvalues of all tasks satisfy $\lambda_{t}^{j} \leq d j^{-\alpha}$ (with $\alpha > 1$). Note that from \thrmref{GRCBoundforMTL}, it follows that a bound on the global Rademacher complexity provides also a bound on the excess risk. This together with \thrmref{GRCFixedMappingStrongConvexity}, gives the \ac {GRC}-based excess risk bounds of the following forms (note that $q \geq 1$)
 \begin{align}
 \label{GRCExcessRate}
    \text{Group norm:}\qquad & \text{(a)} \quad \forall \kappa \in [q,2],&
        P (\ell_{\boldsymbol{\hat{f}}} - \ell_{\boldsymbol{f}^{*}})  & =  O \left(  (R^{2}_{max}{\kappa^{*}}^{2})^{\frac{1}{2}} \left( T^{\frac{2}{\kappa} }\right)^{-\frac{1}{2}} n^{-\frac{1}{2}} \right).
       \nonumber\\
    & \text{(b)} \quad \forall q \in [2,\infty],&
       P (\ell_{\boldsymbol{\hat{f}}} - \ell_{\boldsymbol{f}^{*}})  & =  O \left(  (R^{2}_{max})^{\frac{1}{2}} \left( T^{\frac{2}{q} }\right)^{-\frac{1}{2}} n^{-\frac{1}{2}} \right).
       \nonumber\\       
   \text{Schatten-norm:}\quad & \text{(c)} \quad \forall q \in [1,2],&
      P (\ell_{\boldsymbol{\hat{f}}} - \ell_{\boldsymbol{f}^{*}})  & = O \left( (R'^{2}_{max} q^{*})^{\frac{1}{2}}   T^{-\frac{1}{2}} n^{-\frac{1}{2}} \right).
     \nonumber\\
   & \text{(d)} \quad \forall q \in [2,\infty],&
      P (\ell_{\boldsymbol{\hat{f}}} - \ell_{\boldsymbol{f}^{*}})  & = O \left( (R'^{2}_{max})^{\frac{1}{2}}   T^{-\frac{1}{2}} n^{-\frac{1}{2}} \right).
     \nonumber\\             
  \text{Graph regularizer:} \quad  &\text{(e)} \quad &
       P (\ell_{\boldsymbol{\hat{f}}} - \ell_{\boldsymbol{f}^{*}})  & = O \left(  (R''^{2}_{max})^{\frac{1}{2}}   T^{-\frac{1}{2}} n^{-\frac{1}{2}} \right). 
 \end{align}
 \noindent
which can be compared to their \ac {LRC}-based counterparts as following
 \begin{align}
 \label{LRCExcessRate}
    \text{Group norm:} \qquad & \text{(a)} \quad  \forall \kappa \in [q,2],&
         P (\ell_{\boldsymbol{\hat{f}}} - \ell_{\boldsymbol{f}^{*}})   &= O \left(  (R^{2}_{max}{\kappa^{*}}^{2})^{\frac{1}{1+ \alpha}} \left( T^{\frac{2}{\kappa}}\right)^{-\frac{1}{1+ \alpha}} n^{\frac{-\alpha}{1 + \alpha}} \right).
       \nonumber\\ 
    & \text{(b)} \quad  \forall q \in [2,\infty],&
         P (\ell_{\boldsymbol{\hat{f}}} - \ell_{\boldsymbol{f}^{*}})   &= O \left(  (R^{2}_{max})^{\frac{1}{1+ \alpha}} \left( T^{\frac{2}{q}}\right)^{-\frac{1}{1+ \alpha}} n^{\frac{-\alpha}{1 + \alpha}} \right).
       \nonumber\\              
   \text{Schatten-norm:}\quad & \text{(c)} \quad  \forall q \in [1,2],& 
      P (\ell_{\boldsymbol{\hat{f}}} - \ell_{\boldsymbol{f}^{*}})   &= O \left( (R'^{2}_{max} q^{*})^{\frac{1}{1+ \alpha}}   T^{\frac{-1}{1+ \alpha}} n^{\frac{-\alpha}{1 + \alpha}} \right).
     \nonumber\\
   & \text{(d)} \quad   \forall q \in [2,\infty],& 
      P (\ell_{\boldsymbol{\hat{f}}} - \ell_{\boldsymbol{f}^{*}})   &= O \left( (R'^{2}_{max})^{\frac{1}{1+ \alpha}}   T^{\frac{-1}{1+ \alpha}} n^{\frac{-\alpha}{1 + \alpha}} \right).
     \nonumber\\                
   \text{Graph regularizer:} \quad & \text{(e)} \quad  &
       P (\ell_{\boldsymbol{\hat{f}}} - \ell_{\boldsymbol{f}^{*}})  & = O \left(  (R''^{2}_{max})^{\frac{1}{1+ \alpha}}   T^{\frac{-1}{1+ \alpha}} n^{\frac{-\alpha}{1 + \alpha}} \right).
 \end{align}
 It can be seen that holding all the parameters fixed when $n$ approaches to infinity, the local bounds yield faster rates, since $\alpha > 1$. However, when $T$ grows to infinity, the convergence rate of the local bounds could be only as good as those obtained by the global analysis. 
 
 A close appraisal of the results in \eqref{GRCExcessRate} and \eqref{LRCExcessRate} points to a conservation of asymptotic rates between $n$ and $T$, when all other remaining quantities are held fixed. This phenomenon is more apparent for the Schatten norm and graph-based regularization cases. It can be seen that, for both the global and local analysis results, the rates (exponents) of $n$ and $T$ sum up to $-1$. In the local analysis case, the trade-off is determined by the value of $\alpha$, which can facilitate faster $n$-rates and compromise with slower $T$-rates. A similar trade-off is witnessed in the case of group norm regularization, but this time between $n$ and $T^{2/\kappa}$, instead of $T$, due to specific characteristic of the group norm.  
 
 As mentioned earlier in \remref{NonMonoBoundStrongConvexity}, the bounds for the class of group norm regularizer for $1 \leq q \leq 2$ is not monotonic in $q$; they are minimized for $q^{*} = \log T$. Therefore, we split our analysis for this case as follows: 
 \begin{enumerate}
 \item First, we consider $q^{*} \geq \log T$, which leads to the optimal choice $\kappa^{*}=q^{*}$, and taking the minimum of the global and local bounds gives
 \begin{align}
 \! \! \! P (\ell_{\boldsymbol{\hat{f}}} - \ell_{\boldsymbol{f}^{*}})  & \leq \! O \! \left( \! \min \! \left\lbrace  (R_{max}q^{*}) ( T^{\frac{2}{q} })^{-\frac{1}{2}} n^{-\frac{1}{2}},   (R_{max}q^{*})^{\frac{2}{1+ \alpha}} ( T^{\frac{2}{q}})^{-\frac{1}{1+ \alpha}} n^{\frac{-\alpha}{1 + \alpha}} \right\rbrace  \right). 
  \end{align} 
 
 It is worth mentioning that, for any value of $\alpha > 1$, if the number of tasks $T$ as well as the radius $R_{max}$ of the $L_{2,q}$ ball can grow with $n$, the local bound improves over the global one whenever $\frac{T^{1/q}}{R_{max}} = O (\sqrt{n})$.
 
 \item Secondly, assume that $q^{*} \leq \log T$, in which case the best choice is $\kappa^{*} =  \log T$. Then, the excess risk bound reads
  \begin{align}
   P (\ell_{\boldsymbol{\hat{f}}} - \ell_{\boldsymbol{f}^{*}})  & \leq  O \left(  \min \left\lbrace    \left( \frac{ R_{max} \log{T}}{T}\right)  n^{-1/2}, \left( \frac{R_{max} \log{T}}{T}\right) ^{\frac{2}{1 + \alpha}} n^{\frac{-\alpha}{1 + \alpha}}\right\rbrace   \right), 
   \end{align}
and the local analysis improves over the global one, when $ \frac{T}{R_{max} \log{T}} = O(\sqrt{n})$. 
\end{enumerate}

Also, a similar analysis for Schatten norm and graph regularized hypothesis spaces shows that the local analysis is beneficial over the global one, whenever the number of tasks $T$ and the radius $R$ can grow, such that $\frac{\sqrt{T}}{R} = O(\sqrt{n})$.

%
%

\subsection{Comparisons to Related Works}
\label{ssec:RelatedWorks}

Also, it would be interesting to compare our (global and local) results for the trace norm regularized \ac {MTL} with the \ac {GRC}-baesd excess risk bound provided in \cite{Maurer2013} wherein they apply a trace norm regularizer to capture the tasks' relatedness. It is worth mentioning that they consider a slightly different hypothesis space for $\boldsymbol{W}$, which in our notation reads as
\begin{align}
\label{HSMaurer2013}
 \mathcal{F}_{S_1} := \left\lbrace \boldsymbol{W}: \frac{1}{2} \left\| \boldsymbol{W} \right\|_{S_1}^{2} \leq T R'^{2}_{max}  \right\rbrace.
\end{align}
It is based on the premise that, assuming a common vector $\boldsymbol{w}$ for all tasks, the regularizer should not be a function of number of tasks \citep{Maurer2013}. Given the task-averaged covariance operator $C :=1/T \sum_{t=1}^{T} J_{t}$, the excess risk bound in \cite{Maurer2013} reads as
\begin{align}
P (\ell_{\boldsymbol{\hat{f}}} - \ell_{\boldsymbol{f}^{*}}) & \leq 2 \sqrt{2} L R'_{max} \left( \sqrt{\frac{\left\| C \right\|_{\infty} }{n}}  +  5 \sqrt{\frac{\ln(nT) + 1}{nT}}\right)+ \sqrt{\frac{bLx}{nT}}.
\nonumber 
\end{align}
where loss function $\ell$ is L-Lipschitz and $\mathcal{F}$ has ranges in $[-b,b]$. One can easily verify that the trace norm is a Schatten norm with $q=1$. Note that for any $q \geq 1$ it holds that $\mathcal{F}_{S_1} \subseteq \mathcal{F}_{S_q}$, which implies $\mathfrak{R}(\mathcal{F}_{S_1}) \leq \mathfrak{R}(\mathcal{F}_{S_q})$. This fact, in conjunction with \thrmref{GRCFixedMappingStrongConvexity} and \thrmref{GRCBoundforMTL} (applied to the class of excess loss functions) yields a \ac {GRC}-based excess risk bound. Therefore, considering the trace norm hypothesis space \eqref{HSMaurer2013} and the optimal value of $q^{*}=2$, translates our global and local bounds to the following
\begin{enumerate}
   \item \ac {GRC}-based excess risk bound:
      \begin{align}
      P (\ell_{\boldsymbol{\hat{f}}} - \ell_{\boldsymbol{f}^{*}}) & \leq 4 L R'_{max} \sqrt{\frac{\left\|  \left(  \textbf{tr}\left( J_t\right)  \right)_{t=1}^{T} \right\|_{1}}{nT}} + \sqrt{\frac{bLx}{nT}}.
      \nonumber
     \end{align}
   \item  \ac {LRC}-based excess risk bound ($\forall \alpha > 1$): 
       \begin{align}
       \label{ExcessLocalBoundForTraceNorm}
      P (\ell_{\boldsymbol{\hat{f}}} - \ell_{\boldsymbol{f}^{*}}) \leq  256 K \sqrt{\frac{\alpha +1}{\alpha-1}}\left(2 d R'^{2}_{max} L^{2} \right)^{\frac{1}{1+\alpha}} B^{\frac{\alpha -1}{\alpha+1}} n^{\frac{-\alpha}{1+\alpha}}  + \frac{(48Lb+ 16 B K) x}{n T}.
      \end{align} 
 \end{enumerate}  
Now, assume that each operator $J_{t}$ is of rank $M$ and denote its maximum eigenvalue by $\lambda_{t}^{max}$. If $\lambda_{max} := \max_{t \in \mathbb{N}_{T}} \left\lbrace \lambda_{t}^{max}\right\rbrace$, then it is easy to verify that $\textbf{tr}(J_{t}) \leq M \lambda_{t}^{max}$ and $\left\| C \right\|_{\infty} \leq \lambda_{max}$, which leads to the following \ac {GRC}-based bounds
  \begin{align}
  \label{ExcessGlobalBoundForTraceNorm}
   \text {Ours:}  \qquad P (\ell_{\boldsymbol{\hat{f}}} - \ell_{\boldsymbol{f}^{*}})  & \leq 4 L R'_{max} \sqrt{\frac{M \lambda_{max} }{n}} + \sqrt{\frac{bLx}{nT}},
   \\
   \label{ExcessGlobalBoundForTraceNormMaurer}
   \text {\cite{Maurer2013}:} \: P (\ell_{\boldsymbol{\hat{f}}} - \ell_{\boldsymbol{f}^{*}})  & \leq  2 \sqrt{2} L R'_{max} \left( \sqrt{\frac{\lambda_{max} }{n}}  +  5 \sqrt{\frac{\ln(nT) + 1}{nT}}\right)+ \sqrt{\frac{bLx}{nT}}.
  \end{align}
  
  One can observe that as $n \rightarrow \infty$, in all cases the bound vanishes. However, it does so at a rate of $n^{-\alpha/1+\alpha}$ for our local bound in \eqref{ExcessLocalBoundForTraceNorm}, at a slower rate of $1/\sqrt{n}$ for our global bound in \eqref{ExcessGlobalBoundForTraceNorm}, and at the slowest rate of $\sqrt{\ln n /n}$ for the one in \eqref{ExcessGlobalBoundForTraceNormMaurer}. 
  
  We remark that, as $T \rightarrow \infty$, all bounds converge to a non-zero limit: our local bound in \eqref{ExcessLocalBoundForTraceNorm} at a fast rate of $1/T$,  the one in \eqref{ExcessGlobalBoundForTraceNorm} at a slower rate of $\sqrt{1/T}$, and the bound in \eqref{ExcessGlobalBoundForTraceNormMaurer} at a the slowest rate of $\sqrt{\ln T/ T}$. 
  
Another interesting comparison can be performed between our bounds and the one introduced by \cite{maurer2006Graph} for a graph regularized \ac{MTL}. For this purpose we consider the following hypothesis space similar to what has been considered by \cite{maurer2006Graph} 
\begin{align}
\label{ModifiedGraphHS}
 \mathcal{F}_{G} = \left\lbrace \boldsymbol{W}: \frac{1}{2} \left\|\boldsymbol{D}^{1/2} \boldsymbol{W} \right\|_{F}^{2} \leq T R''^{2}_{max}\right\rbrace. 
\end{align}

\cite{maurer2006Graph} provides a bound on the empirical \ac {GRC} of the aforementioned hypothesis space. However, similar to the proof of \corref{GraphSC}, we can easily convert it to a distribution dependent \ac {GRC} bound which matches our global bound in \eqref{GRCGraphBound} (for the defined hypothesis space \eqref{ModifiedGraphHS}) and in our notation reads as 
\begin{align}  
\mathfrak{R} \left( \mathcal{F}_{G} \right) \leq \sqrt{\frac{2 R''^{2}_{max}}{nT} \left\| \left( \boldsymbol{D}_{tt}^{-1}  \textbf{tr} (J_{t}) \right)_{t=1}^{T}  \right\|_{1}}. 
\nonumber
\end{align}
\noindent
Now, with $\boldsymbol{D} := \boldsymbol{L} + \eta \boldsymbol{I}$ (where $\boldsymbol{L}$ is the graph-Laplacian, $\boldsymbol{I}$ is the identity operator, and $\eta > 0$ is a regularization parameter) and the assumption that the $J_{t}$s are of rank $M$, it can be shown that
\begin{align}
& \left\| \left( \boldsymbol{D}_{tt}^{-1}  \textbf{tr} (J_{t}) \right)_{t=1}^{T}  \right\|_{1} = \sum_{t=1}^{T} \boldsymbol{D}_{tt}^{-1} \textbf{tr} (J_{t}) \leq M \lambda_{max} \left( \sum_{t=1}^{T} \boldsymbol{D}_{tt}^{-1} \right) =   M \lambda_{max} \textbf{tr} \left( \boldsymbol{D}^{-1}\right) =
\nonumber\\
& = M \lambda_{max} \textbf{tr} \left( \boldsymbol{L} + \eta \boldsymbol{I} \right)^{-1} =  M \lambda_{max} \left( \sum_{t=1}^{T} \frac{1}{\delta_{t} + \eta} + \frac{1}{\eta} \right) \leq M \lambda_{max} \left( \frac{T}{\delta_{min} + \eta} + \frac{1}{\eta} \right). 
\nonumber
\end{align}
\noindent
where $\lambda_{max}$ is defined as before. Also, we define $\left( \delta_{t}\right)_{t=1}^{T}$ as the eigenvalues of Laplacian matrix $\boldsymbol{L}$ with $\delta_{min} := min_{t \in \mathbb{N}_{T}} \delta_{t}$.
Therefore, the \underline{matching} \ac {GRC}-based excess risk bounds can be obtained as
 \begin{align}
  \label{ExcessGlobalBoundForGraphNorm}
   \text {Ours \& \cite{maurer2006Graph} :} \qquad P (\ell_{\boldsymbol{\hat{f}}} - \ell_{\boldsymbol{f}^{*}})  & \leq \frac{2 L R''_{max}}{\sqrt{n}} \sqrt{2 M \lambda_{max}   \left( \frac{1}{\delta_{min} } + \frac{1}{ T \eta} \right) } + \sqrt{\frac{bLx}{nT}}.
    \nonumber \\
  \end{align}
Also, from \remref{ExcessRiskBoundsSC}, the \ac {LRC}-based bound is given as
   \begin{align} 
   \label{ExcessLocalBoundForGraphNorm}
       P (\ell_{{\boldsymbol{\hat{f}}}} - \ell_{\boldsymbol{f}^{*}}) & \leq   256 K \sqrt{\frac{\alpha +1}{\alpha-1 }}\left( d R''^{2}_{max} L^{2} \boldsymbol{D}_{max}^{-1} \right)^{\frac{1}{1+\alpha}} B^{\frac{\alpha-1}{\alpha+1}}  n^{\frac{-\alpha}{1+\alpha}} 
       + \frac{(48Lb+ 16 B K)x}{n T}.
   \end{align}    
The above results show that when $n \rightarrow \infty$, both \ac{GRC} and \ac{LRC} bounds approach zero, albeit, the global bound with a rate of $\sqrt{1/n}$, and the local one with a faster rate of $n^{-\alpha/\alpha+1}$, since $\alpha > 1$. Also, as $T \rightarrow \infty$, both bounds approach non-zero limits. However, the global bound does so at a rate of $\sqrt{1/T}$ and the local one at a faster rate of $1/T$.


\newpage
\appendix

\numberwithin{section}{part}

\numberwithin{equation}{section}
\numberwithin{thrm}{section}
\numberwithin{figure}{section}
\numberwithin{table}{section}
\renewcommand{\thesection}{{\Alph{section}}}
\renewcommand{\thesubsection}{\Alph{section}.\arabic{subsection}}
\renewcommand{\thesubsubsection}{\Roman{section}.\arabic{subsection}.\arabic{subsubsection}}
\setcounter{secnumdepth}{-1}
\setcounter{secnumdepth}{3}

\section*{Appendices}
\section{Proofs of the results in \sref{sec:Talagrand-Inequalities}: ``Talagrand-Type Inequality for Multi-Task Learning"}
\label{A}
This section presents the proof of Theorem \ref{TalagrandforMTL}. We first provide some useful foundations used in the derivation of our result in \thrmref{TalagrandforMTL}.
\begin{thrm}[Theorem 2 in \cite{boucheron2003entropy}]
\label{SobolevInq}
Let $X_{1},\ldots,X_{n}$ be $n$ independent random variables taking values in a measurable space $\xcal$. Assume that $g:\xcal^{n} \to \rbb$ is a measurable function and  $Z:=g(X_{1},\ldots,X_{n})$. Let $X'_{1},\ldots,X'_{n}$ denote an independent copy of $X_{1},\ldots,X_{n}$, and $Z'_{i} := g(X_{1},\ldots,X_{i-1},X'_{i},X_{i+1},\ldots,X_{n})$ which is obtained by replacing the variable $X_{i}$ with $X'_{i}$. Define the random variable $V^{+}$ by
\begin{align}
V^{+} := \sum_{i=1}^{n} \ebb' \big[ \big(Z-Z'_{i}\big)^{2}_{+}\big].
\nonumber
\end{align}
\noindent
where $(u)_{+} := \max \{u,0\}$, and $\ebb' [\cdot] := \ebb[\cdot| X]$ denotes the expectation only with respect to the variables $X'_{1},\ldots,X'_{n}$.
Let $\theta >0$ and $ \lambda \in (0,1/\theta)$. Then,
\begin{align}
\log \ebb \big(e^{\lambda (Z-\ebb Z)}\big) \leq \frac{\lambda \theta}{1-\lambda \theta} \log \ebb \big[\exp \big( \frac{\lambda V^{+}}{\theta} \big)\big].
\nonumber
\end{align}
\end{thrm}

\begin{defin}[Section 3.3 in \cite{boucheron2013book}]
\label{b-SelfBoundingFunction}
A function $g: \xcal^{n} \to [0, \infty)$ is said to be $b$-self bounding ($b>0$), if there exist functions $g_{i}:\xcal^{n-1} \to \rbb$, such that for all $X_{1},\ldots,X_{n} \in \xcal$ and all $i \in \nbb_{n}$,
\begin{align}
0 \leq  g(X_{1},\ldots,X_{n})  - g_{i} (X_{1},\ldots,X_{i-1},X_{i+1},\ldots,X_{n}) \leq b,
\nonumber
\end{align}
and,
\begin{align}
\sum_{i=1}^{n} \big[ g(X_{1},\ldots,X_{n})  - g_{i} (X_{1},\ldots,X_{i-1},X_{i+1},\ldots,X_{n}) \big] \leq g(X_{1},\ldots,X_{n}).
\nonumber
\end{align}
\end{defin}

\begin{thrm}[Theorem 6.12 in \cite{boucheron2013book}]
\label{Self-boundindSobolev}
Assume that $Z=g(X_1,\ldots,X_n)$ is a $1$-self bounding function. Then for every $\lambda \in \rbb$,
\begin{equation}\label{eq:self-bounding-1}
  \log \ebb e^{\lambda (Z - \ebb Z)} \leq \phi (\lambda) \ebb Z,
\end{equation}
where $\phi (\lambda) = e^{\lambda} - \lambda -1$.
\end{thrm}

\begin{corr}\label{Self-BoundRange}
Assume that $Z=g(X_1,\ldots,X_n)$ is a $b$-self bounding function ($b>0$). Then, for any $\lambda\in\rbb$ we have
$$
\log \ebb e^{\lambda Z} \leq \frac{\big( e^{\lambda b} - 1\big)}{b} \ebb Z.
$$
\end{corr}
\begin{proof}
Note that Eq. \eqref{eq:self-bounding-1} can be rewritten as
$
  \log\ebb[\exp(\lambda Z)]\leq (e^\lambda-1)\ebb Z.
$
The stated inequality follows immediately by rescaling $Z$ to $Z/b$ in the above inequality.
\end{proof}

\begin{lemm}[Lemma 2.11 in \cite{bousquet2002Thesis}]
 \label{BousquetLemma}
 Let $Z$ be a random variable, $A,B >0$ be some constants. If for any $\lambda \in (0, 1/B)$ it holds
 \begin{align}
 \log \ebb \big(e^{\lambda (Z-\ebb Z)}\big) \leq \frac{A \lambda^{2} }{2 \big(1 - B \lambda \big)},
 \nonumber
 \end{align}
 then for all $x \geq 0$,
 \begin{align}
 P \big[ Z \geq \ebb Z + \sqrt{2Ax}  + Bx \big] \leq e^{-x}.
 \nonumber
 \end{align}
\end{lemm}

\begin{lemm}[Contraction property,~\cite{bartlett2005}]\label{lem:contraction inequality}
  Let $\phi$ be a Lipschitz function with constant $L \geq 0$, that is, $|\phi(x)-\phi(y)|\leq L|x-y|$, $\forall x,y \in \rbb$. Then for every real-valued function class $\fcal$, it holds
  \begin{equation}\label{comprison-inequality}
    \ebb_\sigma \mathfrak{R}(\phi\circ\fcal)\leq L\ebb_\sigma \mathfrak{R}(\fcal),
  \end{equation}
  where $\phi\circ\fcal:=\{\phi\circ f:f\in\fcal\}$ and $\circ$ is the composition operator.
\end{lemm}
Note that in Theorem 17 of \cite{maurer2006}, it has been shown that the result of this lemma also holds for the class of vector-valued functions.
\section*{Proof of \thrmref{TalagrandforMTL}}\label{sec:ProofTheorem}
Before laying out the details, we first provide a sketch of the proof. Defining 
\begin{equation}\label{Z}
  Z:=  \sup_{\boldf \in \fcal} \Big[  \frac{1}{T} \sum_{t=1}^{T} \frac{1}{N_{t}} \sum_{i=1}^{N_{t}} [\ebb f_{t}(X_{t}^{i}) - f_{t}(X_{t}^{i})] \Big],
\end{equation}
we first apply \thrmref{SobolevInq} to control the log-moment generating function $\log \ebb \big(e^{\lambda (Z-\ebb Z)}\big)$. From \thrmref{SobolevInq}, we know that the main component to control $\log \ebb \big(e^{\lambda (Z-\ebb Z)}\big)$ is the variance-type quantity $V^{+} = \sum_{s=1}^{T} \sum_{j=1}^{N_{s}} \ebb' \big[ \big( Z - Z'_{s,j} \big)^{2}_{+} \big]$. In the next step, we show that $V^{+}$ can also be bounded in terms of two other quantities denoted by $W$ and $\Upsilon$. Applying \thrmref{SobolevInq} for a specific value of $\theta$, then gives a bound for $\log \ebb \big(e^{\lambda (Z-\ebb Z)}\big)$ in terms of $\log\ebb[e^{\frac{\lambda}{b'}(W+\Upsilon)}]$. We then turn to controlling $W$ and $\Upsilon$, respectively. Our idea to tackle $W$ is to show that it is a self-bounding function, according to which we can apply Corollary \ref{Self-BoundRange} to control $\log\ebb[e^{\frac{\lambda W}{b'}}]$. The term $\Upsilon$ is closely related to the constraint imposed on the variance of functions in $\fcal$, and can be easily upper bounded in terms of $r$. We finally apply Lemma \ref{BousquetLemma} to transfer the upper bound on the log-moment generating function $\log \ebb \big(e^{\lambda (Z-\ebb Z)}\big)$ to the tail probability on $Z$. To clarify the process we divide the proof into four main steps.

\textbf{Step 1. Controlling the log-moment generating function of $Z$ with the random variable $W$ and variance $\Upsilon$}.
Let $X' : =(X'^{i}_{t})_{(t,i)=(1,1)}^{(T,N_{t})}$ be an independent copy of $X : =(X^{i}_{t})_{(t,i)=(1,1)}^{(T,N_{t})}$. Define the quantity

\begin{multline}\label{Z-s-j}
  Z'_{s,j} :=  \sup_{\boldf \in \fcal}\Big[  \frac{1}{TN_{s}}\big[\ebb' f_{s}(X'^{j}_{s}) -f_{s}(X'^{j}_{s})\big] -   \frac{1}{TN_{s}}\big[ \ebb f_{s}(X^{j}_{s}) - f_{s}(X^{j}_{s})\big] \\
  +  \frac{1}{T} \sum_{t=1}^{T} \frac{1}{N_{t}} \sum_{i=1}^{N_{t}} [\ebb f_{t}(X_{t}^{i}) - f_{t}(X_{t}^{i})] \Big],
\end{multline}
where $Z'_{s,j}$ is obtained from $Z$ by replacing the variable $X_{s}^{j}$ with $X'^{j}_{s}$. Let $\boldsymbol{\hat{f}} : = (\hat{f_{1}},\ldots \hat{f}_{T})$ be such that $Z= \frac{1}{T} \sum_{t=1}^{T} \frac{1}{N_{t}} \sum_{i=1}^{N_{t}} \big[ \mathbb{E}\hat{f_t}(X_t^i) - \hat{f_t}(X_{t}^{i})\big]$, and introduce
   \begin{gather*}
    W  := \sup_{\boldf \in \fcal} \Big[  \frac{1}{T^{2}} \sum_{t=1}^{T} \frac{1}{N^{2}_{t}} \sum_{i=1}^{N_{t}} [\ebb f_{t}(X_{t}^{i}) - f_{t}(X_{t}^{i})]^{2} \Big],  \\
    \Upsilon  : = \sup_{\boldf \in \fcal} \Big[  \frac{1}{T^{2}} \sum_{t=1}^{T} \frac{1}{N^{2}_{t}} \sum_{i=1}^{N_{t}} \ebb[\ebb f_{t}(X_{t}^{i}) - f_{t}(X_{t}^{i})]^{2} \Big].
   \end{gather*}
    It can be shown that for any $j \in \nbb_{n}$ and any $s \in \nbb_{T}$:
    $$
      Z-Z'_{s,j}\leq  \frac{1}{TN_{s}}\big[ \ebb \hat{f}_{s}(X^{j}_{s}) - \hat{f}_{s}(X^{j}_{s})\big] - \frac{1}{TN_{s}}\big[\ebb' \hat{f}_{s}(X'^{j}_{s}) -\hat{f}_{s}(X'^{j}_{s})\big]
    $$
    and therefore
    $$
      (Z-Z'_{s,j})^{2}_{+} \leq \frac{1}{T^{2}N^{2}_{s}} \big( [ \ebb \hat{f}_{s}(X^{j}_{s}) - \hat{f}_{s}(X^{j}_{s})] -  [\ebb' \hat{f}_{s}(X'^{j}_{s}) -\hat{f}_{s}(X'^{j}_{s})]\big)^{2}.
    $$
    Then, it follows from the identity $\ebb' [\ebb' \hat{f}_{s}(X'^{j}_{s}) - \hat{f}_{s}(X'^{j}_{s})]=0$ that
   \begin{align*}
   \sum_{s=1}^{T} & \sum_{j=1}^{N_{s}}  \ebb' \big[ \big( Z  - Z'_{s,j} \big)^{2}_{+} \big]  \leq
   \sum_{s=1}^{T} \sum_{j=1}^{N_{s}} \frac{1}{T^{2}N^{2}_{s}} \ebb' \Big[\Big( [ \ebb \hat{f}_{s}(X^{j}_{s}) - \hat{f}_{s}(X^{j}_{s})] -  [\ebb' \hat{f}_{s}(X'^{j}_{s}) - \hat{f}_{s}(X'^{j}_{s})]\Big)^{2}\Big]
   \\
   & = \sum_{s=1}^{T} \sum_{j=1}^{N_{s}} \frac{1}{T^{2}N^{2}_{s}} [ \ebb \hat{f}_{s}(X^{j}_{s}) - \hat{f}_{s}(X^{j}_{s})]^{2} + \sum_{s=1}^{T} \sum_{j=1}^{N_{s}} \frac{1}{T^{2}N^{2}_{s}} \ebb'[\ebb' \hat{f}_{s}(X'^{j}_{s}) - \hat{f}_{s}(X'^{j}_{s})]^{2}
  \\
  & \leq \sup_{\boldf \in \fcal} \sum_{s=1}^{T} \sum_{j=1}^{N_{s}} \frac{1}{T^{2}N^{2}_{s}} [ \ebb f_{s}(X^{j}_{s}) - f_{s}(X^{j}_{s})]^{2} + \sup_{\boldf \in \fcal} \sum_{s=1}^{T} \sum_{j=1}^{N_{s}} \frac{1}{T^{2}N^{2}_{s}} \ebb [ \ebb f_{s}(X^{j}_{s}) - f_{s}(X^{j}_{s})]^{2}
  \\
   & =  W + \Upsilon.
   \end{align*}
   Introduce $b':= \frac{2 b}{n T}$. Applying \thrmref{SobolevInq} and the above bound on $\sum_{s=1}^{T} \sum_{j=1}^{N_{s}}  \ebb' \big[ \big( Z  - Z'_{s,j} \big)^{2}_{+} \big]$ then gives the following bound on the log-moment generating function of $Z$:
   \begin{equation}
   \label{Sobolev}
   \log \ebb \big(e^{\lambda (Z-\ebb Z)}\big) \leq \frac{\lambda b'}{1-\lambda b'} \log \ebb e^{ \frac{\lambda}{b'}(W + \sigma^{2}) },\quad \forall\lambda \in (0,1/b').
   \end{equation}

   \textbf{Step 2. Controlling the log-moment generating function of $W$.}  We now upper bound the log-moment generating function of $W$ by showing that it is a self-bounding function.
   For any $s\in\nbb_T,j\in\nbb_{N_s}$, introduce 
     $$
     W_{s,j} :=  \sup_{\boldf \in \fcal} \Big[  \frac{1}{T^{2}} \sum_{t=1}^{T} \frac{1}{N^{2}_{t}} \sum_{i=1}^{N_{t}} [\ebb f_{t}(X_{t}^{i}) - f_{t}(X_{t}^{i})]^{2}  -  \frac{1}{T^{2}N^{2}_{s}} [ \ebb f_{s}(X^{j}_{s}) - f_{s}(X^{j}_{s})]^{2} \Big].
     $$
   Note that $W_{s,j}$ is a function of $\{X_t^i,t\in\nbb_T,i\in\nbb_t\}\backslash\{X_s^j\}$. 
   Letting $\boldsymbol{\tilde{f}}:=(\tilde{f_{1}},\ldots,\tilde{f_{T}})$ be the function achieving the supremum in the definition of $W$, it can be checked that (note that $b'=\frac{2b}{nT}$)
   \begin{align}
   T^2[W - W_{s,j}] \leq  \frac{1}{N^{2}_{s}} [ \ebb \tilde{f}_{s}(X^{j}_{s}) - \tilde{f}_{s}(X^{j}_{s})]^{2} \leq \frac{4 b^{2}}{n^{2}}=T^2b'^{2}.
   \end{align}
     Similarly, if $\boldsymbol{\tilde{f}}^{s,j} := (\tilde{f}^{s,j}_{1} \ldots, \tilde{f}^{s,j}_{T})$ is the function achieving the supremum in the definition of $W_{s,j}$, then one can derive the following inequality
     $$
      T^2[W - W_{s,j}] \geq  \frac{1}{N^{2}_{s}} [ \ebb \tilde{f}^{s,j}_{s}(X^{j}_{s}) - \tilde{f}^{s,j}_{s}(X^{j}_{s})]^{2} \geq 0.
     $$
   Also, it can be shown that
  \begin{align}
  \sum_{s=1}^{T} \sum_{i=1}^{N_{s}}  W - W_{s,j}  & \leq \frac{1}{T^{2}} \sum_{s=1}^{T} \frac{1}{N_{s}^{2}} \sum_{i=1}^{N_{s}} [ \ebb \tilde{f}_{s}(X^{j}_{s}) - \tilde{f}_{s}(X^{j}_{s})]^{2}
  \nonumber\\ 
  & = \sup_{\boldf \in \fcal} \Big[  \frac{1}{T^{2}} \sum_{t=1}^{T} \frac{1}{N^{2}_{t}} \sum_{i=1}^{N_{t}} [\ebb f_{t}(X_{t}^{i}) - f_{t}(X_{t}^{i})]^{2} \Big].
  \end{align}
  Therefore (according to \defref{b-SelfBoundingFunction}), $W/b'$ is a $b'$-self bounding function. Applying  \corref{Self-BoundRange} then gives the following inequality for any $\lambda \in (0, 1/b')$:
   \begin{equation}\label{bernstein-mtl-1}
    \log \ebb e^{\lambda (W/b')}   \leq  \frac{(e^{\lambda b'} -1)}{b'^{2}} \ebb W = \frac{(e^{\lambda b'} -1)}{b'^{2}} \Sigma^{2}
    \leq \frac{\lambda\Sigma^{2}}{b'(1-\lambda b')},
   \end{equation}
    where we introduce $\Sigma^{2} := \ebb W$ and the last step uses the inequality $(e^{x} -1) (1 - x) \leq x,\forall x \in [0,1]$. 
    Furthermore, the term $\Sigma^2$ can be controlled as follows: (here $(\sigma_t^i)$ is a sequence of independent Rademacher variables, independent of $X_{t}^{i}$):
   \begin{align*}
     \Sigma^2 & \leq \frac{1}{T^2}\ebb_X\sup_{\boldf\in\fcal}\Big[\sum_{t=1}^{T}\frac{1}{N_t^2}\sum_{i=1}^{N_t}\big[\ebb f_t(X_t^i)-f_t(X_t^i)\big]^2-\sum_{t=1}^{T}\frac{1}{N_t^2}\sum_{i=1}^{N_t}\ebb\big[\ebb f_t(X_t^i)-f_t(X_t^i)\big]^2\Big] + \Upsilon\\
      & \leq 2\ebb_{X, \sigma}\Big[\sup_{\boldf\in\fcal}\frac{1}{T^2}\sum_{t=1}^{T}\frac{1}{N_t^2}\sum_{i=1}^{N_t}\sigma_t^i\big[\ebb f_t(X_t^i)-f_t(X_t^i)\big]^2\Big] + \Upsilon\\
      & \leq 8b\ebb_{X, \sigma}\Big[\sup_{\boldf\in\fcal}\frac{1}{T^2}\sum_{t=1}^{T}\frac{1}{N_t^2}\sum_{i=1}^{N_t}\sigma_t^i\big[\ebb f_t(X_t^i)-f_t(X_t^i)\big]\Big]+\Upsilon \\
      & \leq \frac{16b \mathfrak{R}(\fcal)}{nT}+\Upsilon,
   \end{align*}
   where the first inequality follows from the definition of $W$ and $\Upsilon$, and the second inequality follows from the standard symmetrization technique used to related Rademacher complexity to uniform deviation of empirical averages from their expectation \cite{bartlett2005}. The third inequality comes from a direct application of Lemma \ref{lem:contraction inequality} with $\phi(x)=x^2$ (with Lipschitz constant $4b$ on $[-2b,2b]$), and the last inequality uses Jensen's inequality together with the definition of $\mathfrak{R}(\fcal)$ and the fact that $\frac{1}{N_{t}^{2}} \leq \frac{1}{nN_{t}}$. Plugging the previous inequality on $\Sigma^2$ back into \eqref{bernstein-mtl-1} gives
   \begin{equation}\label{bernstein-mtl-2}
     \log \ebb e^{\lambda (W/b')}\leq \frac{\lambda}{b'(1-\lambda b')}\Big[\frac{16b\mathfrak{R}(\fcal)}{nT}+\Upsilon\Big],\quad\forall \lambda \in (0, 1/b').
   \end{equation}


    \textbf{Step 3. Controlling the term $\Upsilon$}.
    Note that $\Upsilon$ can be upper bounded as
    \begin{equation}\label{bernstein-mtl-3}
    \begin{split}
    \Upsilon  : &= \sup_{\boldf \in \fcal} \Big[  \frac{1}{T^{2}} \sum_{s=1}^{T} \frac{1}{N^{2}_{s}} \sum_{j=1}^{N_{s}} \ebb[\ebb f_{s}(X_{s}^{j}) - f_{s}(X_{s}^{j})]^{2} \Big]
    \\
    & \leq \frac{1}{nT^{2}} \sup_{\boldf \in \fcal} \Big[  \sum_{s=1}^{T}  \ebb[\ebb f_{s}(X_{s}^{1}) - f_{s}(X_{s}^{1})]^{2} \Big]
    \\
    & \leq \frac{1}{nT^{2}} \sup_{\boldf \in \fcal} \Big[  \sum_{s=1}^{T}  \ebb[f_{s}(X_{s}^{1})]^{2} \Big]
    \\
    & \leq \frac{r}{nT}.
    \end{split}
    \end{equation}
    where the last inequality follows from the assumption $\frac{1}{T} \sup_{\boldf \in \fcal} \Big[  \sum_{s=1}^{T}  \ebb[f_{s}(X_{s}^{1})]^{2} \Big] \leq r$ of the theorem.

    \textbf{Step 4. Transferring from the bound on log-moment generating function of $Z$ to tail probabilities}.
    Plugging the bound on $\log\ebb e^{\lambda W/b'}$ given in \eqref{bernstein-mtl-2} and the bound on $\Upsilon$ given in \eqref{bernstein-mtl-3} back into \eqref{Sobolev} immediately yields the following inequality on the log-moment generating function of $Z$ for any $\lambda\in(0,1/2b')$:
    \begin{equation}\label{bernstein-mtl-4}
    \begin{split}
      \log\ebb[e^{\lambda(Z-\ebb Z)}] & \leq \frac{\lambda b'}{1-\lambda b'}\Big[\frac{\lambda}{b'(1-\lambda b')}\big[16(nT)^{-1}b\mathfrak{R}(\fcal)+\Upsilon]+\frac{\lambda \Upsilon}{b'}\Big]\\
      & \leq \frac{\lambda b'}{1-\lambda b'}\frac{\lambda}{b'(1-\lambda b')}\Big[\frac{16b\mathfrak{R}(\fcal)}{nT}+2 \Upsilon \Big]\\
      & \leq \frac{2\lambda^2}{2(1-2\lambda b')}\Big[\frac{16b\mathfrak{R}(\fcal)}{nT}+\frac{2r}{nT}\Big],
    \end{split}
    \end{equation}
    where the second inequality uses $(1-\lambda b')^2\geq 1-2\lambda b'>0$ since $\lambda\in(0,1/2b')$. That is, the conditions of Lemma \ref{BousquetLemma} hold and we can apply it (with $A=2\big[\frac{16b\mathfrak{R}(\fcal)}{nT}+\frac{2r}{nT}\big]$ and $B=2b'$) to get the following inequality with probability at least $1-e^{-x}$ (note that $b'=\frac{2b}{nT}$):
    \begin{align*}
      Z & \leq \ebb[Z] + \sqrt{4x\Big[\frac{16b\mathfrak{R}(\fcal)}{nT}+\frac{2r}{nT}\Big]} + 2b'x \\
       & \leq \ebb[Z] + 8\sqrt{\frac{bx\mathfrak{R}(\fcal)}{nT}} + \sqrt{\frac{8xr}{nT}} + \frac{4bx}{nT}\\
       & \leq \ebb[Z] + 2\mathfrak{R}(\fcal) + \frac{8bx}{nT} + \sqrt{\frac{8xr}{nT}} + \frac{4bx}{nT}\\
       & \leq 4\mathfrak{R}(\fcal)+\sqrt{\frac{8xr}{nT}}+\frac{12bx}{nT},
    \end{align*}
    where the third inequality follows from $2 \sqrt{uv} \leq u + v$, and the last step uses the following inequality due to the symmetrization technique (here the ghost sample $X'$ is an \iid\: copy of the initial sample $X$)
    \begin{align*}
        \mathbb{E}Z & = \mathbb{E}_{X} \Big[ \sup_{\boldsymbol{f} \in \mathcal{F}} \frac{1}{T} \mathbb{E}_{X'} \Big[ \sum_{t=1}^{T} \frac{1}{N_{t}} \sum_{i=1}^{N_{t}} \big( f_t\big( X'^{i}_t  \big) - f_t \big( X_t^{i}\big)  \big) \Big] \Big]
          \\
         & \leq \mathbb{E}_{X,X'} \Big[ \sup_{\boldsymbol{f} \in \mathcal{F}} \frac{1}{T}  \sum_{t=1}^{T} \frac{1}{N_{t}} \sum_{i=1}^{N_{t}} \big( f_t\big( X'^{i}_t \big) - f_t \big( X_t^{i}\big)  \big) \Big]
         \\
         &= \mathbb{E}_{X,X',\sigma} \Big[ \sup_{\boldsymbol{f} \in \mathcal{F}} \frac{1}{T}  \sum_{t=1}^{T}\frac{1}{N_{t}} \sum_{i=1}^{N_{t}} \sigma_{t}^{i} \big( f_t\big( X'^{i}_t  \big) - f_t \big( X_t^{i}\big)  \big) \Big]\\
         & \leq 2 \mathfrak{R}(\mathcal{F}).
    \end{align*}
    Note that the second identity holds since for any $\sigma_t^i$, the random variable $f_t(X'^{i}_t)-f_t(X_t^i)$ has the same distribution as $\sigma_t^i(f_t(X'^{i}_t)-f_t(X_t^i))$.

\section{Proofs of the results in \sref{sec:GeneralizationBounds}: ``Excess \ac{MTL} Risk Bounds based on Local Rademacher Complexities"}
\label{B}
\thrmref{MainTheoremClassF} is at the core of proving \thrmref{MainTheorem} in \sref{sec:GeneralizationBounds}.
We first present the following lemma which is used in the proof of Theorem \ref{MainTheoremClassF}.
 
\begin{lemm}\label{lem:uniform-deviation-transformation}
  Let $K>1, r>0$. Assume that $\mathcal{F}=\{\boldsymbol{f} : =(f_1,\ldots,f_T) : \forall t, f_{t} \in \rbb^{\xcal}\}$ is a vector-valued $(1, B)$-Bernstein class of functions. Also, let the rescaled version of $\fcal$ be defined as
  $$
  \mathcal{F}_{r}:= \left\lbrace \boldsymbol{f}' = \big(f'_{1}, \ldots, f'_{T}\big) : f'_{t}:=\frac{rf_{t}}{\max\left( r,V (\boldsymbol{f}) \right)},\boldf =(f_{t},\ldots,f_{T}) \in \mathcal{F} \right\rbrace.
  $$ 
  If $V^{+}_{r} := \sup_{\boldsymbol{f}' \in \mathcal{F}_{r}} [P \boldsymbol{f}' - P_n \boldsymbol{f}']\leq \frac{r}{BK}$, then
  \begin{equation}
  \label{uniform-deviation-transformation}
  \forall \boldf \in \fcal \qquad  P  \boldsymbol{f}   \leq \frac{K}{K-\beta} P_{n}  \boldsymbol{f} + \frac{r}{BK}.
  \end{equation}
\end{lemm}
\begin{proof}
We prove \eqref{uniform-deviation-transformation} by considering two cases. Let $\boldsymbol{f}$ be any element in $\fcal$. 
If $V(\boldsymbol{f}) \leq r$, then $\boldsymbol{f}'= \boldsymbol{f}$. Therefore, considering the fact that for any $\boldf' \in \fcal_{r}$ it holds that $P\boldf' \leq P_{n}\boldf' + V_{r}^{+}$, the inequality $V_r^+\leq\frac{r}{BK}$ translates to 
\begin{align}
P \boldsymbol{f}  \leq P_{n} \boldsymbol{f}  + \frac{r}{BK} \leq \frac{K}{K-1} P_n \boldf + \frac{r}{BK}.
\label{PeelingConstraintCase1}
\end{align}
If $ V(\boldsymbol{f})  \geq r$, then $\boldsymbol{f}' = r \boldsymbol{f} / V(\boldsymbol{f})$. Therefore, $P\boldf' \leq P_{n}\boldf' + V_{r}^{+}$ together with $V_r^+\leq \frac{r}{BK}$ gives
$$
 \frac{r}{  V(\boldsymbol{f}) }  P \boldsymbol{f}  \leq  \frac{r}{ V(\boldsymbol{f})} P_{n}  \boldsymbol{f}   +  \frac{r}{BK},
$$
which, coupled with $ V(\boldsymbol{f})\leq B P  \boldsymbol{f}$, yields

\begin{align}
P \boldf \leq P_{n} \boldf + \frac{1}{K} P \boldf.
\nonumber
\end{align}
This last inequality then implies
\begin{align}
\label{PleelingConstraintCase2}
P \boldf \leq \frac{K}{K-1} P_{n} \boldf \leq \frac{K}{K-1} P_{n} + \frac{r}{BK}.
\end{align}
Eq. \eqref{uniform-deviation-transformation} follows by combining \eqref{PeelingConstraintCase1} and \eqref{PleelingConstraintCase2} together.
\end{proof}

The following provides another useful definition that will be needed in introducing the result of \thrmref{MainTheoremClassF}.
 \begin{defin}[Star-Hull]
 \label{StarHullDefin}
 The star-hull of a function class $\fcal$ around the function $f_{0}$ is defined as
 \begin{align}
star(\fcal,f_{0}) : = \{ f_{0} + \alpha (f - f_{0}): f \in \fcal, \alpha \in [0,1] \}.
 \nonumber
 \end{align} 
 \end{defin}
 Now, we present a lemma from \cite{bartlett2005} which indicates that the local Rademacher complexity of the star-hull of any function class $\fcal$ is a sub-root function, and it has a unique fixed point.
 \begin{lemm}[Lemma 3.4 in \cite{bartlett2005}]
 \label{LRCIsStarHull}
 For any function class $\fcal$, the local Rademacher complexity of its start-hull is a sub-root function.
 \end{lemm}
\begin{thrm}[Distribution-dependent bound for \ac {MTL}]
\label{MainTheoremClassF}
Let $\mathcal{F}=\{\boldsymbol{f} : =(f_1,\ldots,f_T) : \forall t, f_{t} \in \rbb^{\xcal}\}$ be a class of vector-valued functions satisfying $\sup_{t,x}|f_t(x)|\leq b$. Let $X : =(X^{i}_{t}, Y_{t}^{i})_{(t,i)=(1,1)}^{(T,n)}$ be a vector of $nT$ independent random variables where $(X_{t}^{1},Y_{t}^{1}),\ldots, (X_{t}^{n},Y_{t}^{n}),\forall t\in\nbb_T$ are identically distributed. Assume that $\fcal$ is a $(\beta, B)$-Bernstein class of vector-valued functions. Let $\psi$ be a sub-root function with the fixed point $r^{*}$. Suppose that
$$
B \mathfrak{R}(\mathcal{F}, r) \leq \psi(r),\quad\forall r \geq r^{*},
$$
where $\mathfrak{R}(\mathcal{F}, r):= \mathbb{E} \Big[ \sup_{\boldsymbol{f} \in \mathcal{F}, V(\boldsymbol{f}) \leq r} \frac{1}{nT} \sum_{t=1}^{T} \sum_{i=1}^{n} \sigma_{t}^{i} f_t(X_{t}^{i})  \Big]$ is the \ac {LRC} of the function class $\mathcal{F}$.
Then,
\begin{enumerate}
 \item For any  $K > 1$, and $x > 0$, with probability at least $1-e^{-x}$, every $\boldf \in \fcal$ satisfies
 \begin{align}
\label{eq:MainTheoremClassF-A}
P \boldf \leq \frac{K}{K-1} P_{n} \boldf + \frac{560K}{B} r^{*} + \frac{(24b + 28BK) x}{nT}.
 \end{align}
\item If $\mathcal{F}$ is \underline{convex}, 
then for any $K > 1$, and $x > 0$, the following inequality holds with probability at least $1-e^{-x}$ for every $\boldf \in \fcal$
\begin{align}
\label{eq:MainTheoremClassF-B}
P \boldf \leq \frac{K}{K-1} P_{n} \boldf + \frac{32K}{B} r^{*} + \frac{(24b + 16BK) x}{nT}.
\end{align}
 \end{enumerate}
\end{thrm}
\begin{proof}
\label{sec:ProofTheorem2}
Similar to \lemmref{lem:uniform-deviation-transformation}, define for the vector-valued function class $\fcal$, 
$$
\mathcal{F}_{r}:= \left\lbrace \boldsymbol{f}' = \big(f'_{1}, \ldots, f'_{T}\big) : f'_{t}:=\frac{rf_{t}}{\max\left( r,V (\boldsymbol{f}) \right)},\boldf =(f_{t},\ldots,f_{T}) \in \mathcal{F} \right\rbrace.
$$ 
The proof can be broken down in two steps. The first step applies Theorem \ref{TalagrandforMTL} and the seminal peeling technique \citep{van1987new,van1996weak} to establish an inequality on the uniform deviation over the function class $\mathcal{F}_r$. The second step then uses the Bernstein assumption $V(\boldsymbol{f}) \leq B P \boldsymbol{f}$ to convert this inequality stated for $\mathcal{F}_r$ to a uniform deviation inequality for $\mathcal{F}$.

\textbf{Step 1. Controlling uniform deviations for $\mathcal{F}_r$}. To apply Theorem \ref{TalagrandforMTL} to $\mathcal{F}_r$, we need to control the variances and uniform bounds for elements in $\fcal_r$. We first show $P \boldsymbol{f}'^{2} \leq r,\forall \boldsymbol{f}' \in \mathcal{F}_{r}$. Indeed, for any $\boldsymbol{f} \in \mathcal{F}$ with $V \left( \boldsymbol{f}\right) \leq r$, the definition of $\mathcal{F}_{r}$ implies $f'_{t}=f_{t}$ and, hence, $P \boldsymbol{f}'^{2} = P \boldsymbol{f}^{2} \leq V(\boldsymbol{f}) \leq r$. Otherwise, if $V (\boldsymbol{f}) \geq r$, then $f'_{t}= r f_{t}/ V(\boldsymbol{f})$ and we get
$$
 P \boldsymbol{f}'^{2}  =   \frac{1}{T} \sum_{t=1}^{T} P f'^{2}_{t}   =   \frac{r^{2}}{\big[ V(\boldsymbol{f}) \big] ^{2}} \Big( \frac{1}{T} \sum_{t=1}^{T} P f_{t}^{2} \Big)  \leq  \frac{r^{2}}{\big[ V(\boldsymbol{f}) \big]^2}  V(\boldsymbol{f})  \leq    r.
$$
Therefore, $\frac{1}{T}\sup_{\boldf' \in \fcal_{r}} \sum_{t=1}^{T} \ebb [f'_{t}(X_{t})]^{2} \leq r$. Also, since functions in $\mathcal{F}$ admit a range of $[-b,b]$ and since $0 \leq r/\max(r, V(\boldsymbol{f}) ) \leq 1$, the inequality $\sup_{t,x} |f'_{t}(x)| \leq b$ holds for any $\boldsymbol{f}' \in \mathcal{F}_{r}$.
Applying \thrmref{TalagrandforMTL} to the function class $\mathcal{F}_{r}$ then yields the following inequality with probability at least $1-e^{-x},\forall x>0$
\begin{equation}
\label{TalagrandForPeeledClass}
\sup_{\boldsymbol{f}' \in \mathcal{F}_{r}} [P \boldsymbol{f}' - P_n \boldsymbol{f}'] \leq 4 \mathfrak{R}(\mathcal{F}_{r}) + \sqrt{\frac{8xr}{nT}}+\frac{12bx}{nT}.
\end{equation}
It remains to control the Rademacher complexity of $\fcal_r$. Denote $\mathcal{F} (u,v) := \big\{ \boldsymbol{f} \in \fcal :  u  \leq V(\boldsymbol{f}) \leq v \big\},\forall 0\leq u\leq v$, and introduce the notation
$$
\mathfrak{R}_{n} \boldsymbol{f}':=\frac{1}{nT} \sum_{t=1}^{T} \sum_{i=1}^{n} \sigma_{t}^{i} f'_{t} (X_t^{i}), \qquad \qquad
\mathfrak{R}_{n}(\mathcal{F}_{r}) := \sup_{\boldsymbol{f}' \in \mathcal{F}_{r}} \Big[ \mathfrak{R}_{n} \boldsymbol{f}' \Big].
$$
Note that  $\mathfrak{R}(\mathcal{F}_{r}) = \mathbb{E}\mathfrak{R}_{n}(\mathcal{F}_{r})$. 
Our assumption implies $ V(\boldsymbol{f}) \leq B P \boldsymbol{f} \leq Bb,\forall \boldsymbol{f} \in \mathcal{F}$. Fix $\lambda > 1$ and define $k$ to be the smallest integer such that $r \lambda^{k+1} \geq Bb$.
Then, it follows from the union bound inequality
\begin{align}
\label{UnionBound}
\mathfrak{R}(\mathcal{G}_1\cup\mathcal{G}_2)\leq \mathfrak{R}(\mathcal{G}_1)+\mathfrak{R}(\mathcal{G}_2)
\end{align}
that
\begin{align*}
\mathfrak{R}(\mathcal{F}_{r}) & = \mathbb{E} \bigg[ \sup_{  \boldsymbol{f}' \in \mathcal{F}_{r}} \mathfrak{R}_{n}  \boldsymbol{f}' \bigg]
          = \mathbb{E} \bigg[ \sup_{  \boldsymbol{f} \in \mathcal{F} } \frac{1}{nT} \sum_{t=1}^{T} \sum_{i=1}^{n}  \frac{r}{\max(r,  V(\boldsymbol{f}) )} \sigma_{t}^{i} f_{t} (X_t^{i})  \bigg]\\
          & \stackrel{\eqref{UnionBound}}{\leq} \mathbb{E} \bigg[ \sup_{ \boldsymbol{f} \in \mathcal{F}(0,r)} \frac{1}{nT} \sum_{t=1}^{T} \sum_{i=1}^{n}  \sigma_{t}^{i} f_{t} (X_t^{i}) \bigg]  + \mathbb{E} \bigg[ \sup_{ \boldsymbol{f} \in \mathcal{F}(r,Bb)}  \frac{1}{nT} \sum_{t=1}^{T} \sum_{i=1}^{n}  \frac{r}{ V(\boldsymbol{f}) } \sigma_{t}^{i} f_{t} (X_t^{i})\bigg]\\
          & \stackrel{\eqref{UnionBound}}{\leq} \mathbb{E} \bigg[ \sup_{ \boldsymbol{f} \in \mathcal{F}(0,r)} \frac{1}{nT} \sum_{t=1}^{T} \sum_{i=1}^{n}  \sigma_{t}^{i} f_{t} (X_t^{i}) \bigg] + \sum_{j=0}^{k} \lambda^{-j} \mathbb{E} \bigg[ \sup_{\boldsymbol{f} \in \mathcal{F} (r \lambda^{j},r \lambda^{j+1})}  \mathfrak{R}_{n}  \boldsymbol{f}\bigg]\\
          & \leq \mathfrak{R} (\mathcal{F},r)+ \sum_{j=0}^{k} \lambda^{-j} \mathfrak{R} \big(\mathcal{F} ,r \lambda^{j+1}\big)\\
          & \leq   \frac{\psi(r)}{B}+\frac{1}{B} \sum_{j=0}^{k} \lambda^{-j} \psi(r \lambda^{j+1}).
\end{align*}
The sub-root property of $\psi$ implies that for any $\xi \geq 1$, $\psi(\xi r)\leq \xi^{\frac{1}{2}}\psi(r)$, and hence
$$
\mathfrak{R}(\mathcal{F}_{r})  \leq  \frac{\psi(r)}{B} \bigg( 1+ \sqrt{\lambda} \sum_{j=0}^{k} \lambda^{-j/2}\bigg)
 \leq \frac{\psi(r)}{B} \Big(  1 + \frac{\lambda}{\sqrt{\lambda}  -1 } \Big).
$$
Taking the choice $\lambda=4$ in the above inequality implies that $\mathfrak{R}(\mathcal{F}_{r})\leq 5 \psi(r) / B$, which, together with the inequality
$\psi(r) \leq \sqrt{r/r^{*}} \psi(r^{*}) = \sqrt{rr^{*}},\forall r \geq r^{*}$, gives
$$
\mathfrak{R}(\mathcal{F}_{r}) \leq \frac{5}{B} \sqrt{rr^{*}},\quad\forall r\geq r^*.
$$
Combining \eqref{TalagrandForPeeledClass} and the above inequality together, for any $r \geq r^{*}$ and $x > 0$, we derive the following inequality with probability at least $1 - e^{-x}$,
\begin{equation}\label{boundITOr}
\sup_{\boldsymbol{f}' \in \mathcal{F}_{r}} [P \boldsymbol{f}' - P_n \boldsymbol{f}'] \leq \frac{20}{B} \sqrt{rr^{*}} + \sqrt{\frac{8xr}{nT}}+\frac{12bx}{nT}.
\end{equation}

\medskip
\textbf{Step 2. Transferring uniform deviations for $\fcal_r$ to uniform deviations for $\fcal$.} 
Setting $A = 20 \sqrt{r^{*}} / B + \sqrt{8x/nT}$ and $C = 12bx/nT$, the upper bound \eqref{boundITOr} can be written as $A \sqrt{r} + C$, that is, $\sup_{\boldsymbol{f}' \in \mathcal{F}_{r}} [P \boldsymbol{f}' - P_n \boldsymbol{f}'] \leq A \sqrt{r} + C$. Now, according to \lemmref{lem:uniform-deviation-transformation}, if $\sup_{\boldsymbol{f}' \in \mathcal{F}_{r}} [P \boldsymbol{f}' - P_n \boldsymbol{f}'] \leq \frac{r}{BK}$, then for any $\boldf \in \fcal$,
\begin{equation}
\label{LemmaIneq}
    P  \boldsymbol{f}   \leq \frac{K}{K-1} P_{n}  \boldsymbol{f} + \frac{r}{BK}.
 \end{equation}

Therefore, in order to use the result of \lemmref{lem:uniform-deviation-transformation}, we let $A \sqrt{r} + C = r/(BK)$. Assume $r_0$ is the unique positive solution of the equation $A\sqrt{r}+C=r/(BK)$. It follows immediately that
\begin{align}
r^{*} \leq (ABK)^{2} \leq r_{0} \leq (ABK)^{2} + 2BKC.
\nonumber
\end{align}
\eqref{boundITOr} then shows $\sup_{\boldf' \in \fcal_{r}} [P \boldsymbol{f}' - P_n \boldsymbol{f}'] \leq \frac{r_{0}}{BK}$, and together with \eqref{LemmaIneq} implies

\begin{align*}
P  \boldsymbol{f}  & \leq \frac{K}{K-1} P_n  \boldsymbol{f} +   \frac{r_0}{B K}\\
    & \leq \frac{K}{K-1} P_n  \boldsymbol{f}  + B K \bigg[ \frac{400}{B^{2}} r^{*} + \frac{40}{B} \sqrt{\frac{8 x r^{*}}{nT}} + \frac{8 x}{nT}\bigg]+ \frac{24bx}{n T}.
\end{align*}

The stated inequality \eqref{eq:MainTheoremClassF-A} follows immediately from $\sqrt{8 x r^{*} / nT} \leq B x/ (2nT) + 4r^{*}/B$.

The proof of the \underline{second part} follows from the fact that $\mathcal{F}_{r} \subseteq \left\lbrace \boldsymbol{f} \in star (\mathcal{F},0) : V(\boldsymbol{f}) \leq r\right\rbrace $,
where $star (\mathcal{F},f_{0})$ is defined according to \defref{StarHullDefin}. Also, since any convex class $\mathcal{F}$ is star-shaped around any of its points, we have $\mathcal{F}_{r} \subseteq \left\lbrace \boldsymbol{f} \in \mathcal{F} : V(\boldsymbol{f}) \leq r\right\rbrace $. Therefore, $\mathfrak{R}(\mathcal{F}_{r})$ in \eqref{TalagrandForPeeledClass} can be bounded as $\mathfrak{R}(\mathcal{F}_{r}) \leq \mathcal{R}(\mathcal{F},r) \leq \psi(r)/B$. The rest proof of \eqref{eq:MainTheoremClassF-B} is analogous to that of the first part and is omitted for brevity. 
\end{proof}
\section*{Proof of \thrmref{MainTheorem}}
Note that the proof of this theorem relies on the results of \thrmref{MainTheoremClassF}. Introduce the following class of excess loss functions
\begin{align}
\label{ExcessRiskLoss}
{\mathcal{H}}^{*}_{\mathcal{F}}:=\left\lbrace
 h_{\boldsymbol{f}}= ( h_{f_{1}}, \ldots, h_{f_{T}}), h_{f_{t}} : \left( X_t,Y_t \right) \mapsto  \ell (f_{t}(X_{t}),Y_t) - \ell (f_{t}^{*}(X_{t}),Y_t)  , \boldsymbol{f} \in \mathcal{F}
 \right\rbrace.
\end{align}
It can be shown that $\sup_{t,x} |h_{f_{t}}(x,y)| = \sup_{t,x} |\ell (f_{t}(x),y) -\ell (f^{*}_{t}(x),y)| \leq L \sup_{t,x} |f_{t}(x) - f^{*}_{t}(x)| \leq 2Lb$. Also, \assumref{assumption} implies
$$
  P(\ell_{\boldsymbol{f}} - \ell_{\boldsymbol{f}^{*}})^{2} \leq L^{2} P(\boldsymbol{f} - \boldsymbol{f}^{*})^{2}  \leq B'L^{2} P (\ell_{\boldsymbol{f}} - \ell_{\boldsymbol{f}^{*}}),\quad\forall h_{\boldf} \in \mathcal{H}^{*}_{\mathcal{F}}.
$$
By taking $B = B'L^{2}$, we have for all $h_{\boldf} \in \mathcal{H}^{*}_{\mathcal{F}}$,
$$
  V (h_{\boldsymbol{f}}):= P h_{\boldsymbol{f}}^2\leq L^{2} P (\boldsymbol{f} - \boldsymbol{f}^{*})^{2} \leq B P(\ell_{\boldsymbol{f}} - \ell_{\boldsymbol{f}^{*}}) = B P h_{\boldf},
$$
which implies that ${\mathcal{H}}^{*}_{\mathcal{F}}$ is a $(1,B)$-Bernstein class of vector-valued functions. Also, note that one can verify
\begin{align}
B \mathfrak{R}(\mathcal{H}^{*}_{\mathcal{F}} ,r) & = B \mathbb{E}_{X, \sigma} \left[ \sup_{\substack{\boldsymbol{f} \in \mathcal{F},\\ V (h_{\boldsymbol{f}}) \leq r}} \frac{1}{nT} \sum_{t=1}^{T} \sum_{i=1}^{n} \sigma_{t}^{i} h_{f_{t}} (X_{t}^{i}, Y_{t}^{i})  \right]
\nonumber\\
&= B \mathbb{E}_{X, \sigma} \left[ \sup_{\substack{\boldsymbol{f} \in \mathcal{F},\\ V (h_{\boldsymbol{f}}) \leq r}} \frac{1}{nT} \sum_{t=1}^{T} \sum_{i=1}^{n} \sigma_{t}^{i} \ell_{f_{t}} (X_{t}^{i}, Y_{t}^{i})  \right]
\nonumber\\
& \leq B L \mathfrak{R}(\mathcal{F}^{*} ,r) \leq  \psi(r),
\nonumber
\end{align}
\noindent
where the second last inequality is due to Talagrand's Lemma \citep{ledoux2013}. Applying \thrmref{MainTheoremClassF} (which is the extension of Theorem 3.3 of \cite{bartlett2005} to \ac {MTL} function classes) to the function class $\mathcal{H}^{*}_{\mathcal{F}}$ completes the proof.

%
The following lemma, as a consequence of Corollary 2.2 in \cite{bartlett2005}, is essential in proving \thrmref{MainTheorem2}.
\begin{lemm}
\label{ConsequenceofCorr2.2Bartlett}
Assume that the functions in vector-valued function class  $\fcal =\{\boldf = (f_{1},\ldots,f_{T})\}$ satisfy $\sup_{t,x}|f_t(x)|\leq b$ with $b>0$. For every $x >0$, if $r$ satisfies
\begin{align}
r \geq 32 L^{2} b \mathbb{E}_{\sigma, X} \left\lbrace \sup_{\substack{\boldsymbol{f} \in \mathcal{F},\\ L^{2} P \left( \boldsymbol{f} - \boldsymbol{f}^{*} \right)^{2} \leq r }} \frac{1}{nT} \sum_{t=1}^T \sum_{i=1}^{n} \sigma_{t}^{i} f_{t}(X_{t}^{i}) \right\rbrace + \frac{128 L^{2} b^{2} x}{nT},
\nonumber
\end{align}
\noindent
then, with probability at least $1-e^{-x}$,
\begin{align}
\left\lbrace \boldsymbol{f} \in \mathcal{F} : L^{2} P \left( \boldsymbol{f}-\boldsymbol{f}^{*}\right)^{2} \leq r  \right\rbrace \subset \left\lbrace \boldsymbol{f} \in \mathcal{F} : L^{2} P_{n} \left( \boldsymbol{f}-\boldsymbol{f}^{*}\right)^{2} \leq 2r \right\rbrace.
\nonumber
\end{align}
\end{lemm}
\begin{proof}
First, define
\begin{equation}
\mathcal{F}_{r}^{*} := \left\lbrace \boldsymbol{f}'= (f'_{1}, \ldots, f'_{T}): \forall t, f'_{t} = (f_{t}-f^{*}_{t})^{2}, \boldf = (f_{1}, \ldots, f_{T}) \in \mathcal{F}, L^{2} P (\boldsymbol{f} - \boldsymbol{f}^{*})^{2} \leq r\right\rbrace.
\nonumber
\end{equation}
Note that for all $t \in \mathbb{N}_{T}$, $(f_{t} - f_{t}^{*})^{2} \in [0,4 b^{2}]$. Also, for any function in $\mathcal{F}^{*}_{r}$, it holds that
\begin{equation}
 P \boldsymbol{f}'^{2} = \frac{1}{T} \sum_{t=1}^{T} P f'^{2}_{t}  = \frac{1}{T} \sum_{t=1}^{T} P \left( f_{t} -f_{t}^{*} \right)^{4}  \leq  \frac{ 4 b^{2}}{T} \sum_{t=1}^{T} P \left( f_{t} -f_{t}^{*} \right)^{2} =4  b^{2} P \left( \boldsymbol{f}-\boldsymbol{f}^{*}\right)^{2} \leq \frac{4 b^{2} r}{L^{2}}.
 \nonumber
\end{equation}
Therefore, by \thrmref{TalagrandforMTL}, with probability at least $1-e^{-x}$, every $\boldsymbol{f}' \in \mathcal{F}^{*}_{r} $ satisfies
\begin{align}
\label{MainTheoremApplication}
      P_{n} \boldsymbol{f}' \leq P \boldsymbol{f}' + 4 \mathfrak{R}
      (\mathcal{F}^{*}_{r}) + \sqrt{\frac{32 b^{2} x r}{nTL^{2}} }+ \frac{48 b^{2}x}{nT},
\end{align}
\noindent
where
\begin{align}
\mathfrak{R}
 (\mathcal{F}^{*}_{r}) & = \mathbb{E}_{\sigma, X} \left\lbrace \sup_{\substack{\boldsymbol{f} \in \mathcal{F},\\ L^{2} P \left( \boldsymbol{f} - \boldsymbol{f}^{*} \right)^{2} \leq r }} \frac{1}{nT} \sum_{t=1}^T \sum_{i=1}^{n} \sigma_{t}^{i} (f_{t}(X_{t}^{i}) - f_{t}^{*}(X_{t}^{i}) )^{2} \right\rbrace
 \nonumber\\
 & \leq 4b \mathbb{E}_{\sigma, X} \left\lbrace \sup_{\substack{\boldsymbol{f} \in \mathcal{F},\\ L^{2} P \left( \boldsymbol{f} - \boldsymbol{f}^{*} \right)^{2} \leq r }} \frac{1}{nT} \sum_{t=1}^T \sum_{i=1}^{n} \sigma_{t}^{i} f_{t}(X_{t}^{i}) \right\rbrace.
\end{align}

\noindent
The last inequality follows from the facts that $g(x) = x^{2}$ is $4b$-Lipschitz on $[-2b,2b]$ and $\boldsymbol{f}^{*}$ is fixed. This together with \eqref{MainTheoremApplication}, gives
\begin{align}
      P_{n} \boldsymbol{f}' & \leq P \boldsymbol{f}' + 16 b  \mathbb{E}_{\sigma, X} \left\lbrace \sup_{\substack{\boldsymbol{f} \in \mathcal{F},\\ L^{2} P \left( \boldsymbol{f} - \boldsymbol{f}^{*} \right)^{2} \leq r }} \frac{1}{nT} \sum_{t=1}^T \sum_{i=1}^{n} \sigma_{t}^{i} f_{t}(X_{t}^{i}) \right\rbrace  + \sqrt{\frac{32 b^{2} x r}{nTL^{2}} }+ \frac{48 b^{2}x}{nT}
      \nonumber\\
      & \leq \frac{r}{L^{2}} + 16 b  \mathbb{E}_{\sigma, X} \left\lbrace \sup_{\substack{\boldsymbol{f} \in \mathcal{F},\\ L^{2} P \left( \boldsymbol{f} - \boldsymbol{f}^{*} \right)^{2} \leq r }} \frac{1}{nT} \sum_{t=1}^T \sum_{i=1}^{n} \sigma_{t}^{i} f_{t}(X_{t}^{i}) \right\rbrace   +   \frac{r}{2L^{2}}  +  \frac{64 b^{2} x}{nT}.
			\nonumber
\end{align}
\noindent
Multiplying both sides by $L^{2}$ completes the proof.
\end{proof}

\section*{Proof of \thrmref{MainTheorem2}}
With  $c_{1} = 2L \max\left( B,16Lb\right)$ and $c_{2} = 128 L^{2} b^{2} +2b c_{1}$, define the function $\psi(r)$ as
\begin{align}
\label{PsiForConvex1}
\psi(r) = \frac{c_{1}}{2} \mathbb{E} \left[ \sup_{\substack{\boldsymbol{f} \in \mathcal{F},\\ L^{2} P (\boldsymbol{f}-\boldsymbol{f}^{*})^{2} \leq r}} \frac{1}{nT} \sum_{t=1}^{T} \sum_{i=1}^{n} \sigma_{t}^{i} f_t(X_{t}^{i})  \right] + \frac{(c_2 - 2b c_1)x}{nT}.
\end{align}
Since $\mathcal{F}$ is convex, it is star-shaped around any of its points, thus using Lemma 3.4 in \cite{bartlett2005} it can be shown that $\psi (r)$ defined in \eqref{PsiForConvex1} is a sub-root function. With the help of \corref{DistDependentcorr} and \assumref{assumption}, we have with probability at least $1-e^{-x}$
\begin{align}
\label{DataDepIneq1}
L^{2} P \left( \boldsymbol{\hat{f}} - \boldsymbol{f}^{*}\right) ^{2} \leq B P \left(\ell_{\boldsymbol{\hat{f}}} - \ell_{\boldsymbol{f}^{*}} \right) \leq  32 K r + \frac{(48Lb+ 16 B K)B x}{n T}.
\end{align}
\noindent
where $B := B' L^{2}$. Denote the right hand side of the last inequality by $s$. Since $s \geq r \geq r^{*}$, then by the property of sub-root functions it holds that $s \geq \psi(s)$ which together with \eqref{PsiForConvex1}, gives
\begin{align*}
s \geq 32 L^{2} b \mathbb{E} \left[ \sup_{\substack{\boldsymbol{f} \in \mathcal{F},\\ L^{2} P (\boldsymbol{f}-\boldsymbol{f}^{*})^{2} \leq s}} \frac{1}{nT} \sum_{t=1}^{T} \sum_{i=1}^{n} \sigma_{t}^{i} f_t(X_{t}^{i})  \right] + \frac{128 L^{2} b^{2} x}{nT}.
\end{align*}
\noindent
Applying \lemmref{ConsequenceofCorr2.2Bartlett}, we have with probability at least $1-e^{-x}$,
\begin{align*}
\left\lbrace \boldsymbol{f} \in \mathcal{F}, L^{2} P \left(\boldsymbol{f} - \boldsymbol{f}^{*} \right)^{2} \leq s  \right\rbrace \subset \left\lbrace \boldsymbol{f} \in \mathcal{F}, L^{2} P_{n} \left(\boldsymbol{f} - \boldsymbol{f}^{*} \right)^{2} \leq 2s \right\rbrace.
\end{align*}
Combining this with \eqref{DataDepIneq1}, gives with probability at least $1-2e^{-x}$,
\begin{align}
L^{2} P_{n} \left(\boldsymbol{\hat{f}} - \boldsymbol{f}^{*} \right)^{2} & \leq 2 \left( 32 K r + \frac{(48 Lb+ 16 B K)B x}{n T} \right)
\nonumber\\
\label{DataDepIneq2}
& \leq 2 \left( 32 K + \frac{(48 Lb+ 16 B K)B}{c_{2}} \right) r = c r.
\end{align}
\noindent
where $c :=2(32 K + (48 Lb + 16 B K)B/c_{2})$ and in the last inequality we used the fact that $r \geq \psi(r) \geq c_{2} x /nT$. Applying the triangle inequality, if \eqref{DataDepIneq2} holds, then for any $\boldsymbol{f} \in \mathcal{F}$, we have
\begin{align}
L^{2} P_{n} \left(\boldsymbol{f} - \boldsymbol{\hat{f}} \right)^{2} & \leq \left( \sqrt{ L^{2} P_{n} \left(\boldsymbol{f} - \boldsymbol{f}^{*} \right)^{2}} + \sqrt{ L^{2} P_{n} \left(\boldsymbol{f}^{*} - \boldsymbol{\hat{f}} \right)^{2}} \right)^{2}
\nonumber\\
\label{DataDepIneq3}
& \leq \left( \sqrt{ L^{2} P_{n} \left(\boldsymbol{f} - \boldsymbol{f}^{*} \right)^{2}} + \sqrt{cr} \right)^{2}.
\end{align}
Now, applying \lemmref{ConsequenceofCorr2.2Bartlett} for $r \geq \psi(r)$, implies that with probability at least $1-3e^{-x}$,
\begin{align*}
\left\lbrace \boldsymbol{f} \in \mathcal{F}, L^{2} P \left(\boldsymbol{f} - \boldsymbol{f}^{*} \right)^{2} \leq r  \right\rbrace \subset \left\lbrace \boldsymbol{f} \in \mathcal{F}, L^{2} P_{n} \left(\boldsymbol{f} - \boldsymbol{f}^{*} \right)^{2} \leq 2r \right\rbrace,
\end{align*}
which coupled with \eqref{DataDepIneq3}, implies that with probability at least $1- 3e^{-x}$,
\begin{align*}
\left\lbrace \boldsymbol{f} \in \mathcal{F}, L^{2} P \left(\boldsymbol{f} - \boldsymbol{f}^{*} \right)^{2} \leq r  \right\rbrace \subset \left\lbrace \boldsymbol{f} \in \mathcal{F}, L^{2} P_{n} \left(\boldsymbol{f} - \boldsymbol{\hat{f}} \right)^{2} \leq \left( \sqrt{2} + \sqrt{c}\right)^{2} r  \right\rbrace.
\end{align*}
Also, with the help of Lemma A.4 in \cite{bartlett2005}, it can be shown that with probability at least $1-e^{-x}$,
\begin{align*}
\mathbb{E} \left[ \sup_{\substack{\boldsymbol{f} \in \mathcal{F},\\ L^{2} P (\boldsymbol{f}-\boldsymbol{f}^{*})^{2} \leq r}} \frac{1}{nT} \sum_{t=1}^{T} \sum_{i=1}^{n} \sigma_{t}^{i} f_t(X_{t}^{i})  \right] \leq 2 \mathbb{E}_{\sigma} \left[ \sup_{\substack{\boldsymbol{f} \in \mathcal{F},\\ L^{2} P (\boldsymbol{f}-\boldsymbol{f}^{*})^{2} \leq r}} \frac{1}{nT} \sum_{t=1}^{T} \sum_{i=1}^{n} \sigma_{t}^{i} f_t(X_{t}^{i})  \right] + \frac{4bx}{nT}. 
\end{align*}
Thus, we will have with probability at least $1-4e^{-x}$,
\begin{align*}
\psi(r) & \leq c_{1} \mathbb{E}_{\sigma} \left[ \sup_{\substack{\boldsymbol{f} \in \mathcal{F},\\ L^{2} P (\boldsymbol{f}-\boldsymbol{f}^{*})^{2} \leq r}} \frac{1}{nT} \sum_{t=1}^{T} \sum_{i=1}^{n} \sigma_{t}^{i} f_t(X_{t}^{i})  \Biggr \vert \left\{ x_t^i \right\}_{t \in \mathbb{N}_T, i \in \mathbb{N}_n} \right] + \frac{c_{2} x}{nT}
\\
& \leq c_{1} \mathbb{E}_{\sigma} \left[ \sup_{\substack{\boldsymbol{f} \in \mathcal{F},\\ L^{2} P_{n} (\boldsymbol{f}-\boldsymbol{\hat{f}} )^{2} \leq \left( \sqrt{2} + \sqrt{c}\right)^{2} r}} \frac{1}{nT} \sum_{t=1}^{T} \sum_{i=1}^{n} \sigma_{t}^{i} f_t(X_{t}^{i})  \Biggr \vert \left\{ x_t^i \right\}_{t \in \mathbb{N}_T, i \in \mathbb{N}_n} \right] + \frac{c_{2} x}{nT}
\\
& \leq c_{1} \mathbb{E}_{\sigma} \left[ \sup_{\substack{\boldsymbol{f} \in \mathcal{F},\\ L^{2} P_{n} (\boldsymbol{f}-\boldsymbol{\hat{f}})^{2} \leq \left( 4 + 2c\right) r}} \frac{1}{nT} \sum_{t=1}^{T} \sum_{i=1}^{n} \sigma_{t}^{i} f_t(X_{t}^{i})  \Biggr \vert \left\{ x_t^i \right\}_{t \in \mathbb{N}_T, i \in \mathbb{N}_n} \right] + \frac{c_{2} x}{nT}
\\
& \leq \hat{\psi}(r).
\end{align*}
Setting $r = r^{*}$ and applying Lemma 4.3 of \cite{bartlett2005}, gives $r^{*} \leq \hat{r}^{*}$ which together with \eqref{DataDepIneq1} yields the result.


\section{Proofs of the results in \sref{sec:StrongConvexity}: ``Local Rademacher Complexity Bounds for \ac{MTL} models with Strongly Convex Regularizers"}
\label{C}
In the following, we would like to provide some basic notions of convex analysis which are helpful in understanding the results of \sref{sec:StrongConvexity}.

\begin{defin}[\textsc{Strong Convexity}]
\label{StrongConvexity}
A function $R : \mathcal{X} \mapsto \mathbb{R}$ is $\mu$-strong convex \wrt\, a norm $\|.\|$ if and only if $\forall x,y \in \mathcal{X}$ and $\forall \alpha \in (0,1)$, we have
\begin{align}
R(\alpha x + (1- \alpha) y) \leq \alpha R(x) + (1- \alpha) R(y) - \frac{\mu}{2} \alpha (1- \alpha) \|x-y\|^{2}.
\nonumber
\end{align}
\end{defin}
\begin{defin}[\textsc{Strong Smoothness}]
\label{StrongSmoothness}
A function $R^{*} : \mathcal{X} \mapsto \mathbb{R}$ is $\frac{1}{\mu}$-strong smooth \wrt\, a norm $\|.\|_{*}$ if and only if $R^{*}$ is everywhere differentiable and $\forall x,y \in \mathcal{X}$, we have
\begin{align}
R^{*}( x + y) \leq  R^{*}(x) + \left\langle \triangledown R^{*}(x) ,y \right\rangle  + \frac{1}{2 \mu}  \left\| y \right\|_{*}^{2}.
\nonumber
\end{align}
\end{defin}
\begin{Property}[Theorem 3 in \cite{kakade2012}: strong convexity/strong smoothness duality]
\label{DualityofStrongConvexity}
A function $R$ is $\mu$-strongly convex \wrt\, the norm $\left\| . \right\| $ if and only if its Fenchel conjugate $R^{*}$ is $\frac{1}{\mu}$-strongly smooth \wrt\, the dual norm $\left\| . \right\|_{*}$. The Fenchel conjugate $R^{*}$ is defined as
\begin{align}
R^{*}(\boldsymbol{w}) := \sup_{\boldsymbol{v}} \left\lbrace \left\langle \boldsymbol{w}, \boldsymbol{v}\right\rangle - R(\boldsymbol{v})\right\rbrace. 
\nonumber
\end{align}
\end{Property}
\begin{Property}[\textsc{Fenchel-Young inequality}]
\label{FenchelYoung}
The definition of Fenchel dual implies that for any strong convex function $R$, 
\begin{align}
\forall \boldsymbol{w}, \boldsymbol{v} \in S, \, \left\langle \boldsymbol{w}, \boldsymbol{v} \right\rangle \leq R(\boldsymbol{w}) + R^{*}(\boldsymbol{v}).
\nonumber
\end{align}
\noindent
Combining this with the strong duality property of $R^{*}$ gives the following
\begin{align}
\label{FenchelYoungInq}
\left\langle \boldsymbol{w}, \boldsymbol{v} \right\rangle - R(\boldsymbol{w}) \leq  R^{*}(\boldsymbol{v}) \leq R^{*}(\boldsymbol{0}) + \left\langle \triangledown    R^{*} (\boldsymbol{0}) , \boldsymbol{v}  \right\rangle + \frac{1}{2 \mu} \left\| \boldsymbol{v} \right\|_{*}^{2}. 
\end{align}
\end{Property}

\begin{lemm}
\label{MTLemma}
 Assume that the conditions of \thrmref{GeneralLRCBoundForFixedMappingStrongConvex} hold. Then,  for ever $\boldsymbol{f} \in \mathcal{F}_{q}$,\\

(\textbf{a}) 
$P\boldsymbol{f}^{2} \leq r $ implies $1/T \sum_{t=1}^{T} \sum_{j=1}^{\infty} \lambda_t^{j} \left\langle  \boldsymbol{w}_{t}, \boldsymbol{u}_{t}^{j} \right\rangle^{2} \leq r$.

(\textbf{b}) $\mathbb{E}_{X, \sigma} \left\langle \frac{1}{n} \sum_{i=1}^{n} \sigma_t^{i} \phi (X_t^{i}) , \boldsymbol{u}_{t}^{j} \right\rangle^{2} = \frac{\lambda_t^{j}}{n}$. 

\end{lemm}
\begin{proof} \leavevmode

Part (\textbf{a})
\begin{align}
P\boldsymbol{f}^{2} = & \frac{1}{T} \sum_{t=1}^{T} \mathbb{E} \left( \left\langle \boldsymbol{w}_{t}, \phi(X_t^{i})\right\rangle  \right) ^{2} 
 \frac{1}{T} \sum_{t=1}^{T} \mathbb{E} \left( \left\langle \boldsymbol{w}_{t} \otimes \boldsymbol{w}_{t}, \phi(X_t^{i}) \otimes \phi(X_t^{i}) \right\rangle  \right) 
\nonumber\\
= & \frac{1}{T} \sum_{t=1}^{T} \left\langle \boldsymbol{w}_{t} \otimes \boldsymbol{w}_{t},  \mathbb{E}_{X} \left(  \phi(X_t^{i}) \otimes \phi(X_t^{i}) \right)  \right\rangle
= \frac{1}{T} \sum_{t=1}^{T} \sum_{j=1}^{\infty} \lambda_t^{j}  \left\langle \boldsymbol{w}_{t} \otimes \boldsymbol{w}_{t},   \boldsymbol{u}_{t}^{j} \otimes \boldsymbol{u}_{t}^{j}  \right\rangle
\nonumber\\
= & \frac{1}{T} \sum_{t=1}^{T} \sum_{j=1}^{\infty} \lambda_t^{j} \left\langle \boldsymbol{w}_{t} , \boldsymbol{u}_{t}^{j} \right\rangle  \left\langle \boldsymbol{w}_{t} , \boldsymbol{u}_{t}^{j} \right\rangle  = \frac{1}{T} \sum_{t=1}^{T} \sum_{j=1}^{\infty} \lambda_t^{j} \left\langle \boldsymbol{w}_{t} , \boldsymbol{u}_{t}^{j} \right\rangle^{2} \leq r.
\nonumber
\end{align}

Part (\textbf{b})
\begin{align}
 & \mathbb{E}_{X, \sigma} \left\langle \frac{1}{n} \sum_{i=1}^{n} \sigma_t^{i} \phi (X_t^{i}) , \boldsymbol{u}_{t}^{j} \right\rangle^{2}   =   \frac{1}{n^2} \mathbb{E}_{X,\sigma} \sum_{i,k=1}^{n} \sigma_t^{i} \sigma_{t}^{k} \left\langle \phi(X_t^{i}) , \boldsymbol{u}_{t}^{j}\right\rangle \left\langle \phi(X_t^{k}) , \boldsymbol{u}_{t}^{j}\right\rangle  
 \nonumber\\
 & \stackrel{\boldsymbol{\sigma}_t \iid}{=}   \frac{1}{n^2} \mathbb{E}_{X} \left( \sum_{i=1}^{n} \left\langle  \phi(X_t^{i}), \boldsymbol{u}_{t}^{j} \right\rangle ^{2} \right) 
 =  \frac{1}{n} \left\langle \frac{1}{n} \sum_{i=1}^{n} \mathbb{E}_{X} \left(  \phi(X_t^{i}) \otimes \phi(X_t^{i}) \right)  , \boldsymbol{u}_{t}^{j} \otimes \boldsymbol{u}_{t}^{j} \right\rangle 
\nonumber\\
 & = \frac{1}{n}  \sum_{l=1}^{\infty} \lambda_t^{l} \left\langle  \boldsymbol{u}_{t}^{l} \otimes  \boldsymbol{u}_{t}^{l}  , \boldsymbol{u}_{t}^{j} \otimes \boldsymbol{u}_{t}^{j} \right\rangle  = \frac{\lambda_t^{j}}{n}.  
 \nonumber
\end{align}
\end{proof}

The following lemmas are used in the proof of the \ac {LRC} bound for the $L_{2,q}$-group norm regularized \ac {MTL} in \corref{GroupNormLRCStrongConvexity}.

\begin{lemm}[Khintchine-Kahane Inequality \citep{peshkir1995}]
\label{lemm:KKinequality}
Let $\mathcal{H}$ be an inner-product space with induced norm $\left\| \cdot \right\|_{\mathcal{H}}$, $v_1, \ldots, v_M \in \mathcal{H}$ and $\sigma_1, \ldots, \sigma_n$ i.i.d. Rademacher random variables. Then, for any $p \geq 1$, we have that 
\begin{align}
\label{K.K.Inq}
	\mathbb{E}_{\boldsymbol{\sigma}}{ \left\| \sum_{i=1}^n \sigma_i v_i \right\|_{\mathcal{H}}^p } \leq \left( c \sum_{i=1}^n \left\| v_i \right\|_{\mathcal{H}}^2 \right)^{\frac{p}{2}}.
\end{align}
\noindent
where $c := \max\left\{ 1, p - 1 \right\}$. The inequality also holds for $p$ in place of $c$.

\end{lemm}


\begin{lemm}[Rosenthal-Young Inequality; Lemma 3 of \cite{kloft2012convergence}]
\label{lemm:RYinequality}
Let the independent non-negative random variables $X_1, \ldots, X_n$ satisfy $X_i \leq B < +\infty$ almost surely for all $i=1, \ldots, n$. If $q \geq \frac{1}{2}$, $c_q := (2qe)^q$, then it holds 
\begin{align}
\label{Rose}
	\mathbb{E}{ \left( \frac{1}{n} \sum_{i=1}^n X_i \right)^q } \leq C_q \left[ \left( \frac{B}{n} \right)^q + \left( \frac{1}{n} \sum_{i=1}^n \mathbb{E}X_i \right)^q \right].
\end{align}
\end{lemm}

\section*{Proof of \lemmref{A_2ExcpectationGroupNorm}}

  For the group norm regularizer $\left\| \boldsymbol{W} \right\|_{2,q}$, we can further bound the expectation term in \eqref{A_2BoundStrongConvexity} for $\boldsymbol{D} = \boldsymbol{I}$ as follows
 \begin{align}
 \mathbb{E} & := \mathbb{E}_{X, \sigma} \left\| \left(\sum_{j > h_{t}}  \left\langle \frac{1}{n} \sum_{i=1}^{n} \sigma_{t}^{i} \phi(X_{t}^{i}) , \boldsymbol{u}_{t}^{j} \right\rangle \boldsymbol{u}_{t}^{j} \right)_{t=1}^{T}  \right\|_{2,q^{*}}   
 \nonumber\\
 & =   \mathbb{E}_{X,\sigma}    \left( \sum_{t=1}^{T} \left\| \sum_{j > h_t} \left\langle \frac{1}{n} \sum_{i=1}^{n} \sigma_t^{i} \phi(X_t^{i}) , \boldsymbol{u}_{t}^{j} \right\rangle  \boldsymbol{u}_{t}^{j} \right\| ^{q^{*}} \right) ^{\frac {1}{q^{*}}}
  \nonumber\\
  \label{GroupK.KStrongConvexity}
  & \stackrel{\text{Jensen}}{\leq}  \mathbb{E}_{X}  \left( \sum_{t=1}^{T} \mathbb{E}_{\sigma} \left\| \sum_{j > h_t} \left\langle \frac{1}{n} \sum_{i=1}^{n} \sigma_t^{i} \phi(X_t^{i}) , \boldsymbol{u}_{t}^{j} \right\rangle  \boldsymbol{u}_{t}^{j} \right\| ^{q^{*}} \right) ^{\frac {1}{q^{*}}}
  \nonumber\\
  & \stackrel{\eqref{K.K.Inq}}{\leq}   \mathbb{E}_{X}  \left( \sum_{t=1}^{T} \left( q^{*} \sum_{i=1}^{n} \left\| \sum_{j > h_t} \left\langle \frac{1}{n}  \phi(X_t^{i}) , \boldsymbol{u}_{t}^{j} \right\rangle  \boldsymbol{u}_{t}^{j} \right\|^{2} \right)^\frac{q^{*}}{2} \right) ^{\frac {1}{q^{*}}}
  \nonumber\\
  & =  \sqrt{\frac{q^{*}}{n}} \mathbb{E}_{X} \left( \sum_{t=1}^{T} \left(  \sum_{j > h_t} \frac{1}{n} \sum_{i=1}^{n} \left\langle  \phi(X_t^{i}) , \boldsymbol{u}_{t}^{j} \right\rangle^{2} \right)^\frac{q^{*}}{2} \right) ^{\frac {1}{q^{*}}}
  \nonumber\\
  & \stackrel{\text{Jensen}}{\leq}  \sqrt{\frac{q^{*}}{n}}  \left( \sum_{t=1}^{T} \mathbb{E}_{X} \left(  \sum_{j > h_t} \frac{1}{n} \sum_{i=1}^{n} \left\langle  \phi(X_t^{i}) , \boldsymbol{u}_{t}^{j} \right\rangle^{2} \right)^\frac{q^{*}}{2} \right) ^{\frac {1}{q^{*}}}.
 \end{align}
 \noindent 
 Note that for $ q \leq 2$, it holds that $q^{*}/2 \geq 1$. Therefore, we cannot employ Jensen's inequality to move the expectation operator inside the inner term, and instead we need to apply the Rosenthal-Young (R+Y) inequality (see \lemmref{lemm:RYinequality} in the Appendix), which yields
 \begin{align}
 \label{subaddStrongConvexity}
\mathbb{E} &  \stackrel{\text{R+Y}}{\leq} \sqrt{\frac{q^{*}}{n}}  \left( \sum_{t=1}^{T} \left( e q^{*} \right)^{\frac{q^{*}}{2}}  \left( \left( \frac{\mathcal{K}}{n}  \right)^{\frac{q^{*}}{2}}  + \left( \sum_{j > h_t} \frac{1}{n} \sum_{i=1}^{n} \mathbb{E}_{X}\left\langle  \phi(X_t^{i}) , \boldsymbol{u}_{t}^{j} \right\rangle^{2} \right)^\frac{q^{*}}{2} \right)  \right) ^{\frac {1}{q^{*}}} 
 \nonumber\\
 & = \sqrt{\frac{q^{*}}{n}}  \left( \sum_{t=1}^{T} \left( e q^{*} \right)^{\frac{q^{*}}{2}}  \left( \left( \frac{\mathcal{K}}{n}  \right)^{\frac{q^{*}}{2}}  + \left( \sum_{j > h_t}\lambda_t^{j} \right)^\frac{q^{*}}{2} \right)  \right) ^{\frac {1}{q^{*}}}.
 \end{align}
 \noindent
 The last quantity can be further bounded using the sub-additivity of $\sqrt[q^{*}]{.}$ and $\sqrt{.}$ respectively in $(\dagger \dagger)$ and $(\dagger)$ below,
 \begin{align}
 \label{A_2BoundGroupNormStrongConvexity}
 \mathbb{E} & \stackrel{(\dagger)}{\leq} q^{*} \sqrt{\frac{e}{n}} \left[   \left( T \left( \frac{\mathcal{K}}{n}\right) ^{\frac{q^{*}}{2}} \right)^ {\frac{1}{q^{*}}} + \left( \sum_{t=1}^{T} \left(  \sum_{j> h_{t}} \lambda_t^{j} \right)^{\frac{q^{*}}{2}} \right)^{\frac{1}{q^{*}}} \right]  
 \nonumber\\
 & \stackrel{(\dagger \dagger)}{\leq} q^{*} \sqrt{\frac{e}{n}} \left[   T^{\frac{1}{q^{*}}} \sqrt{ \frac{\mathcal{K} }{n}}   +   \left\|  \left( \sum_{j>h_t} \lambda_t^{j} \right)_{t=1}^{T}   \right\| _{\frac{q^{*}}{2}}^{\frac{1}{2}}  \right] 
 \nonumber\\
 &  =  \frac{\sqrt{\mathcal{K} e} q^{*} T^{\frac{1}{q^{*}}}}{n}  + \sqrt{\frac{ e {q^{*}}^{2} }{n}  \left\|  \left( \sum_{j>h_t} \lambda_t^{j} \right)_{t=1}^{T}   \right\|_{\frac{q^{*}}{2}} }.
 \end{align}
 
 \section*{Proof of \corref{GroupNormLRCStrongConvexity}}
 Substituting the result of \lemmref{A_2ExcpectationGroupNorm} into \eqref{A_2BoundHolder} gives,
 \begin{align}
 \label{A2BoundHolder}
 A_{2} (\mathcal{F}_{q}) \leq \sqrt{\frac{2 e {q^{*}}^{2} R^{2}_{max}}{n T^{2}} \left\|  \left( \sum_{j>h_t} \lambda_t^{j} \right)_{t=1}^{T}   \right\|_{\frac{q^{*}}{2}}} + \frac{\sqrt{2 \mathcal{K} e } R_{max} q^{*} T^{\frac{1}{q^{*}}}}{nT}.
 \end{align}
 Now, combining \eqref{GeneralA1BoundStrongConvexity} and \eqref{A2BoundHolder} provides the bound on $\mathfrak{R}(\mathcal{F}_{q},r)$ as
 \begin{align}
 \label{LRCBoundGroupNormStrongConvexity}
 \mathfrak{R}(\mathcal{F}_{q}& , r) \leq   \sqrt{  \frac{ r \sum_{t=1}^{T} h_t} {nT}  } + \sqrt{\frac{2 e  {q^{*}}^{2} R^{2}_{max} }{nT^{2}}  \left\|  \left( \sum_{j>h_t} \lambda_t^{j} \right)_{t=1}^{T}   \right\|_{\frac{q^{*}}{2}} }  +  \frac{\sqrt{2 \mathcal{K} e } R_{max} q^{*} T^{\frac{1}{q^{*}}}}{nT} 
 \\
 & \stackrel{(\star)}{\leq} \sqrt{ \frac{2}{nT}\left(  r \sum_{t=1}^{T} h_t  +  \frac{2e  {q^{*}}^{2} R^{2}_{max}}{T}  \left\|  \left( \sum_{j>h_t} \lambda_t^{j} \right)_{t=1}^{T}   \right\|_{\frac{q^{*}}{2}} \right)  } + \frac{\sqrt{ 2 \mathcal{K} e } R_{max} q^{*} T^{\frac{1}{q^{*}}}}{nT} 
 \nonumber\\
 & \stackrel{( \star \star )}{\leq} \sqrt{ \frac{2}{nT}\left(\! r T^{1 - \frac{2}{q^{*}}} \left\| \left( h_t \right)_{t=1}^{T} \right\|_{\frac{q^{*}}{2}}\!  +  \frac{2 e  {q^{*}}^{2} R^{2}_{max}}{T}  \left\|  \left( \sum_{j>h_t} \lambda_t^{j} \right)_{t=1}^{T}   \right\|_{\frac{q^{*}}{2}} \! \!  \right)} + \! \frac{\sqrt{2 \mathcal{K} e } R_{max} q^{*} T^{\frac{1}{q^{*}}}}{nT}
 \nonumber\\
 & \stackrel{(\star \star \star)}{\leq} \sqrt{ \frac{4}{nT}   \left\| \left(  r T^{1 - \frac{2}{q^{*}}}  h_t      +  \frac{2 e  {q^{*}}^{2} R^{2}_{max}}{T}  \sum_{j>h_t} \lambda_t^{j} \right)_{t=1}^{T}   \right\|_{\frac{q^{*}}{2}}   } + \frac{\sqrt{ 2 \mathcal{K} e } R_{max} q^{*} T^{\frac{1}{q^{*}}}}{nT}.
 \nonumber
 \end{align}
 \noindent
 where in $(\star)$, $(\star \star)$ and $(\star \star \star) $ we applied following inequalities receptively, according which for all non-negative numbers $\alpha_1$ and $\alpha_2$, and non-negative vectors $ \boldsymbol{a}_1, \boldsymbol{a}_2 \in \mathbb{R}^{T}$  with $0 \leq q  \leq  p  \leq  \infty$ and $s \geq 1$ it holds
 \begin{align}
 & (\star) \sqrt{\alpha_1}+ \sqrt{\alpha_2} \leq \sqrt{2 (\alpha_1 + \alpha_2)}
 \nonumber \\
 & (\star \star)  \quad l_p-to-l_q:   \quad    \left\| \boldsymbol{a}_1 \right\|_q  = \left\langle \boldsymbol{1}, \boldsymbol{a}_1 \right\rangle^{\frac{1}{q}} \stackrel{\text{H\"{o}lder}}{\leq} \left( \left\| \boldsymbol{1} \right\|_{(p/q)^{*}}  \left\| \boldsymbol{a}_1^{q} \right\|_{(p/q)}  \right) ^{\frac{1}{q}} = T^{\frac{1}{q} - \frac{1}{p}}  \left\| \boldsymbol{a}_1 \right\|_p
 \nonumber\\
 &(\star \star \star)  \quad    \left\| \boldsymbol{a}_1 \right\|_s  +  \left\| \boldsymbol{a}_2 \right\|_s  \leq  2^{1 - \frac{1}{s}} \left\| \boldsymbol{a}_1  +  \boldsymbol{a}_2 \right\|_s  \leq   2 \left\| \boldsymbol{a}_1  +  \boldsymbol{a}_2 \right\|_s.
 \nonumber
 \end{align}
 \noindent 
 Since inequality $(\star \star \star)$ holds for all non-negative $h_t$, it follows
 \begin{align}
 \mathfrak{R}(\mathcal{F}_{q},r) & \leq \sqrt{ \frac{4}{nT}   \left\| \left(  \min _{h_t \geq 0}   r T^{1 - \frac{2}{q^{*}}}  h_t      +  \frac{2 e  {q^{*}}^{2} R^{2}_{max}}{T}  \sum_{j>h_t} \lambda_t^{j} \right)_{t=1}^{T}   \right\|_{\frac{q^{*}}{2}}   } + \frac{\sqrt{2 \mathcal{K} e } R_{max} q^{*} T^{\frac{1}{q^{*}}}}{nT}
 \nonumber\\
 & \leq \sqrt{ \frac{4}{nT}   \left\| \left( \sum_{j=1}^{\infty} \min  \left(  r T^{1 - \frac{2}{q^{*}}}     ,    \frac{2 e  {q^{*}}^{2} R^{2}_{max}}{T}  \lambda_t^{j} \right)  \right)_{t=1}^{T}   \right\|_{\frac{q^{*}}{2}}   } + \frac{\sqrt{2 \mathcal{K} e } R_{max} q^{*} T^{\frac{1}{q^{*}}}}{nT}.
 \nonumber
 \end{align}
 
 \section*{Proof of \thrmref{LowerBoundStrongConvexity}}
    \begin{align}
      \mathfrak{R} (\mathcal{F}_{q, R_{max}, T},r) & = \frac{1}{T} \mathbb{E}_{X,\sigma} \left\lbrace  \sup_{\substack{P\boldsymbol{f}^{2} \leq r, \\ \left\| \boldsymbol{W} \right\| _{2,q}^{2} \leq 2 R_{max}^{2}}} \sum_{t=1}^{T}  \left\langle  \boldsymbol{w}_{t}, \frac{1}{n} \sum_{i=1}^{n} \sigma_t^{i} \phi (X_t^{i})\right\rangle    \right\rbrace
    \nonumber\\
      & = \frac{1}{T} \mathbb{E}_{X,\sigma} \left\lbrace  \sup_{\substack{1/T \sum_{t=1}^{T} \mathbb{E} \left\langle \boldsymbol{w}_{t}, \phi (X_t)\right\rangle^{2}  \leq r, \\ \left\| \boldsymbol{W} \right\| _{2,q}^{2} \leq 2 R_{max}^{2}}} \sum_{t=1}^{T}  \left\langle  \boldsymbol{w}_{t}, \frac{1}{n} \sum_{i=1}^{n} \sigma_t^{i} \phi (X_t^{i})\right\rangle    \right\rbrace
    \nonumber\\
      & \geq \frac{1}{T} \mathbb{E}_{X,\sigma} \left\lbrace  \sup_{\substack{\forall t \; \mathbb{E}_{X} \left\langle \boldsymbol{w}_{t}, \phi (X_t)\right\rangle^{2}  \leq r, \\ \left\| \boldsymbol{W} \right\| _{2,q}^{2} \leq 2 R_{max}^{2}, \\ \left\| \boldsymbol{w}_1 \right\|_{2} = \ldots = \left\| \boldsymbol{w}_t \right\| _{2} } } \sum_{t=1}^{T}  \left\langle  \boldsymbol{w}_{t}, \frac{1}{n} \sum_{i=1}^{n} \sigma_t^{i} \phi (X_t^{i})\right\rangle    \right\rbrace
    \nonumber\\
      & = \frac{1}{T} \mathbb{E}_{X,\sigma} \left\lbrace  \sup_{\substack{\forall t \; \mathbb{E}_{X} \left\langle \boldsymbol{w}_{t}, \phi (X_t)\right\rangle^{2}  \leq r, \\ \forall t \; \left\| \boldsymbol{w}_t \right\|_{2}^{2} \leq 2 R_{max}^{2} T^{-\frac{2}{q}} } } \sum_{t=1}^{T}  \left\langle  \boldsymbol{w}_{t}, \frac{1}{n} \sum_{i=1}^{n} \sigma_t^{i} \phi (X_t^{i})\right\rangle    \right\rbrace
    \nonumber\\
      & =\frac{1}{T} \sum_{t=1}^{T} \mathbb{E}_{X,\sigma} \left\lbrace  \sup_{\substack{\forall t \; \mathbb{E}_{X} \left\langle \boldsymbol{w}_{t}, \phi (X_t)\right\rangle^{2}  \leq r, \\ \forall t \; \left\| \boldsymbol{w}_t \right\|_{2}^{2} \leq 2 R_{max}^{2} T^{-\frac{2}{q}} } }   \left\langle  \boldsymbol{w}_{t}, \frac{1}{n} \sum_{i=1}^{n} \sigma_t^{i} \phi (X_t^{i})\right\rangle    \right\rbrace
    \nonumber\\
      & = \mathbb{E}_{X,\sigma} \left\lbrace  \sup_{\substack{ \mathbb{E}_{X} \left\langle \boldsymbol{w}_{1}, \phi (X_1)\right\rangle^{2}  \leq r, \\  \left\| \boldsymbol{w}_{1} \right\|_{2}^{2} \leq 2 R_{max}^{2} T^{-\frac{2}{q}} } }   \left\langle  \boldsymbol{w}_{1}, \frac{1}{n} \sum_{i=1}^{n} \sigma_{1}^{i} \phi (X_{1}^{i})\right\rangle    \right\rbrace
    \nonumber\\
      & = \mathfrak{R} (\mathcal{F}_{1, R_{max} T^{-\frac{1}{q}}, 1},r).
      \nonumber
   \end{align}
    \noindent
    According to \cite{mendelson2003}, it can be shown that there is a constant $c$ such that if $\lambda_{t}^{1} \geq \frac{1}{nR_{max}^{2}}$, then for all $r \geq \frac{1}{n}$ it holds $ \mathfrak{R} (\mathcal{F}_{1, R_{max} T^{-\frac{1}{q}} , 1},r) \geq \sqrt{\frac{c}{n} \sum_{j=1}^{\infty} \min \left( r , R_{max}^{2} T^{-\frac{2}{q}} \lambda_{1}^{j}\right)}$, which with some algebra manipulations gives the desired result.

The following lemma is used in the proof of the \ac {LRC} bounds for the $L_{S_{q}}$-Schatten norm  regularized \ac {MTL} in \corref{SchattenNormSC}.
\begin{lemm}
[Non-commutative Khintchine's inequality \citep{lustpiquard1986}]
\label{Non-commutativeKhintchineinequality}
Let $\boldsymbol{Q}_{1},\ldots, \boldsymbol{Q}_{n}$ be a set of arbitrary $m \times n$ matrices, and let $\sigma_{1},\ldots, \sigma_{n}$ be a sequence of independent Bernoulli random variables. Then for all $p \geq 2$,  
\begin{align}
\label{Non-commutativeK.K.inequality}
 \left[ \mathbb{E}_{\sigma} \left\|  \sum_{i=1}^{n} \sigma_{i} \boldsymbol{Q}_{i}\right\|_{S_{p}}^{p} \right]^{1/p} \leq  p^{1/2} \max \left\lbrace  \left\| \left( \sum_{i=1}^{n} \boldsymbol{Q}_{i}^{T} \boldsymbol{Q}_{i}\right)^{1/2} \right\|_{S_{p}}   ,   \left\| \left( \sum_{i=1}^{n} \boldsymbol{Q}_{i} \boldsymbol{Q}_{i}^{T}\right)^{1/2} \right\|_{S_{p}} \right\rbrace. 
\end{align}
\noindent
\end{lemm}
\section*{Proof of \corref{SchattenNormSC}}
In order to find an \ac {LRC} bound for a $L_{S_{q}}$-Schatten norm regularized hypothesis space \eqref{HSForFixedMAppingSchatten}, one just needs to bound the expectation term in \eqref{LRCforStrongConvexRegularizedMTL}. Define $\boldsymbol{U}_{t}^{i}$ as a matrix with $T$ columns, whose only non-zero $t^{th}$ column equals $ \sum_{j>h_t} \left\langle \frac{1}{n} \phi (X_t^{i}) , \boldsymbol{u}_{t}^{j} \right\rangle  \boldsymbol{u}_{t}^{j}$. Also, note that for the Schatten norm regularized hypothesis space \eqref{HSForFixedMAppingSchatten}, it holds that $\boldsymbol{D} = \boldsymbol{I}$. Therefore, we will have
\begin{align}
\mathbb{E}_{X, \sigma} & \left\|  \boldsymbol{D}^{-1/2} \boldsymbol{V} \right\|_{*} \! \!= \mathbb{E}_{X, \sigma} \left\| \left(\sum_{j > h_{t}}  \left\langle \frac{1}{n} \sum_{i=1}^{n} \sigma_{t}^{i} \phi(X_{t}^{i}) , \boldsymbol{u}_{t}^{j} \right\rangle \boldsymbol{u}_{t}^{j} \right)_{t=1}^{T}  \right\|_{S_{q^{*}}}
\nonumber\\
& = \mathbb{E}_{X, \sigma} \left\|  \sum_{t=1}^{T} \sum_{i=1}^{n} \sigma_{t}^{i} \boldsymbol{U}_{t}^{i} \right\|_{S_{q^{*}}}  \stackrel{ \text{Jensen}}{\leq}  \mathbb{E}_{X} \left\lbrace  \mathbb{E}_{\sigma} \left\|  \sum_{t=1}^{T} \sum_{i=1}^{n} \sigma_{t}^{i} \boldsymbol{U}_{t}^{i} \right\|_{S_{q^{*}}}^{q^{*}} \right\rbrace ^{\frac{1}{q^{*}}} 
\nonumber\\
& \stackrel{  \eqref{Non-commutativeK.K.inequality}}{\leq} \sqrt{q^{*}} \mathbb{E}_{X}       \max \left\lbrace \left\| \left( \sum_{t=1}^{T} \sum_{i=1}^{n} \left( \boldsymbol{U}_{t}^{i} \right)^{T}  \boldsymbol{U}_{t}^{i} \right)^{1/2} \right\|_{S_{q^{*}}} ,
 \left\| \left( \sum_{t=1}^{T} \sum_{i=1}^{n}   \boldsymbol{U}_{t}^{i} \left( \boldsymbol{U}_{t}^{i} \right)^{T} \right)^{1/2} \right\|_{S_{q^{*}}}  \right\rbrace
\nonumber\\
& \stackrel{(\dag\dag\dag)}{=}  \sqrt{q^{*}}  \mathbb{E}_{X} \left\| \left( \sum_{t=1}^{T} \sum_{i=1}^{n} \left( \boldsymbol{U}_{t}^{i} \right)^{T}  \boldsymbol{U}_{t}^{i} \right)^{1/2} \right\|_{S_{q^{*}}}
 =  \sqrt{q^{*}}  \mathbb{E}_{X} \left( \textbf{tr} \left( \sum_{t=1}^{T} \sum_{i=1}^{n} \left( \boldsymbol{U}_{t}^{i} \right)^{T}  \boldsymbol{U}_{t}^{i} \right)^{\frac{q^{*}}{2}} \right)^\frac{1}{q^{*}} 
\nonumber\\
& =  \sqrt{q^{*}}  \mathbb{E}_{X} \left(  \left(\sum_{t=1}^{T} \sum_{i=1}^{n}  \left\|  \sum_{j>h_t} \left\langle \frac{1}{n} \phi (X_t^{i}) , \boldsymbol{u}_{t}^{j} \right\rangle  \boldsymbol{u}_{t}^{j} \right\|^{2}  \right)^{\frac{q^{*}}{2}} \right)^\frac{1}{q^{*}} 
\nonumber\\
& = \sqrt{q^{*}} \mathbb{E}_{X}  \left(  \sum_{t=1}^{T} \sum_{i=1}^{n}  \left\|  \sum_{j>h_t} \left\langle \frac{1}{n} \phi (X_t^{i}) , \boldsymbol{u}_{t}^{j} \right\rangle  \boldsymbol{u}_{t}^{j} \right\|^{2}  \right)^{\frac{1}{2}}
\nonumber\\
& = \frac{\sqrt{q^{*}}}{n}  \mathbb{E}_{X}  \left(\sum_{t=1}^{T} \sum_{i=1}^{n}  \sum_{j>h_t}  \left\langle  \phi (X_t^{i}) , \boldsymbol{u}_{t}^{j} \right\rangle ^{2}   \right)^{\frac{1}{2}}  
\nonumber\\
& \stackrel{\text {Jensen}}{\leq} \frac{\sqrt{q^{*}}}{n} \left( \sum_{t=1}^{T} \sum_{i=1}^{n}  \sum_{j > h_t} \lambda_{t}^{j}  \right)^{\frac{1}{2}}  
 =   \sqrt{ \frac{q^{*}} {n} \left\|  \left( \sum_{j > h_t} \lambda_{t}^{j}  \right)_{t=1}^{T} \right\|_{1} }. 
\label{ExpectationBoundq<2}
\end{align}
\noindent
where in $(\dag\dag\dag)$, we assumed that the first term in the max argument is the largest one. 

\section*{Proof of \corref{GraphSC}}
Similar to the proof of \corref{SchattenNormSC}, for the graph regularized hypothesis space \eqref{GraphRegHS}, one can bound the expectation term in \eqref{LRCforStrongConvexRegularizedMTL} as
\begin{align}
 \mathbb{E}_{X, \sigma} & \left\|  \boldsymbol{D}^{-1/2} \boldsymbol{V} \right\|_{*}  =  \mathbb{E}_{X,\sigma}   \left[ \text {tr} \left(  \boldsymbol{V}^{T} \boldsymbol{D}^{-1} \boldsymbol{V} \right)   \right]^{\frac{1}{2}}
 \nonumber\\
 & \stackrel{\text{Jensen}}{\leq} \mathbb{E}_{X}  \left( \frac{1}{n^{2}}  \sum_{t,s=1}^{T,T}   \sum_{i,l=1}^{n,n}  \sum_{j>h_{t}} \sum_{k > h_{s}} \boldsymbol{D}_{st}^{-1} \mathbb{E}_{\sigma} \left( \sigma_{t}^{i} \sigma_{s}^{l}\right)   \left\langle  \phi(X_t^{i}) , \boldsymbol{u}_{t}^{j} \right\rangle    \left\langle  \phi(X_s^{l}) , \boldsymbol{u}_{s}^{k} \right\rangle  \left\langle \boldsymbol{u}_{t}^{j}, \boldsymbol{u}_{s}^{k} \right\rangle   \right)^{\frac{1}{2}}
 \nonumber\\
 & = \mathbb{E}_{X}  \left( \frac{1}{n}  \sum_{t=1}^{T} \boldsymbol{D}_{tt}^{-1} \sum_{j>h_{t}}  \frac{1}{n} \sum_{i=1}^{n}  \left\langle  \phi(X_t^{i}) , \boldsymbol{u}_{t}^{j} \right\rangle^{2}    \right)^{\frac{1}{2}}
 \nonumber\\
 & \stackrel{\text{Jensen}}{\leq} \left( \frac{1}{n}  \sum_{t=1}^{T} \boldsymbol{D}_{tt}^{-1} \sum_{j>h_{t}}  \frac{1}{n} \sum_{i=1}^{n} \mathbb{E}_{X} \left\langle   \phi(X_t^{i}) , \boldsymbol{u}_{t}^{j} \right\rangle^{2} \right)^{\frac{1}{2}} 
 \nonumber\\
 & = \frac{1}{\sqrt{n}} \left(  \sum_{t=1}^{T} \sum_{j>h_{t}} \boldsymbol{D}_{tt}^{-1}  \lambda_{t}^{j} \right)^{\frac{1}{2}}  =   \sqrt{ \frac{1}{n} \left\| \left( \boldsymbol{D}_{tt}^{-1} \sum_{j > h_{t}} \lambda_{t}^{j}\right)_{t=1}^{T}  \right\|_{1} }. 
 \label{ExpectationBoundGraph}
 \end{align}

\section{Proof of the results in \sref{sec:StrongConvexityExcessRisk}: ``Excess Risk Bounds for \ac{MTL} models with Strongly Convex Regularizers"}
\label{D}
\section*{Proof of \corref{DataBoundGroupNorm}}
First notice that $\hat{\mathfrak{R}}(\mathcal{F}^{*}_{q}, c_{3} r) \leq 2 \hat{\mathfrak{R}}(\mathcal{F}_{q},\frac{c_{3} r}{4 L^{2}})$. Assume that $(\hat{\boldsymbol{u}}_{t}^{j})_{j \geq 1}$ is an orthonormal basis of $\hcal_K$ of matrix $\boldsymbol{K}_{t}$. Then similar to the proof of \thrmref{GeneralLRCBoundForFixedMapping2} it can be shown that 
\begin{align}
\hat{\mathfrak{R}}(\mathcal{F}_{q}, r) & \leq \frac{1}{T} \mathbb{E}_{\sigma} \left\lbrace  \sup_{P_{n} \boldsymbol{f}^{2} \leq r} \left[  \left( \sum_{t=1}^{T}  \sum_{j=1}^{\hat{h}_t} \hat{\lambda}_t^{j} \left\langle \boldsymbol{w}_t, \hat{\boldsymbol{u}}_{t}^{j} \right\rangle^{2}     \right) ^{\frac{1}{2}}  \left( \sum_{t=1}^{T}  \sum_{j=1}^{\hat{h}_t} \hat{\lambda}^{j^{-1}}_t \left\langle \frac{1}{n} \sum_{i=1}^{n} \sigma_t^{i} \hat{\phi} (X_t^{i}) , \hat{\boldsymbol{u}}_{t}^{j} \right\rangle^{2}   \right) ^{\frac{1}{2}}  \right] \right\rbrace
   \nonumber\\
   & + \frac{\sqrt{2} R}{T} \mathbb{E}_{\sigma}  \left\| \boldsymbol{D}^{-1/2} \hat{\boldsymbol{V}} \right\|_{2,q^*}
   \nonumber\\
   & \leq \sqrt{  \frac{ r \sum_{t=1}^{T} \hat{h}_t} {nT}  } +  \frac{\sqrt{2} R}{T} \mathbb{E}_{\sigma}  \left\| \left( \sum_{j > \hat{h}_t}^{n}  \left\langle \frac{1}{n} \sum_{i=1}^{n} \sigma_{t}^{i} \hat{\phi}(X_t^{i}), \hat{\boldsymbol{u}}_{t}^{j} \right\rangle \hat{\boldsymbol{u}}_{t}^{j}\right)_{t=1}^{T} \right\|_{2,q^*},
   \nonumber
\end{align}
where the last inequality is obtained by replacing $\hat{\boldsymbol{V}} = \left( \sum_{j > \hat{h}_t}^{n}  \left\langle \frac{1}{n} \sum_{i=1}^{n} \sigma_{t}^{i} \hat{\phi}(X_t^{i}), \hat{\boldsymbol{u}}_{t}^{j} \right\rangle \hat{\boldsymbol{u}}_{t}^{j}\right)_{t=1}^{T}$ and $\boldsymbol{D} = \boldsymbol{I}$, and regarding the fact that  $\mathbb{E}_{\sigma} \left\langle \frac{1}{n} \sum_{i=1}^{n} \sigma_t^{i} \hat{\phi} (X_t^{i}) , \hat{\boldsymbol{u}}_{t}^{j} \right\rangle^{2} = \frac{ \hat{\lambda}_t^{j}}{n}$ and $P_{n} \boldsymbol{f}^{2} \leq r $ implies $\frac{1}{T} \sum_{t=1}^{T} \sum_{j=1}^{n} \hat{\lambda}_t^{j} \left\langle  \boldsymbol{w}_{t}, \hat{\boldsymbol{u}}_{t}^{j} \right\rangle^{2} \leq r$.

Now, similar to the proof of \lemmref{A_2ExcpectationGroupNorm}, it can be shown that 
\begin{align}
\mathbb{E}_{\sigma} \left\| \left(\sum_{j > \hat{h}_{t}}  \left\langle \frac{1}{n} \sum_{i=1}^{n} \sigma_{t}^{i} \hat{\phi}(X_{t}^{i}) , \hat{\boldsymbol{u}}_{t}^{j} \right\rangle \hat{\boldsymbol{u}}_{t}^{j} \right)_{t=1}^{T}  \right\|_{2,q^{*}}   \leq \sqrt{\frac{ {q^{*}}^{2} }{n}  \left\|  \left( \sum_{j > \hat{h}_t} \hat{\lambda}_t^{j} \right)_{t=1}^{T}   \right\|_{\frac{q^{*}}{2}} }.
\nonumber
\end{align}
Note that, for the empirical \ac {LRC}, the expectation is taken only with respect to the Radamacher variables $(\sigma_{t}^{i})_{(t,i=1)}^{(T,n)}$. Therefore, we get
\begin{align}
\hat{\mathfrak{R}}(\mathcal{F}_{q},\frac{c_{3} r}{4 L^{2}}) \leq \sqrt{\frac{c_{3} r \sum_{t=1}^{T} \hat{h}_{t}}{4 nT L^{2}}} + \sqrt{\frac{2 {q^{*}}^{2} R^{2}_{max}}{nT^{2}} \left\| \left( \sum_{j>\hat{h}_{t}}^{n} \hat{\lambda}_{t}^{j}\right)_{t=1}^{T}  \right\|_{\frac{q^{*}}{2}} },
\nonumber
\end{align}
\noindent
which implies,
\begin{align}  
\hat{\psi}_{n}(r) & \leq 2 c_{1} \left( \sqrt{\frac{c_{3} r \sum_{t=1}^{T} \hat{h}_{t}}{4 nT L^{2}}} + \sqrt{\frac{2 {q^{*}}^{2} R^{2}_{max}}{nT^{2}} \left\| \left( \sum_{j>\hat{h}_{t}}^{n} \hat{\lambda}_{t}^{j}\right)_{t=1}^{T}  \right\|_{\frac{q^{*}}{2}} } \right) +  \frac{c_{2}x}{nT}
\nonumber\\
& =  \sqrt{\frac{ c_{1}^{2} c_{3} r \sum_{t=1}^{T} \hat{h}_{t}}{nT L^{2}}} + \sqrt{\frac{ 8 c_{1}^{2} {q^{*}}^{2} R^{2}_{max}}{nT^{2}} \left\| \left( \sum_{j>\hat{h}_{t}}^{n} \hat{\lambda}_{t}^{j}\right)_{t=1}^{T}  \right\|_{\frac{q^{*}}{2}} } +  \frac{c_{2}x}{nT}.
\nonumber
\end{align}
\noindent
Denote the right hand side by $\hat{\psi}_{n}^{ub} (r)$. Solving the fixed point equation $\hat{\psi}_{n}^{ub} (r) =\sqrt{\alpha r}+\gamma =r$ for
\begin{align}
\alpha = \frac{c_{1}^{2} c_{3} \sum_{t=1}^{T} \hat{h}_{t} }{nT L^{2}}, \qquad \gamma= \sqrt{\frac{ 8 c_{1}^{2} {q^{*}}^{2} R^{2}_{max}}{nT^{2}} \left\| \left( \sum_{j>\hat{h}_{t}}^{n} \hat{\lambda}_{t}^{j}\right)_{t=1}^{T}  \right\|_{\frac{q^{*}}{2}} } +  \frac{c_{2}x}{nT}, 
\end{align}
gives $\hat{r}^{*} \leq \alpha+2\gamma$. Substituting $\alpha$ and $\gamma$ completes the proof.
\section{Proof of the results in \sref{sec:Discussion}: ``Discussion"} 
\label{E}
\section*{Proof of \thrmref{GRCFixedMappingStrongConvexity}}
 Note that regarding the definition of $A_{2}$ in \eqref{A_2DefinitionStrongConvexity}, the global rademacher complexity for each case can be obtained by replacing the tail-sum $\sum_{j>h_{t}} \lambda_{t}^{j}$ in the bound of its corresponding $A_{2}(\mathcal{F})$ by $\sum_{j=1}^{\infty} \lambda_{t}^{j}=\textbf{tr} (J_{t})$. Indeed, similar to the proof of \lemmref{A_2ExcpectationGroupNorm}, it can be shown that for the group norm with $q \in [1,2]$,
    \begin{align}
       \mathfrak{R}(\mathcal{F}_{q}) & = \mathbb{E}_{X, \sigma} \left\lbrace \sup_{\boldsymbol{f}=(f_1,\ldots,f_T) \in \mathcal{F}_{q}} \frac{1}{nT}  \sum_{t=1}^{T} \sum_{i=1}^{n} \sigma_{t}^{i} f_{t}(X_{t}^{i}) \right\rbrace
     \nonumber\\
       & \leq \frac{ \sqrt{2} R}{T}  \mathbb{E}_{X, \sigma} \left\| \left(  \frac{1}{n} \sum_{i=1}^{n} \sigma_{t}^{i} \phi(X_{t}^{i}) \right)_{t=1}^{T}  \right\|_{2,q^{*}}.
     \nonumber
    \end{align}
    \noindent
    Also, one can verify the following 
     \begin{align}
       \mathbb{E}_{X, \sigma} \left\| \left(  \frac{1}{n} \sum_{i=1}^{n} \sigma_{t}^{i} \phi(X_{t}^{i}) \right)_{t=1}^{T}  \right\|_{2,q^{*}} & = \mathbb{E}_{X, \sigma} \left\| \left(\sum_{j=1}^{\infty}  \left\langle \frac{1}{n} \sum_{i=1}^{n} \sigma_{t}^{i} \phi(X_{t}^{i}) , \boldsymbol{u}_{t}^{j} \right\rangle \boldsymbol{u}_{t}^{j} \right)_{t=1}^{T}  \right\|_{2,q^{*}} 
       \nonumber\\
       & \leq  \sqrt{\frac{ {q^{*}}^{2} \mathcal{K} e T^{\frac{2}{q^{*}}}}{n^{2}}} + \sqrt{ \frac{{ eq^{*}}^{2}}{n} \left\| \left( \sum_{j=1}^{\infty} \lambda_{t}^{j} \right)_{t=1}^{T}  \right\|_{\frac{q^{*}}{2}} }
        \nonumber\\
       & =  \frac{ \sqrt{\mathcal{K} e }  {q^{*}} T^{\frac{1}{q^{*}}}}{n} + \sqrt{ \frac{ e {q^{*}}^{2}}{n} \left\| \left( \text {tr} \left( J_{t}\right)  \right)_{t=1}^{T}  \right\|_{\frac{q^{*}}{2}}}.
     \end{align}
     \noindent
     where the inequality is obtained in a similar way as in \lemmref{A_2ExcpectationGroupNorm}.  The \ac {GRC} bounds for the other cases can be easily derived in a very similar manner.

\Urlmuskip=0mu plus 1mu\relax

\bibliographystyle{plainurl}
\bibliography{Arxiv2017paperA}

\end{document}